\documentclass[letterpaper, 11pt]{article}
\usepackage[margin = 1in]{geometry}

\usepackage{natbib}

\usepackage{amsmath}
\usepackage{amssymb}
\usepackage{amsthm}
\usepackage{booktabs}
\usepackage{xspace}
\usepackage{graphicx}

\usepackage{float}
\usepackage{subcaption}

\usepackage{mathtools}
\usepackage{pifont}
\usepackage{cancel}
\usepackage{algorithm}
\usepackage{algorithmic}
\usepackage{multirow}
\usepackage{colortbl}
\usepackage{utfsym}

\newcommand{\BlackBox}{\rule{1.5ex}{1.5ex}}  
\ifdefined\proof
    \renewenvironment{proof}{\par\noindent{\bf Proof\ }}{\hfill\BlackBox\\[2mm]}
\else
    \newenvironment{proof}{\par\noindent{\bf Proof\ }}{\hfill\BlackBox\\[2mm]}
\fi

\newtheorem{proposition}{Proposition}

\newtheorem{lemma}{Lemma}
\newtheorem{corollary}{Corollary}

\def\bfc{\mathbf{c}}

\def\bff{\mathbf{f}}

\def\bfm{\mathbf{m}}

\def\bfr{\mathbf{r}}

\def\bfu{\mathbf{u}}
\def\bfv{\mathbf{v}}
\def\bfw{\mathbf{w}}
\def\bfx{\mathbf{x}}
\def\bfy{\mathbf{y}}
\def\bfz{\mathbf{z}}

\def\bfB{\mathbf{B}}
\def\bfE{\mathbf{E}}
\def\bfN{\mathbf{N}}
\def\bfR{\mathbf{R}}
\def\bfO{\mathbf{O}}
\def\bfU{\mathbf{U}}
\def\bfV{\mathbf{V}}
\def\bfT{\mathbf{T}}

\def\rmd{\mathrm{d}}
\def\rmdx{\mathrm{d}\bfx}
\def\rmdz{\mathrm{d}\bfz}

\def\bbE{\mathbb{E}}
\def\bbR{\mathbb{R}}
\def\bbV{\mathbb{V}}

\def\nablazt{\nabla_{\bfz_t}}

\def\nablaxt{\nabla_{\bfx_t}}

\def\calI{\mathcal{I}}
\def\calR{\mathcal{R}}

\def\calL{\mathcal{L}}
\def\calD{\mathcal{D}}
\def\calN{\mathcal{N}}

\def\calH{\mathcal{H}}

\def\hatx{\hat{\bfx}}

\def\hatz{\hat{\bfz}}

\def\bfmu{\boldsymbol{\mu}}
\def\bftheta{\boldsymbol{\theta}}
\def\bfeps{\boldsymbol{\epsilon}}

\def\bfphi{\boldsymbol{\phi}}
\def\bfI{\mathbf{I}}

\def\eg{\emph{e.g.}\xspace}
\def\ie{\textit{i.e.}\xspace}

\def\eqref#1{(\ref{#1})}

\def\DAE{\mathrm{DAE}}

\def\DSM{\mathrm{DSM}}

\def\bfSigma{\boldsymbol{\Sigma}}
\def\bfU{\mathbf{U}}
\def\bfR{\mathbf{R}}
\def\bfQ{\mathbf{Q}}

\def\bfLambda{\boldsymbol{\Lambda}}
\def\ourName{GITS\xspace}

\usepackage{hyperref}
\hypersetup{
	colorlinks=true,
	breaklinks=true,
	urlcolor=blue,
	linkcolor=blue,
	bookmarksopen=false,
	filecolor=black,
	citecolor=blue,
	linkbordercolor=blue
}

\begin{document}

\title{Geometric Regularity in Deterministic Sampling Dynamics of \\ Diffusion-based Generative Models}

\author{Defang Chen$^{1}$ \quad Zhenyu Zhou$^{2}$ \quad Can Wang$^{2}$ \quad Siwei Lyu$^{1}$
\\$^{1}$University at Buffalo, State University of New York\quad $^{2}$Zhejiang University
}
 
\maketitle

\begingroup
\renewcommand\thefootnote{}        
\footnotetext{The full manuscript was accepted by \textit{Journal of Statistical Mechanics: Theory and Experiment} (2025), and a short version was published in \textit{International Conference on Machine Learning} (2024).}  
\endgroup

\begin{abstract}
    Diffusion-based generative models employ stochastic differential equations (SDEs) and their equivalent probability flow ordinary differential equations (ODEs) to establish a smooth transformation between complex high-dimensional data distributions and tractable prior distributions. In this paper, we reveal a striking geometric regularity in the deterministic sampling dynamics of diffusion generative models: each simulated sampling trajectory along the gradient field lies within an extremely low-dimensional subspace, and all trajectories exhibit an almost identical ``boomerang'' shape, regardless of the model architecture, applied conditions, or generated content. We characterize several intriguing properties of these trajectories, particularly under closed-form solutions based on kernel-estimated data modeling. We also demonstrate a practical application of the discovered trajectory regularity by proposing a dynamic programming-based scheme to better align the sampling time schedule with the underlying trajectory structure. This simple strategy requires minimal modification to existing deterministic numerical solvers, incurs negligible computational overhead, and achieves superior image generation performance, especially in regions with only $5 \sim 10$ function evaluations. 
\end{abstract}
    

\section{Introduction}
\label{sec:intro}

Diffusion-based generative models~\citep{sohl2015deep,song2019ncsn,ho2020ddpm,song2021sde,karras2022edm,chen2024trajectory}, originally inspired by nonequilibrium statistical mechanics~\citep{jarzynski1997equilibrium,sohl2015deep,bahri2020statistical}, have recently garnered significant attention and achieved remarkable results in image~\citep{dhariwal2021diffusion,rombach2022ldm}, audio~\citep{kong2021diffwave,huang2023make}, video~\citep{ho2022video,blattmann2023videoLDM}, and notably in text-to-image synthesis~\citep{saharia2022photorealistic,ruiz2023dreambooth,podell2024sdxl,esser2024scaling}. These models introduce noise into data through a {\em forward process} and subsequently generate data by sampling via a {\em backward process}. Both processes are characterized and modeled using stochastic differential equations (SDEs)~\citep{song2021sde}. In diffusion-based generative models, the pivotal element is the score function, defined as the gradient of the log data density \textit{w.r.t.}\ the input~\citep{hyvarinen2005estimation,lyu2009interpretation,raphan2011least,vincent2011dsm}, irrespective of specific model configurations. 
Training such a model involves learning the score function, which can be equivalently achieved by training a noise-dependent denoising model to minimize the mean squared error in data reconstruction, using the data-noise pairings generated during the forward process~\citep{karras2022edm,chen2024trajectory}.
To synthesize new data, diffusion-based generative models solve the acquired score-based backward SDE through a numerical solver. Recent research has shown that the backward SDE can be effectively replaced by an equivalent probability flow ordinary differential equation (PF-ODE), preserving identical marginal distributions~\citep{song2021sde,song2021ddim,lu2022dpm,zhang2023deis,zhou2024fast}. This deterministic ODE-based generation reduces the need for stochastic sampling to just the randomness in the initial sample selection, thereby simplifying and granting more control over the entire generative process~\citep{song2021ddim,song2021sde}. Under the PF-ODE formulation, starting from white Gaussian noise, the \textit{sampling trajectory} is formed by running a numerical solver with discretized time steps. These steps collectively constitute the \textit{time schedule} used in sampling.

Despite the impressive generative capabilities exhibited by diffusion-based models, many mathematical and statistical aspects of these models remain veiled in mystery. This obscurity primarily stems from the inherent complexity of the associated SDEs, the nonlinear nature of neural network parameterizations, and the high dimensionality of real-world data~\citep{biroli2023generative,biroli2024kernel,biroli2024dynamical,ghio2024sampling,ikeda2025speed,yu2025nonequilbrium,achilli2025memorization}. 
In this paper, we reveal a striking regularity in the deterministic sampling dynamics of diffusion models, \ie, the tendency of sample paths to exhibit a consistent ``boomerang'' shape, as illustrated in Figure~\ref{fig:model}. More precisely, we observe that each sampling trajectory barely strays from the displacement vector connecting its starting and ending points (Section~\ref{subsec:one_dim}), while the trajectory deviation can be effectively captured using two orthogonal bases (Section~\ref{subsec:multi_dim}). Therefore, the sampling trajectory in the original high-dimensional data space can be faithfully represented by its projection onto a three-dimensional subspace. These projected spatial curves are fully characterized by the \textit{Frenet-Serret formulas} and exhibit a remarkably consistent geometric structure, irrespective of initial random samples, applied control signals, or target data samples (Figure~\ref{fig:traj_3d}, Section~\ref{subsec:three_dim}). This intrinsic regularity provides theoretical support for several empirical practices in the literature, such as employing a shared time schedule across different samples and using large sampling steps with negligible truncation error~\citep{song2021ddim,karras2022edm,lu2022dpm}, particularly during the initial stage of generation~\citep{dockhorn2022genie,zhou2024fast}.
\begin{figure}[t]
	\centering
	\includegraphics[width=0.85\textwidth]{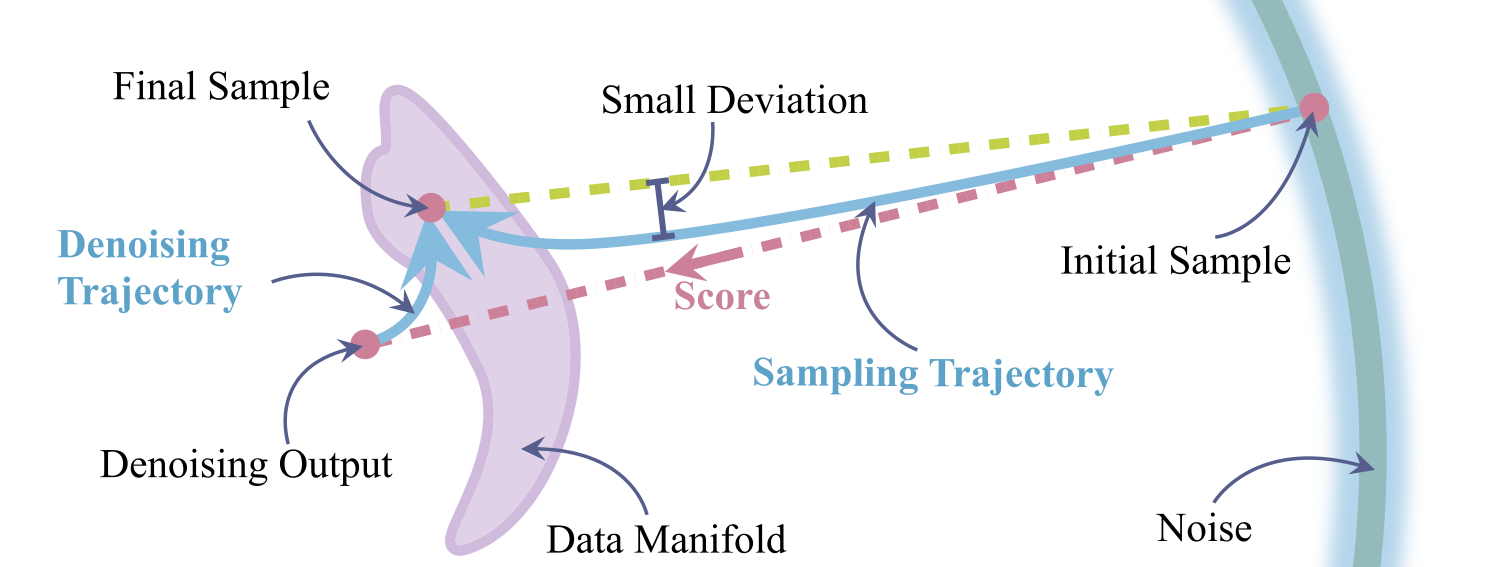}
	\caption{A geometric picture of deterministic sampling dynamics in diffusion-based generative models. Each initial sample (from the noise distribution) starts from a big sphere and converges to the final sample (in the data manifold) along a regular \textit{sampling trajectory}. 
    The score direction points to the denoising output of the current position, and the denoising output forms an implicit \textit{denoising trajectory} controlling the explicit sampling trajectory. Each sampling trajectory inherently lies in a low-dimensional subspace with almost the same shape.
    }
    \label{fig:model}
    \vspace{-1em}
\end{figure}

The geometric trajectory regularity of deterministic sampling trajectories has not been previously investigated. This work aims to elucidate this phenomenon. We begin by simplifying any ODE-based sampling trajectory to its drift-free counterpart (Section~\ref{subsec:equivalence}), which reveals an implicit \textit{denoising trajectory} controlling the direction of the associated sampling dynamics (Section~\ref{subsec:denoising_trajectory}). Building on this insight, we establish a connection between the closed-form solution of denoising trajectory, which is derived under kernel density estimates (KDEs) with varying bandwidths to approximate the data distribution perturbed by different noise levels, and the classical mean-shift algorithm~\citep{fukunaga1975estimation,cheng1995mean,comaniciu2002mean}. Although the KDE-based solution is not directly tractable for practical trajectory simulation, it asymptotically converges to the optimal solution derived from the real data distribution and provides a solid foundation for theoretical analysis of our discovered trajectory structure. We further characterize the deterministic sampling dynamics from both local and global perspectives: locally, they exhibit stepwise rotation and monotone likelihood increase; globally, they follow a linear-nonlinear-linear mode-seek path of approximately constant length, as implied by this interpretation of the PF-ODE (Section~\ref{subsec:theoretical_analysis}). Moreover, we theoretically analyze the trajectory deviation under the Gaussian data assumption (Section~\ref{subsec:theoretical_analysis_gaussian}). This geometric regularity unifies prior empirical observations and clarifies several existing heuristics for accelerating diffusion sampling. As a demonstration of this insight, we develop an efficient and effective accelerated sampling algorithm based on dynamic programming to determine the optimal time schedule (Section~\ref{sec:trajectory_algorithm}). Experimental results demonstrate that the proposed approach significantly improves the performance of diffusion-based generative models using only a few ($\leq10$) function evaluations. Our main contributions are summarized as follows: 
\begin{itemize} 
    \item We demonstrate and characterize a strong geometric regularity in deterministic sampling dynamics of diffusion-based generative models, \ie, each sampling trajectory exhibits a consistent ``boomerang''-shaped structure confined to an extremely low-dimensional subspace. 
    \item We provide theoretical explanations for this regularity through closed-form analyses of the denoising trajectory under the empirical data distribution and under the Gaussian data assumption. Several derived properties offer insights into both the local and global structures of sampling trajectories. 
    \item We develop a dynamic programming-based algorithm that leverages the trajectory regularity to determine an optimal sampling time schedule. It incurs negligible computational overhead while substantially improving image quality, particularly in few-step inference regimes.
\end{itemize}



\section{Preliminaries}
\label{sec:preliminaries}

\subsection{Generative Modeling with Stochastic Differential Equations}
For successful generative modeling, it is essential to connect the data distribution $p_{d}$ with a manageable, non-informative noise distribution $p_{n}$. Diffusion models achieve this objective by incrementally introducing white Gaussian noise into the data, effectively obliterating its structures, and subsequently reconstructing the synthesized data from noise samples via a series of denoising steps. A typical choice for $p_{n}$ is an isotropic multivariate normal distribution with zero mean. The forward step can be modeled as a diffusion process $\{\bfz_t\}$ for $t \in [0,T]$ starting from the initial condition $\bfz_0\sim p_d$, which corresponds to the solution of an It\^{o} stochastic differential equation (SDE)~\citep{song2021sde,oksendal2013stochastic}
\begin{equation}
    \label{eq:pre_forward_sde}
    \rmd \bfz_t = \bff(\bfz_t, t) \rmd t + g(t) \rmd \bfw_t, \quad \bff(\cdot, t): \bbR^d \rightarrow \bbR^d, \quad g(\cdot): \bbR \rightarrow \bbR,
\end{equation}
where $\bfw_t$ denotes the Wiener process; $\bff(\cdot, t)$ is a vector-valued function referred to as {\it drift} coefficient and $g(\cdot)$ is a scalar function referred to as {\it diffusion} coefficient.\footnote{The noise term in this case is independent of the state $\bfz_t$ (\textit{a.k.a.}\ additive noise), and therefore the It\^{o} and Stratonovich interpretations of the above SDE coincide~\citep{stratonovich1968conditional,sarkka2019applied}. A unique, strong solution of this SDE exists when the time-varying drift and diffusion coefficients are globally Lipschitz in both state and time~\citep{oksendal2013stochastic}.} The temporal marginal distribution of $\bfz_t$ is denoted as $p_t(\bfz_t)$, with $p_0(\bfz_0)=p_d(\bfz_0)$. By properly setting the coefficients and terminal time $T$, the data distribution $p_d$ is smoothly transformed to the approximate noise distribution $p_T(\bfz_T)\approx p_n$ in a forward manner. 
The solutions to It\^{o} SDEs are always Markov processes, and they can be fully characterized by the transition kernel $p_{st}(\bfz_t | \bfz_s)$ with $0\leq s < t \leq T$. This transition kernel becomes a Gaussian distribution when considering the linear SDE with an affine drift coefficient $\bff(\bfz_t, t)=f(t)\bfz_t$. In this case, we can directly sample data $\bfz_0$ and its corrupted version $\bfz_t$ with different levels of noise, which largely simplifies the computation of the forward process and eases the model training. Therefore, linear SDEs are widely used in practice.\footnote{Some non-linear diffusion-based generative models also exist~\citep{zhang2021diffusion,chen2022likelihood,liu2023i2sb}, but they are beyond the scope of this paper.} The transition kernel $p_{0t}(\bfz_t | \bfz_0)$ derived with standard techniques~\citep{sarkka2019applied,karras2022edm} has the following analytic form
\begin{equation}
    \label{eq:pre_kernel}
    p_{0t}(\bfz_t | \bfz_0)=\calN\left(\bfz_t ; s(t)\bfz_0, s^2(t)\sigma^2(t)\bfI\right),
\end{equation}
or equivalently, $\bfz_t = s(t)\bfz_0 + \left[s(t)\sigma(t)\right]\bfeps_t$,  
where $s(t) = \exp (\int_{0}^{t} f(\xi) \rmd \xi)$, $\sigma(t) = \sqrt{\int_{0}^{t} \left[g(\xi)/s(\xi)\right]^2\rmd \xi}$, and $\bfeps_t\sim \calN(\mathbf{0}, \bfI)$. For notation simplicity, we hereafter denote them as $s_t$ and $\sigma_t$, respectively. Then, we can rewrite the forward linear SDE \eqref{eq:pre_forward_sde} in terms of $s_t$ and $\sigma_t$, 
\begin{equation}
    \label{eq:pre_new_forword_sde}
    \rmd \bfz_t =\frac{\rmd \log s_t}{\rmd t}\bfz_t \, \rmd t + s_t\sqrt{\frac{\rmd \sigma^2_t}{\rmd t}}\ \rmd \bfw_t, \quad 
    f(t) = \frac{\rmd \log s_t}{\rmd t}, \quad \text{and} \quad g(t)=s_t\sqrt{\frac{\rmd \sigma^2_t}{\rmd t}}. 
\end{equation}
Furthermore, following previous works~\citep{kingma2021vdm,rombach2022ldm}, we define the signal-to-noise ratio (SNR) of the transition kernel \eqref{eq:pre_kernel} as $\text{SNR}(t)=s^2_t/(s_t^2\sigma_t^2)=1/\sigma^2_t$, which is a monotonically non-increasing function of $t$. A simple corollary is that any linear diffusion process with the same $\sigma_t$ exhibits an identical SNR function. 
Two specific forms of linear SDEs, namely, the variance-preserving (VP) SDE and the variance-exploding (VE) SDE~\citep{song2021sde,karras2022edm} are widely used in large-scale diffusion models, see more details in Section~\ref{subsec:details_linear_SDEs}.

The reversal of the forward linear SDE as expressed in \eqref{eq:pre_new_forword_sde} is represented by another backward SDE, which facilitates the synthesis of data from noise through a backward sampling~\citep{feller1949theory,anderson1982reverse}. 
Based on the well-known Fokker-Planck-Kolmogorov (FPK) equation that describes the evolution of $p_t(\bfz_t)$ given the initial condition $p_0(\bfz_0)=p_d(\bfz_0)$~\citep{oksendal2013stochastic}, \ie, 
\begin{equation}
    \label{eq:pre_fpe}
        \frac{\partial p_t(\bfz_t)}{\partial t} 
        = - \nabla \cdot \left[p_t(\bfz_t)f(t)\bfz_t  -\frac{g^2(t)}{2}\nablazt p_t(\bfz_t)\right],
\end{equation}
it is straightforward to verify that a family of backward diffusion processes with varying $\eta_t$, as described by the following formula, all maintain the same temporal marginal distributions $\{p_t(\bfz_t)\}_{t=0}^T$ as the forward SDE at each time throughout the diffusion process
\begin{equation}
	\label{eq:pre_pf_ode}
		\rmd \bfz_t
		= \left[f(t)\bfz_t - \frac{1+\eta_t^2}{2} g^2(t) \nablazt \log p_t(\bfz_t)\right] \rmd t + \eta_t g(t)  \rmd \bar{\bfw}_t,
\end{equation}
where $\eta_t$ controls the amount of stochasticity and $\bar{\bfw}_t$ denotes the Wiener process when time flows backwards. Notably, there exists a particular deterministic process with the parameter $\eta_t\equiv 0$, termed {\it probability flow ordinary differential equation} (PF-ODE) in the literature~\citep{song2021sde,karras2022edm}. PF-ODE describes a time-dependent vector field, which can directly initialize a generative modeling framework and then induce the associated probability path~\citep{lipman2023flow,liu2023flow,albergo2023stochastic}.
The deterministic nature of ODE offers several benefits in generative modeling, including efficient sampling, unique encoding, and meaningful latent manipulations~\citep{song2021sde,song2021ddim,chen2024trajectory}. We thus choose this mathematical formula to analyze the sampling behavior of diffusion models throughout this paper.

\subsection{Score Estimation and Diffusion Sampling}
\label{subsec:pre_training_sampling}
Simulating the preceding PF-ODE requires having access to the score function $\nablazt \log p_t(\bfz_t)$ ~\citep{hyvarinen2005estimation,lyu2009interpretation}, which is typically estimated with denoising score matching (DSM)~\citep{vincent2011dsm,song2019ncsn,karras2022edm}. Thanks to a profound connection between the score function and the posterior expectation from the perspective of \textit{empirical Bayes}~\citep{robbins1956eb,morris1983parametric,efron2010large,raphan2011least}, we can also train a denoising autoencoder (DAE)~\citep{vincent2008extracting,bengio2013generalized,alain2014regularized} to estimate the conditional expectation $\bbE(\bfz_0|\bfz_t)$, and then convert it to the score function, see more details in Appendix~\ref{subsec:details_score_matching}. We summarize this connection as the following lemma.
\begin{lemma}
    \label{lemma:pre_posterior}
    Let the clean data be $\bfz_0\sim p_d$, and consider a transition kernel that adds Gaussian noise to the data, $p_{0t}(\bfz_t | \bfz_0)=\calN\left(\bfz_t; s_t\bfz_0, \, s^2_t\sigma^2_t\bfI\right)$. Then the score function is related to the posterior expectation by
    \begin{equation}
        \nablazt \log p_t(\bfz_t)=\left(s_t\sigma_t\right)^{-2}\left(s_t\bbE(\bfz_0|\bfz_t)-\bfz_t\right),    
    \end{equation}
    or equivalently, by linearity of expectation, 
    \begin{equation}
        \nablazt \log p_t(\bfz_t)=-\left(s_t\sigma_t\right)^{-1}\bbE_{p_{t0}(\bfz_0|\bfz_t)} \bfeps_t, \quad \quad \bfeps_t = \left(s_t\sigma_t\right)^{-1}(\bfz_t - s_t \bfz_0). 
    \end{equation} 
\end{lemma}
Therefore, we can train a \textit{data-prediction model} $r_{\bftheta}(\bfz_t; t)$ to approximate the posterior expectation $\bbE(\bfz_0|\bfz_t)$, or train a \textit{noise-prediction model} $\bfeps_{\bftheta}(\bfz_t; t)$ to approximate the posterior expectation $\bbE\left(\frac{\bfz_t-s_t\bfz_0}{s_t\sigma_t}|\bfz_t\right)$, and then substitute the score in \eqref{eq:pre_pf_ode} with the learned model for the diffusion sampling process. The DAE objective function of training a data-prediction model $r_{\bftheta}(\bfz_t; t)$ across different noise levels with a weighting function $\lambda(t)$ is 
\begin{equation}
	\label{eq:pre_dae}
	\calL_{\DAE}\left(\bftheta; \lambda(t)\right):=\int_{0}^{T} \lambda(t)\bbE_{\bfz_0 \sim p_d} \bbE_{\bfz_t \sim p_{0t}(\bfz_t|\bfz_0)} \lVert r_{\bftheta}(\bfz_t; t) - \bfz_0  \rVert^2_2 \rmd t.
\end{equation}
\begin{lemma}
	\label{lemma:pre_dae}
	The optimal estimator $r_{\bftheta}^{\star}\left(\bfz_t; t\right)$ for the denoising autoencoder objective, also known as the Bayesian least squares estimator or minimum mean square error (MMSE) estimator, is given by $\bbE\left(\bfz_0|\bfz_t\right)$.
\end{lemma}
In particular, this optimal estimator admits a closed-form solution under the empirical data distribution~\citep{karras2022edm,chen2023geometric,scarvelis2023closed}, as stated in the following lemma.
\begin{lemma}
    \label{lemma:pre_optimal_denoiser}
    Let $\calD \coloneqq \{\bfy_{i} \in \mathbb{R}^d\}_{i \in \calI}$ denote a dataset of $|\calI|$ i.i.d.\ data points drawn from $p_d$. When training a denoising autoencoder with the empirical data distribution $\hat{p}_{d}$, the optimal denoising output is a convex combination of original data points, namely 
    \begin{equation}
    	\label{eq:optimal}
    	\begin{aligned}
            r_{\bftheta}^{\star}(\bfz_t; t)=\min_{r_{\bftheta}} \bbE_{\bfy \sim \hat{p}_d} \bbE_{\bfz_t \sim p_{0t}(\bfz_t|\bfy)} \lVert r_{\bftheta}(\bfz_t; t) - \bfy  \rVert^2_2
    		=\sum_{i} \frac{\exp \left(-\lVert \bfz_t - \bfy_i \rVert^2_2/2\sigma_t^2\right)}{\sum_{j}\exp \left(-\lVert \bfz_t - \bfy_j \rVert^2_2/2\sigma_t^2\right)} \bfy_i,
    	\end{aligned}
    \end{equation}
    where $\hat{p}_{d}(\bfy)$ is the sum of multiple \textit{Dirac delta functions}, \ie, $\hat{p}_{d}(\bfy)=(1/|\calI|)\sum_{i\in\calI}\delta(\lVert \bfy - \bfy_i\rVert)$.
\end{lemma}
In practice, it is assumed that $\nablazt \log p_t(\bfz_t)\approx \left(s_t\sigma_t\right)^{-2}\left(s_t r_{\bftheta}(\bfz_t; t)-\bfz_t\right)$ for a converged model\footnote{We slightly abuse the notation and still denote the converged model as $r_{\bftheta}(\cdot; t)$ hereafter.}, and we can plug it into \eqref{eq:pre_pf_ode} with $\eta_t\equiv 0$ to derive the \textit{empirical} PF-ODE for sampling as follows
\begin{equation}
    \label{eq:pre_epf_ode}
    \begin{aligned}
        \frac{\rmd \bfz_t}{\rmd t} &= \frac{\rmd \log s_t}{\rmd t}\bfz_t - \frac{\rmd \log\sigma_t}{\rmd t} \left(s_t r_{\bftheta}(\bfz_t; t)-\bfz_t\right)\\
        &= \frac{\rmd \log s_t}{\rmd t}\bfz_t + s_t \frac{\rmd \sigma_t}{\rmd t} \bfeps_{\bftheta}(\bfz_t; t). 
    \end{aligned}
\end{equation}
Both the data-prediction model $r_{\bftheta}(\bfz_t; t)$ and the noise-prediction model $\bfeps_{\bftheta}(\bfz_t; t)$ above are widely used in existing works~\citep{ho2020ddpm,song2021ddim,bao2022analytic,karras2022edm,lu2022dpmpp,zhang2023deis,zhou2024fast,chen2024trajectory}. 

Given the empirical PF-ODE~\eqref{eq:pre_epf_ode}, we can synthesize novel samples by first drawing pure noises $\hatz_{t_N} \sim p_n$ as the initial condition, and then numerically solving this equation backward with $N$ steps to obtain a sequence $\{\hatz_{t_n}\}_{n=0}^{N}$ with a certain time schedule $\Gamma=\{t_0=\epsilon, \cdots, t_N=T\}$. We adopt hat notations such as $\hatz_{t_n}$ to denote the samples generated by numerical methods, which differs from the exact solutions denoted as $\bfz_{t_n}$. The final sample $\hatz_{t_0}$ is considered to approximately follow the data distribution $p_{d}$. We designate this sequence as a \textbf{sampling trajectory} generated by the diffusion model. More details about numerical approximation can be found in Section~\ref{subsec:details_numerical}.

\subsection{The Equivalence of Diffusion Models}
\label{subsec:equivalence}
We further demonstrate that diffusion models modeled by linear SDEs are equivalent up to a scaling transformation, provided they share the same SNR function of transition kernels~\eqref{eq:pre_kernel}. In particular, any other model type (\eg, the VP diffusion process) can be transformed into its VE counterparts via the following lemma. 
\begin{lemma}
    \label{lemma:ito_lemma}
	The linear diffusion process defined as \eqref{eq:pre_new_forword_sde} can be transformed into its VE counterpart with the change of variables $\bfx_t=\bfz_t / s_t$, keeping the SNR function unchanged.
\end{lemma}
Similarly, we provide the PF-ODE and its empirical version in terms of the $\bfx$ variable (or say, in the $\bfx$-space) as follows 
\begin{equation}
    \label{eq:epf_ode}
        \rmd \bfx_t = - \sigma_t\nablaxt \log p_t(\bfx_t) \rmd \sigma_t
                    = \frac{\bfx_t-r_{\bftheta}(\bfx_t; t)}{\sigma_t} \rmd \sigma_t
                    = \bfeps_{\bftheta}(\bfx_t; t) \rmd \sigma_t,
\end{equation}
with the score function $\nablaxt \log p_t(\bfx_t)=s_t\nablazt \log p_t(\bfz_t)$, for $t\in [0,T]$.
Because of the above analysis, we can safely remove the drift term in the forward SDE \eqref{eq:pre_new_forword_sde} by transforming them into the VE counterparts without changing the essential characteristics of the underlying diffusion model. 
In the following discussions, we merely focus on the mathematical properties and geometric behaviors of a standardized VE-SDE, \ie, 
\begin{equation}
    \label{eq:pre_vesde}
    \rmd \bfx_t = \sqrt{\rmd \sigma^2_t/\rmd t}\, \rmd \bfw_t, \quad \sigma_t: \bbR\rightarrow \bbR,
\end{equation}
with a pre-defined increasing noise schedule $\sigma_t$. 
Lemma~\ref{lemma:ito_lemma} guarantees the applicability of our conclusions to any other types of linear diffusion processes, including the typical flow matching-based models~\citep{liu2023flow,lipman2023flow,albergo2023stochastic}. In this case, the \textbf{sampling trajectory} is denoted as $\{\hatx_{t_n}\}_{n=0}^{N}$ with the time schedule $\Gamma=\{t_0=\epsilon, \cdots, t_N=T\}$ and the initial noise is denoted as $\hatx_{t_N}\sim p_n=\calN(0, \sigma_T^2I)$.


\subsection{Conditional and Latent Diffusion Models}

It is straightforward to extend the above framework of unconditional diffusion models into the conditional variants~\citep{song2021sde,dhariwal2021diffusion,rombach2022ldm}. Given the class or text-based condition $\bfc$, the modeled marginal distributions become $p_t(\bfx_t|\bfc)$ (or $p_t(\bfz_t|\bfc)$ in the $\bfz$-space) instead of the original $p_t(\bfx_t)$, 
and the sampling process relies on the learned conditional score $\nablaxt \log p_t(\bfx_t|\bfc)$
at each time. In general, discrete texts are first mapped into a continuous text embedding space~\citep{nichol2022glide,rombach2022ldm,saharia2022photorealistic}, which distinguishes text-conditional diffusion models from the diffusion models conditioned on discrete class labels. Another extension from the practical consideration is performing the diffusion process in a low-dimensional latent space rather than the original high-dimensional data space~\citep{vahdat2021lsgm,rombach2022ldm}. With the help of an autoencoder structure, latent diffusion models significantly reduce computational demand and scale up to high-resolution generation.

In the following empirical analysis (Section~\ref{sec:trajectory_regularity}), we will demonstrate that strong trajectory regularity is widely present in \textit{unconditional}, \textit{class-conditional}, and \textit{text-conditional} diffusion models. This observation motivates us to investigate the underlying mechanism behind (Section~\ref{sec:trajectory_theory}) and to develop an improved algorithm for sampling acceleration (Section~\ref{sec:trajectory_algorithm}). 
\section{Geometric Regularity in Deterministic Sampling Dynamics}
\label{sec:trajectory_regularity}

As mentioned in Section~\ref{sec:intro}, the sampling trajectories of diffusion-based generative models under the PF-ODE framework exhibit a certain regularity in their shapes, regardless of the specific content generated. 
To better demonstrate this concept, we undertake a series of empirical studies in this section, covering unconditional generation (pixel space) on CIFAR-10~\citep{krizhevsky2009learning}, class-conditional generation (pixel space) on ImageNet~\citep{russakovsky2015ImageNet}, and text-conditional generation (latent space) with Stable Diffusion v1.5~\citep{rombach2022ldm}. The spatial resolutions used for these diffusion processes are $32\times 32$, $64\times 64$, $64\times 64$, respectively. Given the complexity of visualizing the entire sampling trajectory and analyzing its geometric characteristics in the original high-dimensional space, we develop subspace projection techniques to better capture the intrinsic structure of diffusion models.

\begin{figure}[t]
    \centering
    \begin{subfigure}[t]{0.5\textwidth}
        \includegraphics[width=\textwidth]{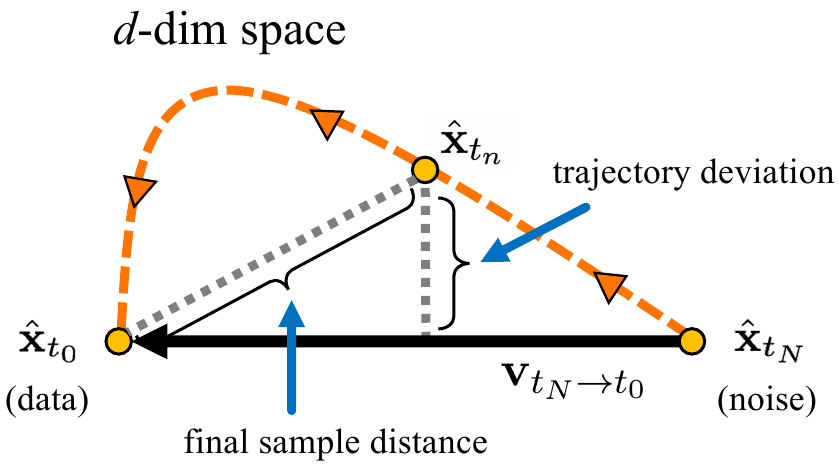}
        \caption{1-D reconstruction.}
        \label{fig:vis_sketch_1D}
    \end{subfigure}
    \quad
    \begin{subfigure}[t]{0.35\textwidth}
        \includegraphics[width=\textwidth]{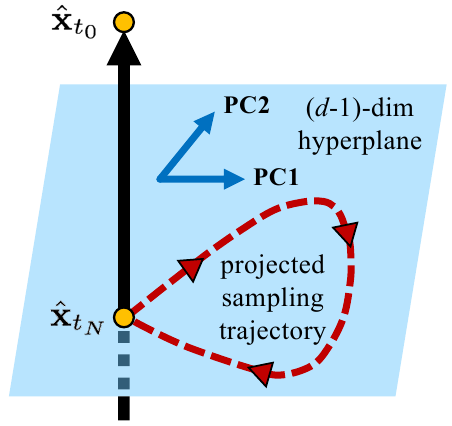}
        \caption{3-D reconstruction.}
        \label{fig:vis_sketch_3D}
    \end{subfigure}
    \caption{Illustration of subspace projection techniques. The deterministic sampling trajectory begins with an initial noise $\hatx_{t_N}$ and progresses to the synthesized data $\hatx_{t_0}$.  
    (a) The trajectory deviation equals the reconstruction error when the $d$-dimensional point of the sampling trajectory is projected onto the displacement vector $\bfv_{{t_N}\rightarrow {t_0}}\coloneq \hatx_{t_0}-\hatx_{t_N}$. (b) We adopt $\bfv_{{t_N}\rightarrow {t_0}}$ and several top principal components (PCs) from its $(d-1)$-dimensional orthogonal complement to approximate the original $d$-dimensional sampling trajectory.
    }
    \vspace{-0.5em}
\end{figure}

\subsection{One-Dimensional Projection}
\label{subsec:one_dim}
We first examine the \textit{trajectory deviation} from the straight line connecting the two endpoints, which serves to assess the linearity of the sampling trajectory. A sketch of this computation is provided in Figure~\ref{fig:vis_sketch_1D}. This approach allows us to align and collectively observe the general behaviors of all trajectories. 
Specifically, we denote the \textit{displacement vector} between the two endpoints as $\bfv_{{t_N}\rightarrow {t_0}}\coloneq \hatx_{t_0}-\hatx_{t_N}$, and compute the trajectory deviation as the perpendicular Euclidean distance ($L^2$) from each intermediate sample $\hatx_{t_n}$ to the vector $\bfv_{{t_N}\rightarrow {t_0}}$, \ie, 
$d_{\text{td}}\coloneq \sqrt{\lVert \bfv_{{t_n}\rightarrow {t_0}}\rVert_2^2-
    \left(
    \bfv_{{t_n}\rightarrow {t_0}}^{T}\cdot \bfv_{{t_N}\rightarrow {t_0}}
    /\lVert \bfv_{{t_N}\rightarrow {t_0}} \rVert_2
    \right)^2
    }$,    
Additionally, we calculate the $L^2$ distance between each intermediate sample $\hatx_{t_n}$ and the final sample $\hatx_{t_0}$, denoted as $d_{\text{fsd}}\coloneq \lVert \hatx_{t_n}-\hatx_{t_0}\rVert_2$, and refer to it as the \textit{final sample distance}. 
\begin{figure}[t]
    \centering
    \begin{subfigure}[t]{0.48\textwidth}
        \includegraphics[width=\columnwidth]{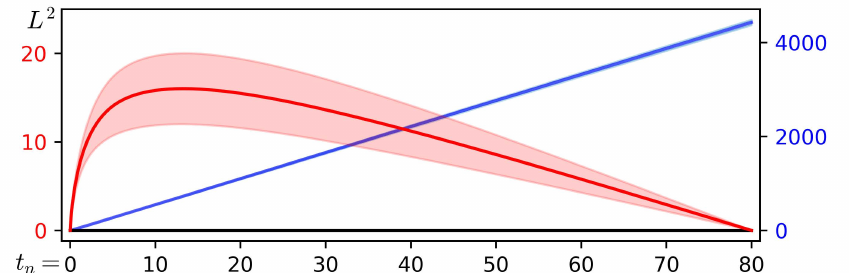}
        \caption{Unconditional generation (CIFAR-10).}
        \label{fig:deviation_cifar10}
    \end{subfigure}
    \begin{subfigure}[t]{0.48\textwidth}
        \includegraphics[width=\columnwidth]{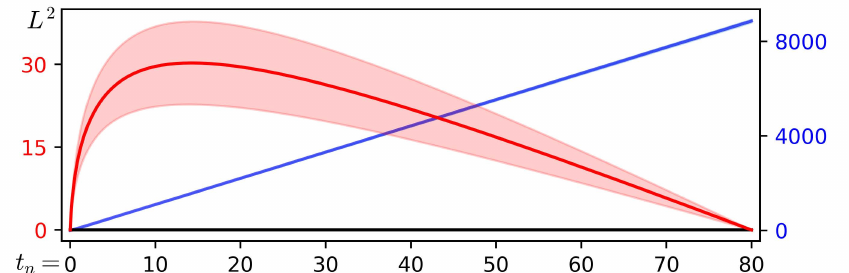}
        \caption{Class-conditional generation (ImageNet).}
        \label{fig:deviation_imagenet}
    \end{subfigure}
    \begin{subfigure}[t]{0.48\textwidth}
        \includegraphics[width=\columnwidth]{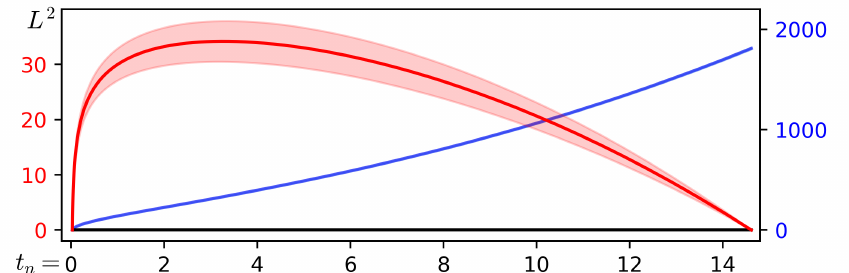}
        \caption{Text-conditional generation (SDv1.5). 
        }
        \label{fig:deviation_sd}
    \end{subfigure}
    \caption{Results of the 1-D trajectory projection. The sampling trajectory exhibits an extremely small trajectory deviation (red curve) compared to the final sample distance (blue curve) in the sampling process. Each trajectory is simulated with the Euler method and 100 number of function evaluations (NFEs). The reported average and standard deviations are based on 5,000 randomly generated sampling trajectories, considering variations in initial noises, class labels, and text prompts. 
    }
    \label{fig:vis_deviation}
\end{figure}

\begin{figure}[t!]
    \centering
    \begin{subfigure}[t]{0.51\textwidth}
    	\includegraphics[width=\columnwidth]{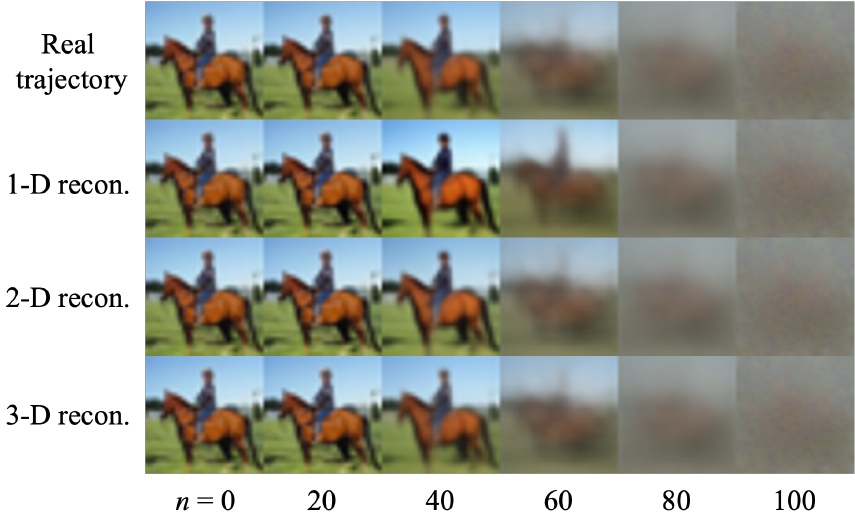}
    	\caption{Visual comparison (CIFAR-10).}
    	\label{fig:recon_visual_cifar10}
    \end{subfigure}
    \begin{subfigure}[t]{0.255\textwidth}
    	\includegraphics[width=\columnwidth]{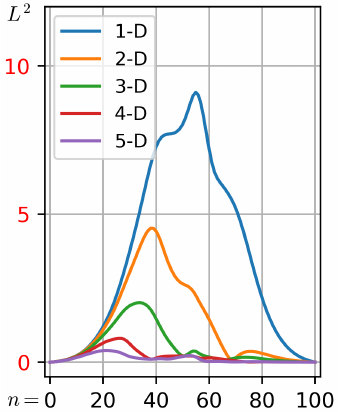}
    	\caption{Reconstruction error.}
    	\label{fig:recon_l2_cifar10}
    \end{subfigure}
    \begin{subfigure}[t]{0.20\textwidth}
        \includegraphics[width=\textwidth]{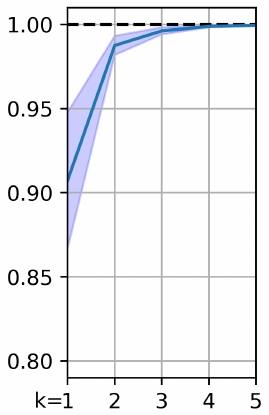}
        \caption{PCA ratio.}
        \label{fig:recon_pca_cifar10}
    \end{subfigure}
    \begin{subfigure}[t]{0.51\textwidth}
        \includegraphics[width=\columnwidth]{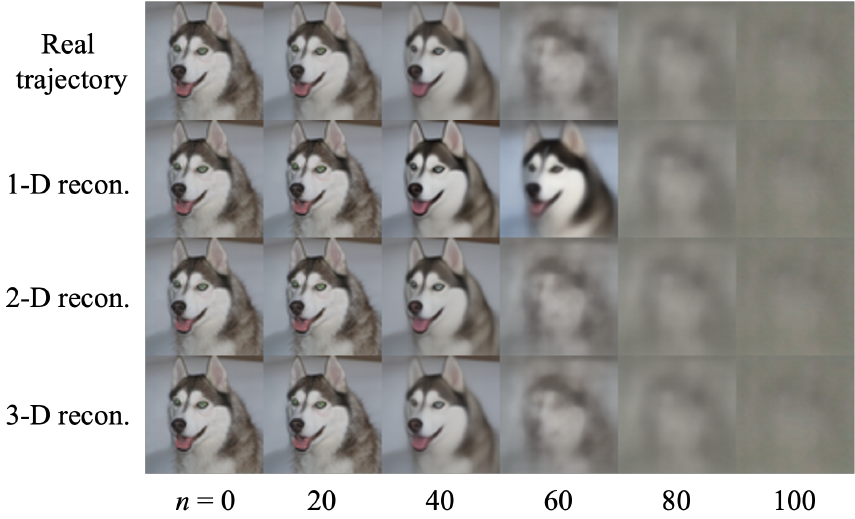}
        \caption{Visual comparison (ImageNet).}
        \label{fig:recon_visual_imagenet}
    \end{subfigure}
    \begin{subfigure}[t]{0.255\textwidth}
        \includegraphics[width=\columnwidth]{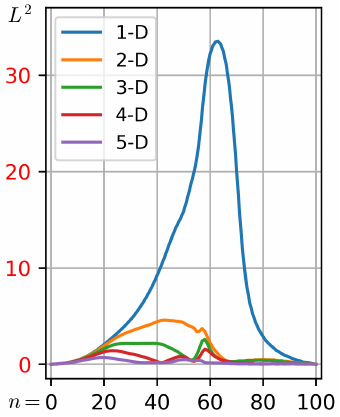}
        \caption{Reconstruction error.}
        \label{fig:recon_l2_imagenet}
    \end{subfigure}
    \begin{subfigure}[t]{0.20\textwidth}
        \includegraphics[width=\textwidth]{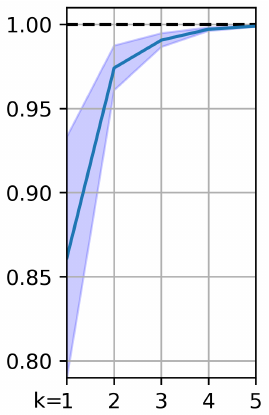}
        \caption{PCA ratio.}
        \label{fig:recon_pca_imagenet}
    \end{subfigure}
    \begin{subfigure}[t]{0.51\textwidth}
        \includegraphics[width=\columnwidth]{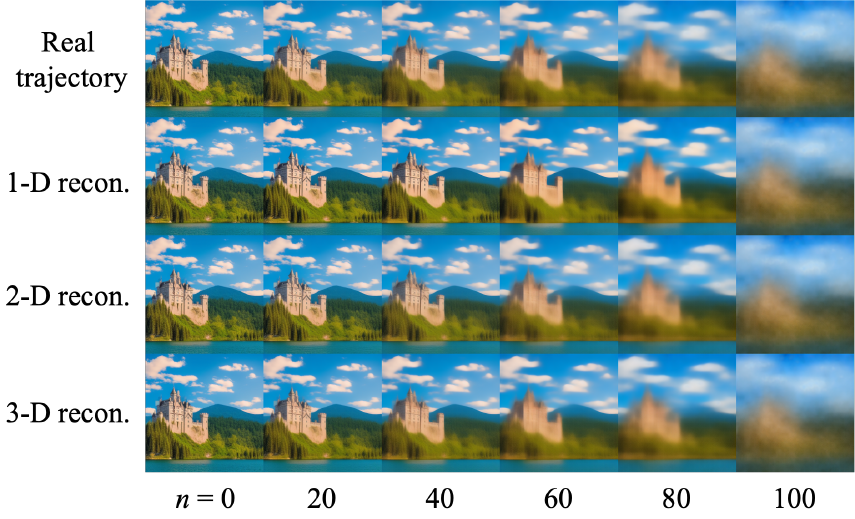}
        \caption{Visual comparison (SDv1.5).}
        \label{fig:recon_visual_sd}
    \end{subfigure}
    \begin{subfigure}[t]{0.255\textwidth}
        \includegraphics[width=\columnwidth]{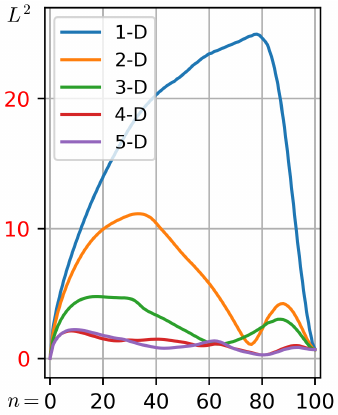}
        \caption{Reconstruction error.}
        \label{fig:recon_l2_sd}
    \end{subfigure}
    \begin{subfigure}[t]{0.20\textwidth}
        \includegraphics[width=\textwidth]{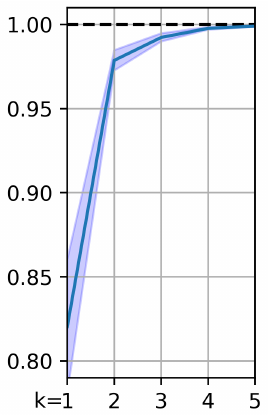}
        \caption{PCA ratio.}
        \label{fig:recon_pca_sd}
    \end{subfigure}
    \caption{
    Visual comparison of trajectory reconstruction for (a) unconditional, (d) class-conditional (generated by EDM~\citep{karras2022edm}), and (g) text-conditional generation (generated by SDv1.5~\citep{rombach2022ldm}). The real sampling trajectories (top row) are reconstructed using $\bfv_{{t_N} \rightarrow {t_0}}$ (1-D recon.), along with their top 1 or 2 principal components (2-D or 3-D recon.). To amplify visual differences, we present the denoising outputs of these trajectories. 
    (b/e/h) The $L^2$ distance between the real trajectory samples and their reconstructed counterparts is computed up to 5-D reconstruction. 
    (c/f/i) The variance explained by the top $k$ principal components is reported as the ratio of the sum of the top $k$ eigenvalues to the sum of all eigenvalues.
    }
    \label{fig:vis_recon}
\end{figure}
The empirical results of trajectory deviation $d_{\text{td}}$ and final sample distance $d_{\text{fsd}}$ are depicted as the red curves and blue curves in Figure~\ref{fig:vis_deviation}, respectively. Note that we use \textit{sampling time as the horizontal axis}, which allows all sampling trajectories to be aligned and compared both within and across different time slices.
From Figures~\ref{fig:deviation_cifar10}\,-\,\ref{fig:deviation_imagenet}, we observe that the sampling trajectory's deviation gradually increases from $t=80$ to approximately $t=10$, then swiftly diminishes as it approaches the final samples. This pattern suggests that initially, each sample might be influenced by various modes, experiencing significant impact, but later becomes strongly guided by its specific mode after a certain turning point. This behavior supports the heuristic approach of arranging time intervals more densely near the minimum timestamp and sparsely towards the maximum one \citep{song2021ddim,karras2022edm,song2023consistency,chen2024trajectory}. 
However, when we consider the ratio of the maximum deviation to the endpoint distance in Figure~\ref{fig:deviation_cifar10}\,-\,\ref{fig:deviation_imagenet}, we find that the trajectory deviation is remarkably slight (e.g., $30/8800 \approx 0.0034$ for ImageNet), indicating a pronounced straightness. Additionally, the generated samples along the sampling trajectory tend to move monotonically from their initial points toward their final points (as illustrated by the blue curves). Similar results can be found for the text-conditional generation in the latent space, as shown in Figure~\ref{fig:deviation_sd}.

The trajectory deviation also reflects the reconstruction error if we project all $d$-dimensional points of the sampling trajectory onto the displacement vector $\bfv_{{t_N}\rightarrow {t_0}}$. As demonstrated in Figure~\ref{fig:vis_recon}, the one-dimensional ($1$-D) approximation proves inadequate, leading to a significant deviation from the actual trajectory both in terms of visual comparison and quantitative results. These observations imply that while all trajectories share a similar macro-structure, the $1$-D projection cannot accurately capture the full trajectory structure, probably due to the failure of modeling rotational properties. Therefore, we further develop a multi-dimensional subspace projection technique, as detailed below.


\subsection{Multiple-Dimensional Projections}
\label{subsec:multi_dim}
We then implement principal component analysis (PCA) on the orthogonal complement of the displacement vector $\bfv_{{t_N}\rightarrow {t_0}}$, which assists in assessing rotational properties of the sampling trajectory. A sketch of this computation is provided in Figure~\ref{fig:vis_sketch_3D}. This $(d-1)$-D orthogonal space relative to $\bfv_{{t_N}\rightarrow {t_0}}$ is denoted as $\mathcal{V}=\{\bfu: \bfu^T \bfv_{{t_N}\rightarrow {t_0}}=0, \forall\, \bfu\in \bbR^{d}\}$. 
We begin by projecting each $d$-D sampling trajectory into $\mathcal{V}$, followed by conducting PCA. 

As illustrated in Figure~\ref{fig:vis_recon}, the 2-D approximation using $\bfv_{{t_N}\rightarrow {t_0}}$ and the first principal component markedly narrows the visual discrepancy with the real trajectory, thereby reducing the $L^2$ reconstruction error. This finding suggests that all points in each $d$-D sampling trajectory diverge slightly from a 2-D plane. Consequently, the tangent and normal vectors of the sampling trajectory can be effectively characterized in this manner.
By incorporating an additional principal component, we enhance our ability to capture the torsion of the sampling trajectory, thereby increasing the total explained variance to approximately 85\% (Figures~\ref{fig:recon_pca_cifar10},\,\ref{fig:recon_pca_imagenet},\,\ref{fig:recon_pca_sd}). 
This improvement allows for a more accurate approximation of the actual trajectory and further reduces the $L^2$ reconstruction error (Figures~\ref{fig:recon_l2_cifar10},\,\ref{fig:recon_l2_imagenet},\,\ref{fig:recon_l2_sd}). In practical terms, this level of approximation effectively captures all the visually pertinent information, with the deviation from the real trajectory being nearly indistinguishable (Figures~\ref{fig:recon_visual_cifar10},\,\ref{fig:recon_visual_imagenet},\,\ref{fig:recon_visual_sd}). Consequently, we can confidently utilize a 3-D subspace, formed by two principal components and the displacement vector $\bfv_{{t_N} \rightarrow {t_0}}$, to understand the geometric structure of high-dimensional sampling trajectories. 
\begin{figure}[t!]
    \centering
    \begin{subfigure}[t]{0.92\textwidth}
        \includegraphics[width=\columnwidth]{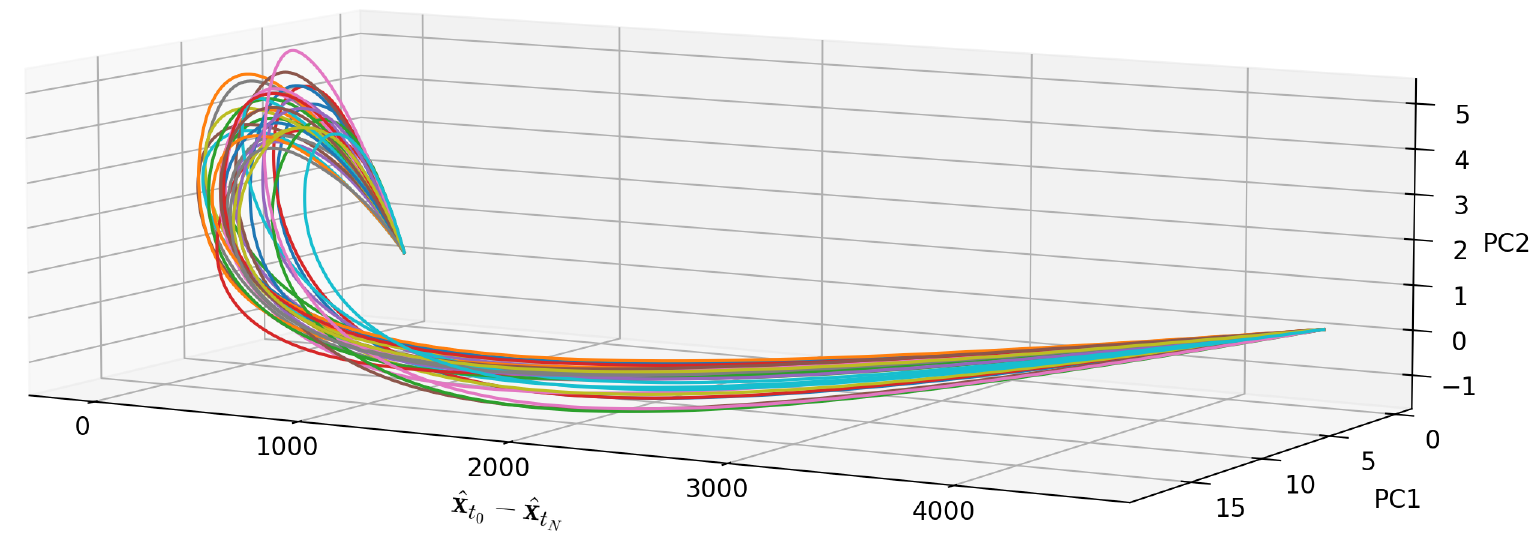}
        \caption{Unconditional generation (CIFAR10, pixel space).}
    \end{subfigure}
    \begin{subfigure}[t]{0.92\textwidth}
        \includegraphics[width=\columnwidth]{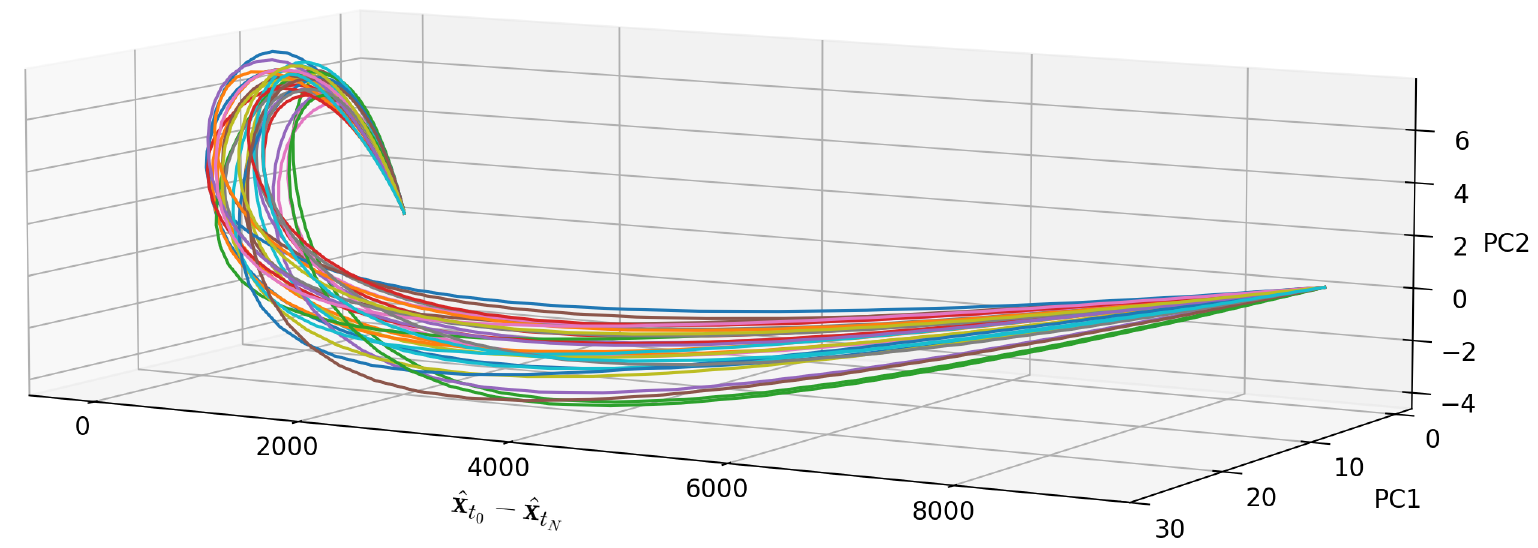}
        \caption{Class-conditional generation (ImageNet, pixel space).}
    \end{subfigure}
    \begin{subfigure}[t]{0.92\textwidth}
        \includegraphics[width=\columnwidth]{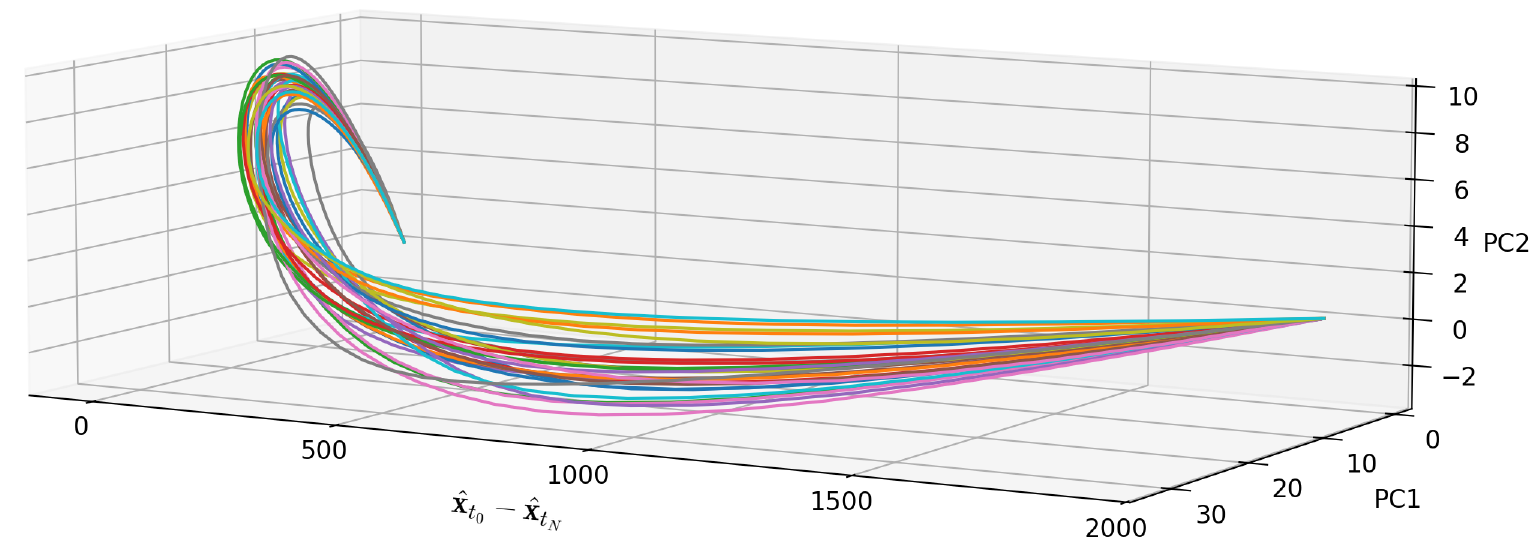}
        \caption{Text-conditional generation (SDv1.5, latent space).}
    \end{subfigure}
    \caption{We project 30 sampling trajectories generated by (a) unconditional, (b) class-conditional, and (c) text-conditional diffusion models into the 3-D subspaces. Each trajectory is simulated with the Euler method and 100 number of function evaluations (NFEs). These trajectories are first aligned to the direction of the displacement vector $\hatx_{t_0}-\hatx_{t_N}$ (this direction is slightly different for each sample), and then projected to the top 2 principal components in the orthogonal space to $\hatx_{t_0}-\hatx_{t_N}$. See texts for more details.}
    \label{fig:traj_3d}
\end{figure}

Expanding on this understanding, in Figure~\ref{fig:traj_3d}, we present a visualization of randomly selected sampling trajectories created by diffusion models under various generation settings. Note that the scale along the axis corresponding to $\hatx_{t_0}-\hatx_{t_N}$ is orders of magnitude larger than those of the other two principal components. 
Since we focus on the geometric shape regularity of sampling trajectories rather than their absolute locations, we align all trajectories via orthogonal transformations to eliminate arbitrary orientation variations. These transformations, including rotations and reflections, are determined by solving the classic \textit{Orthogonal Procrustes Problem}~\citep{hurley1962procrustes,schonemann1966generalized,golub2013matrix}, after which we visualize the calibrated trajectories. Specifically, we represent each projected sampling trajectory in the 3-D subspace as a matrix $\bfR\in \bbR^{N\times 3}$, where each row corresponds to a sample coordinate $\bfr\in \bbR^3$ at a particular time step, and $N$ is the total number of generated samples in the trajectory. Then, we seek an orthogonal transformation $\bfO \in \bbR^{3\times 3}$ to minimize the Frobenius norm of the residual error matrix $\bfE$ when aligning one matrix $\bfR\in \bbR^{N\times 3}$ with a reference projected trajectory matrix $\widetilde{\bfR}\in \bbR^{N\times 3}$ of the same dimensions. 
This optimization problem is formulated as 
\begin{equation}
    \bfE=\min_{\bfO} \|\widetilde{\bfR} - \bfR\bfO  \|^2_{{\text{F}}},\quad \textit{s.t.} \quad \bfO^{\text{T}}\bfO=\bfI_{3\times 3}.    
\end{equation}
The optimal solution can be derived using the method of Lagrange multipliers, yielding $\bfO^{\star}=\bfU\bfV^\text{T}$, where $\bfU$ and $\bfV$ are obtained via singular value decomposition (SVD) of $\bfR^{\text{T}}\widetilde{\bfR}$, \ie, $\bfR^{\text{T}}\widetilde{\bfR}=\bfU \boldsymbol{\Sigma} \bfV^{\text{T}}$. As shown in Figure~\ref{fig:vis_deviation}, we adopt \textit{sampling time as the horizontal axis} to better observe the trajectory shapes by aligning different trajectories across time slices. Here, time is scaled by a factor of $\sqrt{d}$ to preserve relative magnitudes between axes. In other words, the visualization remains almost unchanged whether we use the displacement vector $\bfv_{t_{N}\rightarrow t_{0}}$ or scaled time as the first axis. Moreover, in this case, an orthogonal transformation $\widetilde{\bfO} \in \bbR^{2\times 2}$ that fixes the first axis suffices for alignment. 

As a result, the calibrated trajectories depicted in Figure~\ref{fig:traj_3d} largely adhere to the straight line connecting their endpoints, corroborating the small trajectory deviation observed in our previous findings (Figure~\ref{fig:vis_deviation}). Furthermore, Figure~\ref{fig:traj_3d} accurately depicts the sampling trajectory's behavior, showing its gradual departure from the osculating plane during sampling. Interestingly, each trajectory consistently exhibits a simple, approximately \textit{linear-nonlinear} structure. This reveals a strong regularity in all sampling trajectories, independent of the specific content generated and variations in initial noises, class labels, and text prompts. 

\subsection{Three-Dimensional Projection Revisited}
\label{subsec:three_dim}
Given the strong trajectory regularity of deterministic diffusion sampling manifested in the three-dimensional Euclidean space, as shown in Figure~\ref{fig:traj_3d}, we further resort to a differential geometry tool known as the \textit{Frenet–Serret formulas}~\citep{do2016differential} to precisely characterize geometric properties of the projected sampling trajectory. 

We denote the projected sampling trajectory consisting of $N$ discrete points in the 3-D subspace as $\bfr(\xi)$, where $\xi=T-t$ with $T$ as the terminal time and $t$ as the sampling time (see Section~\ref{sec:preliminaries} for detailed notations). Thanks to our proposed subspace projection techniques (Sections~\ref{subsec:one_dim}\,-\,\ref{subsec:multi_dim}), each projected sampling trajectory keeps starting from pure noise $\bfr(0)$ and ends at the synthesized data $\bfr(T)$. The \textit{arc-length} of this spatial curve is denoted as $s(\xi)=\int_{0}^{\xi} \|\bfr^{\prime}(u)\| \rmd u$, which is a strictly monotone increasing function. We then use the arc-length $s$ to parameterize the spatial curve, and define the \textit{tangent unit vector}, \textit{normal unit vector}, and \textit{binormal unit vector} as $\bfT(s)\coloneq\bfr^{\prime}(s)$, $\bfN(s)\coloneq\bfr^{\prime\prime}(s)/\|\bfr^{\prime\prime}(s)\|$, and $\bfB(s)\coloneq\bfT(s) \times \bfN(s)$, respectively. These three unit vectors are interrelated, and their relationship is characterized by the well-known Frenet–Serret formulas listed below,
\begin{equation}
    \frac{\rmd \bfT(s)}{\rmd s}=\kappa(s)\bfN(s), \quad \frac{\rmd \bfN(s)}{\rmd s}=-\kappa(s)\bfT(s)+\tau(s)\bfB(s), \quad \frac{\rmd \bfB(s)}{\rmd s}=-\tau(s)\bfN(s),
\end{equation}
where the curvature $\kappa(s)$ and the torsion $\tau(s)$ measure the deviations of a spatial curve $\bfr(s)$ from being a straight line and from being planar, respectively. In practice, we generally employ the following expressions for numerical calculation.\footnote{These two equations also hold for the spatial curve $\bfr(\xi)$ parametrized by $\xi$.}
\begin{equation}
    \kappa(s) = \left[\bfr^{\prime}(s)\times \bfr^{\prime\prime}(s)\right] / \|\bfr^{\prime}(s)\|^3, \quad\tau(s)=\left[(\bfr^{\prime}(s)\times \bfr^{\prime\prime}(s))\cdot \bfr^{\prime\prime\prime}(s)\right] / \|\bfr^{\prime}(s)\times \bfr^{\prime\prime}(s)\|^2.
\end{equation}
Specifically, for each discrete point on the spatial curve, we employ its surrounding points within a given window size to estimate the first-, second-, and third-order derivatives based on the third-order Taylor expansion~\citep{lewiner2005curvature}. With the help of least squares fitting using an appropriate window size, the torsion and curvature functions of each projected sampling trajectory can be reliably estimated with small reconstruction errors. 
\begin{figure}[t!]
    \centering
    \begin{subfigure}[t]{0.32\textwidth}
        \includegraphics[width=\columnwidth]{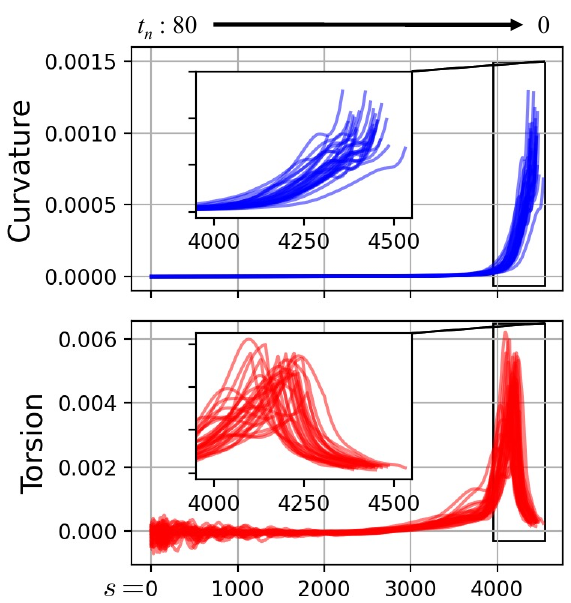}
        \caption{Unconditional generation (CIFAR-10), window size $=$ 101.}
        \label{fig:traj_stats_cifar10}
    \end{subfigure}
    \begin{subfigure}[t]{0.32\textwidth}
        \includegraphics[width=\columnwidth]{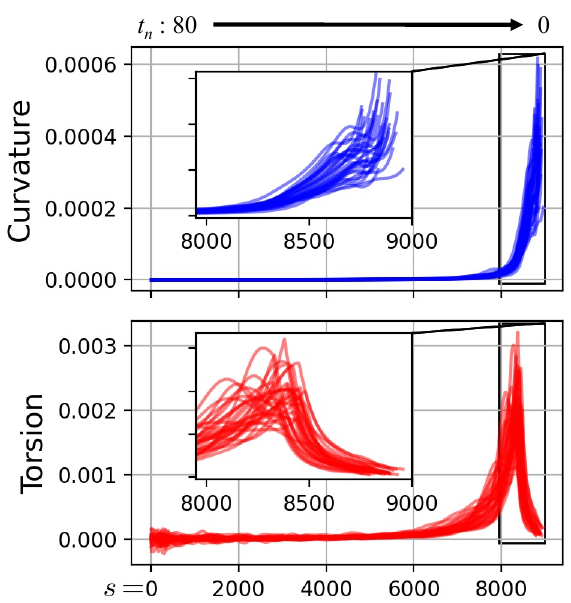}
        \caption{Class-conditional generation (ImageNet), window size $=$ 101.}
        \label{fig:traj_stats_imagenet64}
    \end{subfigure}
    \begin{subfigure}[t]{0.32\textwidth}
        \includegraphics[width=\columnwidth]{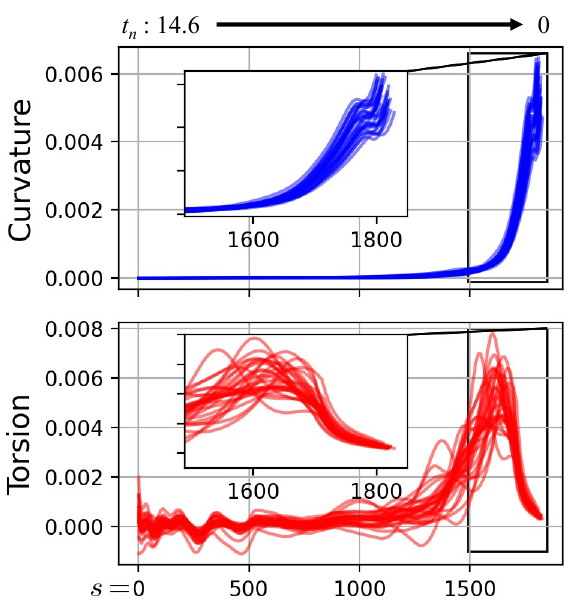}
        \caption{Text-conditional generation (SDv1.5), window size $=$ 151.}
        \label{fig:traj_stats_sd}
    \end{subfigure}
    \caption{We compute the curvature and torsion functions of each 3-D spatial curve that approximates the original high-dimensional sampling trajectory. The Euler method with 1,000 NFEs is employed to faithfully simulate 30 sampling trajectories with (a) unconditional, (b) class-conditional, and (c) text-conditional diffusion models. The geometric properties of the curves are then estimated using least-squares fitting.}
    \label{fig:traj_stats}
\end{figure}

As shown in Figure~\ref{fig:traj_stats}, the projected sampling trajectory remains nearly straight for a large portion ($\approx 80\%$) of the entire sampling process, with both curvature and torsion staying close to zero, consistent with the analysis in Sections~\ref{subsec:one_dim}\,-\,\ref{subsec:multi_dim}. 
For example, in class-conditional generation (Figure~\ref{fig:traj_stats_imagenet64}), the curvature and torsion gradually increase once the arc-length $s$ exceeds around $7,000$. The torsion reaches its peak at around $8,250$, corresponding to a sampling time of around four, and then decreases back toward zero. Note that Figure~\ref{fig:traj_stats} serves as an important complement to Figure~\ref{fig:traj_3d}, providing a more faithful description of the rotational structure of the 3-D spatial curve throughout the sampling process. In contrast, the significant differences in axis magnitudes in Figure~\ref{fig:traj_3d} may visually distort the actual trajectory shape. Furthermore, according to the \textit{fundamental theorem of curves} in differential geometry~\citep{do2016differential}, the shape of any regular curve with non-zero curvature in 3-D space is fully determined by its curvature and torsion. The highly similar evolution patterns of curvature and torsion observed in Figure~\ref{fig:traj_stats} provides strong evidence of trajectory regularity, suggesting that sampling trajectories are congruent across different diffusion models and generation conditions. 
\section{Understanding the Deterministic Sampling Dynamics}
\label{sec:trajectory_theory}

In this section, we investigate several properties of deterministic sampling in diffusion models and provide explanations for the trajectory regularity observed in the previous section. We begin by noting that, beyond the explicit sampling trajectory, there exists an additional but often overlooked trajectory that determines each intermediate sampling point through a convex combination (Section~\ref{subsec:denoising_trajectory}). Under the empirical data distribution, we establish a connection between the closed-form solutions of these trajectories and the classic mean-shift algorithm (Section~\ref{subsec:theoretical_analysis}). We then present a detailed theoretical analysis of the sampling dynamics, revealing stepwise rotation and monotone likelihood increase as local behaviors, and characterizing the overall trajectory as a linear–nonlinear–linear mode-seeking path of approximately constant length as a global behavior. Finally, we revisit trajectory regularity under the Gaussian data assumption (Section~\ref{subsec:theoretical_analysis_gaussian}).


\subsection{Implicit Denoising Trajectory}
\label{subsec:denoising_trajectory}

Given a parametric diffusion model with the \textit{denoising output} $r_{\bftheta}(\cdot; \cdot)$, the sampling trajectory is simulated by numerically solving the empirical PF-ODE~\eqref{eq:epf_ode}. Meanwhile, an implicitly coupled sequence $\{r_{\bftheta}(\hatx_{t_n}, {t_n})\}_{n=1}^{N}$ is formed as a by-product. We designate this sequence, or simplified to $\{r_{\bftheta}(\hatx_{t_n})\}_{n=1}^{N}$ if there is no ambiguity, as the \textit{denoising trajectory}. We then rearrange the empirical PF-ODE~\eqref{eq:epf_ode} as $r_{\bftheta}(\bfx_t; t)=\bfx_t - \sigma_t (\rmd \bfx_t/\rmd \sigma_t)$, and take the derivative of both sides 
to obtain the differential equation of the denoising trajectory
\begin{equation}
    \label{eq:epf_ode_denoising}
    \rmd r_{\bftheta}(\bfx_t; t)/\rmd \sigma_t = - \sigma_t \left[\rmd^2 \bfx_t/\rmd \sigma^2_t\right].
\end{equation}
This equation, although not directly applicable for simulation, reveals that the denoising trajectory encapsulates the curvature information of the associated sampling trajectory. The following proposition reveals how these two trajectories are inherently related.
\begin{proposition}
    \label{prop:convex}
	Given the probability flow ODE~\eqref{eq:epf_ode} and a current position $\hatx_{t_{n+1}}$, $n\in[0, N-1]$ in the sampling trajectory, the next position $\hatx_{t_{n}}$ predicted by a $k$-th order Taylor expansion with the time step size $\sigma_{t_{n+1}}-\sigma_{t_n}$ is 
	\begin{equation}
    \label{eq:convex}
        \begin{aligned}
		\hatx_{t_{n}}&=\frac{\sigma_{t_n}}{\sigma_{t_{n+1}}} \hatx_{t_{n+1}} +  \frac{\sigma_{t_{n+1}}-\sigma_{t_n}}{\sigma_{t_{n+1}}} \calR_{\bftheta}(\hatx_{t_{n+1}}),
        \end{aligned}
	\end{equation}
which is a convex combination of $\hatx_{t_{n+1}}$ and the generalized denoising output
\begin{equation}
    \calR_{\bftheta}(\hatx_{t_{n+1}})=r_{\bftheta}(\hatx_{t_{n+1}})- \sum_{i=2}^{k}\frac{1}{i!}\frac{\rmd^{(i)} \bfx_t}{\rmd \sigma_t^{(i)}}\Big|_{\hatx_{t_{n+1}}}\sigma_{t_{n+1}}(\sigma_{t_n} - \sigma_{t_{n+1}})^{i-1}.
\end{equation}
In particular, we have $\calR_{\bftheta}(\hatx_{t_{n+1}})=r_{\bftheta}(\hatx_{t_{n+1}})$ for the Euler method ($k=1$), and $\calR_{\bftheta}(\hatx_{t_{n+1}})=r_{\bftheta}(\hatx_{t_{n+1}})+\frac{\sigma_{t_{n}}-\sigma_{t_{n+1}}}{2}\frac{\rmd r_{\bftheta}(\hatx_{t_{n+1}})}{\rmd \sigma_t}$ for second-order numerical methods ($k=2$). 
\end{proposition}
\begin{corollary}
    \label{cor:jump}
	The denoising output $r_{\bftheta}(\hatx_{t_{n+1}})$ reflects the prediction made by a single Euler step from $\hatx_{t_{n+1}}$ with the time step size $\sigma_{t_{n+1}}$.
\end{corollary}
\begin{corollary}
     Each second-order ODE-based accelerated sampling method corresponds to a specific first-order finite difference of $\rmd r_{\bftheta}(\hatx_{t_{n+1}})/\rmd \sigma_t$. 
\end{corollary}
The ratio $\sigma_{t_{n}}/\sigma_{t_{n+1}}$ in~\eqref{eq:convex} quantifies the relative preference for maintaining the current position versus transitioning to the generalized denoising output at $t_{n+1}$. Since the denoising output starts from a spurious mode and gradually converges toward a true mode, a reasonable strategy is to decrease this weight progressively during sampling. From this perspective, different time-scheduling functions designed for diffusion sampling, such as uniform, quadratic, or polynomial schedules \citep{song2021ddim,lu2022dpm,karras2022edm,chen2024trajectory}, essentially represent various weighting functions. This interpretation also motivates the direct search for proper weights to further enhance visual quality (Section~\ref{sec:trajectory_algorithm}). 

We primarily focus on the Euler method to simplify subsequent discussions, though these insights can be readily extended to higher-order methods. The trajectory behavior in the continuous-time scenario is similarly discernible through examining the sampling process with an infinitesimally small Euler step. 

\subsection{Theoretical Analysis of the Trajectory Structure}
\label{subsec:theoretical_analysis}
In this section, we leverage the well-established closed-form solutions under the empirical data distribution~\citep{karras2022edm,chen2023geometric,scarvelis2023closed} to connect deterministic sampling dynamics with the classic mean-shift algorithm~\citep{fukunaga1975estimation,cheng1995mean,comaniciu2002mean}, and characterize their \textit{local} and \textit{global} behaviors. 

As discussed in Section~\ref{subsec:pre_training_sampling}, once a diffusion model converges to the optimum, \ie, $\forall t,\, r_{\bftheta}(\bfx_t; t)\rightarrow r_{\bftheta}^{\star}(\bfx_t; t)=\bbE(\bfx_0|\bfx_t)$, 
it captures the score $\nablaxt \log p_t(\bfx_t)$ across different noise levels. The exact formula for the \textit{optimal denoising output} is given in~\eqref{eq:optimal}, under which both the sampling trajectory and the denoising trajectory admit closed-form solutions. In this case, the marginal density at each time step of the forward diffusion process~\eqref{eq:pre_vesde} becomes a Gaussian kernel density estimate (KDE) with bandwidth $\sigma_t^2$, \ie, $\hat{p}_t(\bfx_t)=\int p_{0t}(\bfx_t|\bfy)\hat{p}_{d}(\bfy) = (1/|\calI|)\sum_{i\in\calI}\calN\left(\bfx_t;\bfy_i, \sigma_t^2\bfI\right)$. Intuitively, the forward process can be viewed as an expansion in both magnitude and manifold: the training data samples leave the original small-magnitude low-rank manifold and spread onto a large-magnitude high-rank manifold. Consequently, the squared magnitude of a noisy sample is expected to exceed that of the original sample. As the dimension $d\rightarrow \infty$, this expansion occurs with probability one and the isotropic Gaussian noise becomes approximately uniformly distributed on the sphere~\citep{vershynin2018high,chen2023geometric}. In contrast, the backward process exhibits the opposite trend due to marginal preservation. 


In particular, the closed-form solution~\eqref{eq:optimal} is highly reminiscent of the iterative formula used in mean shift~\citep{fukunaga1975estimation,cheng1995mean,comaniciu2002mean,yamasaki2020mean}. 
Mean shift is a non-parametric algorithm designed to locate modes of a density function, typically a KDE, via iterative gradient ascent with adaptive step sizes. Given a current position $\bfx$, mean shift with a Gaussian kernel and bandwidth $h$ iteratively adds the vector $\bfm(\bfx)-\bfx$, which points toward the direction of maximum increase of the KDE $p_h(\bfx)=(1/|\calI|)\sum_{i=1}^{|\calI|}\calN(\bfx; \bfy_i, h^2\bfI)$, to itself, \ie, $\bfx \leftarrow \left[\bfm(\bfx)-\bfx\right] + \bfx$. The \textit{mean vector} is
\begin{equation}
	\label{eq:mean_shift}
	\bfm (\bfx, h)=
	\sum_{i} \frac{\exp \left(-\lVert \bfx - \bfy_i \rVert^2_2/2h^2\right)}{\sum_{j}\exp \left(-\lVert \bfx - \bfy_j \rVert^2_2/2h^2\right)} \bfy_i.
\end{equation}
As a mode-seeking algorithm, mean shift has been particularly successful in clustering \citep{cheng1995mean,carreira2015review}, image segmentation \citep{comaniciu1999mean,comaniciu2002mean}, and video tracking \citep{comaniciu2000real,comaniciu2003kernel}. By identifying the bandwidth $\sigma_t$ in \eqref{eq:optimal} with the bandwidth $h$ in \eqref{eq:mean_shift}, we build a connection between the optimal denoising output of a diffusion model and annealed mean shift under the KDE-based data modeling. Moreover, the time-decreasing bandwidth ($\sigma_t\rightarrow 0$ as $t\rightarrow 0$) in~\eqref{eq:optimal} strongly parallels \textit{annealed mean shift}, or \textit{multi-bandwidth mean shift}~\citep{shen2005annealedms}, a metaheuristic algorithm designed to escape local maxima where classical mean shift is susceptible to stuck, by monotonically decreasing the bandwidth in iterations. Analogous to~\eqref{eq:convex}, each Euler step in the optimal case equals a convex combination of the annealed mean vector and the current position, with the PF-ODE $\rmd \bfx_t / \rmd \sigma_t = (\bfx_t - r^{\star}_{\bftheta}(\bfx_t; t))/\sigma_t= \bfeps_{\bftheta}^{\star}(\bfx_t; t)$.

The above analysis further implies that all ODE trajectories generated by an optimal diffusion model are uniquely determined and governed by a bandwidth-varying mean shift. In this setting, both the forward (encoding) process and backward (decoding) process depend solely on the data distribution and the noise distribution, independent of the model architecture or optimization algorithm. This property, previously referred to as uniquely identifiable encoding and empirically verified in~\citep{song2021sde}, is here shown to be theoretically connected to a global KDE-based mode-seeking algorithm~\citep{shen2005annealedms}, and thus reveals the asymptotic sampling behavior of diffusion models as training converges to the optimum. Although optimal diffusion models essentially memorize the dataset and replay discrete data points during sampling, we argue that in practice, a slight score deviation from the optimum both preserves generative ability and substantially mitigates mode collapse (see Section~\ref{subsec:score_deviation}).

\begin{figure}[t]
	\centering
    \begin{subfigure}[t]{0.5\textwidth}
        \includegraphics[width=\columnwidth]{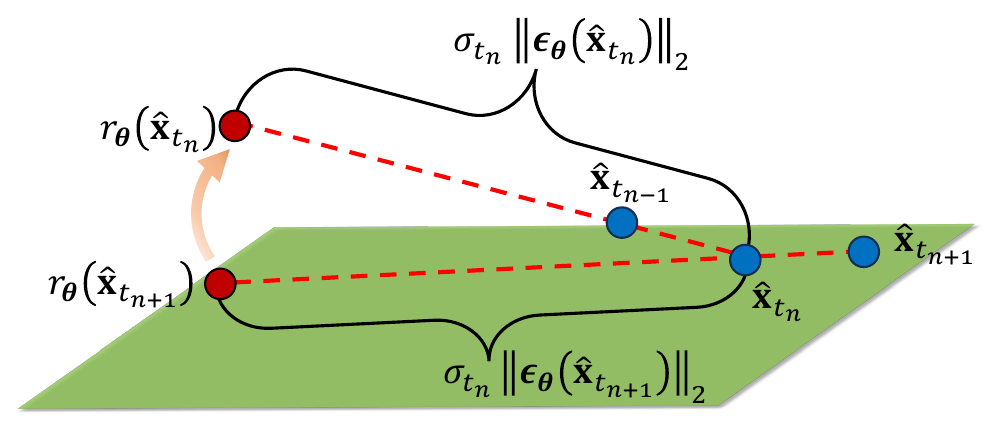}
        \caption{Stepwise rotation.}
        \label{fig:convex}
    \end{subfigure}
    \begin{subfigure}[t]{0.4\textwidth}
        \includegraphics[width=\columnwidth]{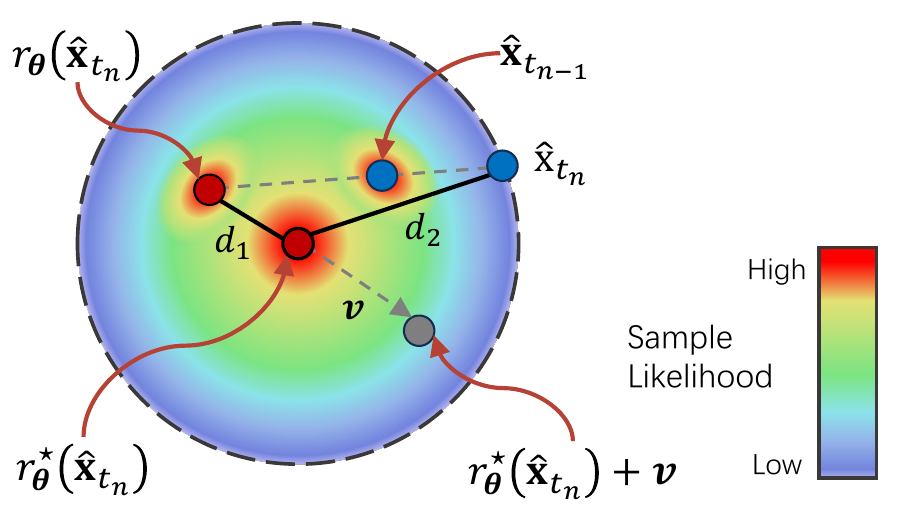}
        \caption{Monotone likelihood increasing.}
        \label{fig:meanshift}    
    \end{subfigure}
    \caption{
    (a) An illustration of two consecutive Euler steps, starting from a current sample $\hatx_{t_{n+1}}$. A single Euler step in the ODE-based sampling is a convex combination of the denoising output and the current position to determine the next position. Blue points form a piecewise linear sampling trajectory, while red points form the denoising trajectory governing the rotation direction. 
    (b) We have three likelihood orders in the ODE-based diffusion sampling: (1) $p_{h}(r_{\bftheta}(\hatx_{t_n}))\ge p_{h}(\hatx_{t_n})$, (2) $p_{h}(\hatx_{t_{n-1}})\ge p_{h}(\hatx_{t_n})$, and (3) $p_{h}(r_{\bftheta}^{\star}(\hatx_{t_n}))\ge p_{h}(\hatx_{t_n})$. Note that $p_{h}(r_{\bftheta}^{\star}(\hatx_{t_n}))$ may not possess the highest likelihood within the sphere.
    }
    \label{fig:local_properties}
\end{figure}

\subsubsection{Local Properties}
\label{subsubsec:local_properties}

We first show that (1) the denoising output governs the rotation of the sampling trajectory, and (2) each sampling trajectory converges monotonically in terms of sample likelihood, with its coupled denoising trajectory consistently achieving higher likelihood. 

Figure~\ref{fig:convex} illustrates two successive Euler steps according to Proposition~\ref{prop:convex}. The direction of the sampling trajectory (depicted as a polygonal chain with consecutive blue vertices) is controlled by the denoising outputs, while the vertex locations depend on the time schedule. In the optimal case, the sampling path follows a similar structure with the PF-ODE $\rmd \bfx_t = \bfeps_{\bftheta}^{\star}(\bfx_t; t) \rmd \sigma_t$. This equation defines a special vector field featuring an approximately constant magnitude $\|\bfeps_{\bftheta}^{\star}(\bfx_t; t)\|_2$ across all marginal distributions $p_t(\bfx_t)$. 
\begin{proposition}
    \label{prop:eps_norm}
    The magnitude of $\bfeps_{\bftheta}^{\star}(\bfx_t; t)$ concentrates around $\sqrt{d}$, where $d$ denotes the data dimension. Consequently, the total length of the sampling trajectory is approximately $\sigma_T \sqrt{d}$, where $\sigma_T$ denotes the maximum noise level.
\end{proposition}
The above results also hold for practical diffusion models, with proofs and supporting empirical evidence provided in Section~\ref{subsec:proof_eps_norm}. We further deduce that the position of each intermediate point $\hatx_{t_{n}}$, $n\in [1, N-1]$ in the sampling trajectory is primarily determined by the chosen time schedule, given that $\lVert r_{\bftheta}(\hatx_{t_{n+1}}) - \hatx_{t_{n}}\rVert_2 = (\sigma_{t_n}/\sigma_{t_{n+1}})\lVert r_{\bftheta}(\hatx_{t_{n+1}}) - \hatx_{t_{n+1}}\rVert_2=\sigma_{t_{n}}\lVert\bfeps_{\bftheta}(\hatx_{t_{n+1}})\rVert_2\approx\sigma_{t_{n}}\lVert\bfeps_{\bftheta}(\hatx_{t_{n}})\rVert_2=\lVert r_{\bftheta}(\hatx_{t_{n}}) - \hatx_{t_{n}}\rVert_2$. In this scenario, the denoising output $r_{\bftheta}(\hatx_{t_{n+1}})$ appears to be oscillating toward $r_{\bftheta}(\hatx_{t_{n}})$ around $\hatx_{t_n}$, akin to the motion of a simple gravity pendulum~\citep{young1996university}. The pendulum length contracts by the factor $\sigma_{t_{n}}/\sigma_{t_{n+1}}$ at each sampling step, starting from an initial length of roughly $\sigma_T\sqrt{d}$. This specific structure is shared across all trajectories. In practice, the oscillation amplitude is extremely small ($\approx 0^\circ$), and the entire sampling trajectory remains nearly confined to a two-dimensional plane. The minor deviations can be effectively represented using a small number of orthogonal bases, as discussed in Section~\ref{sec:trajectory_regularity}.

Next, we characterize the likelihood orders in the deterministic sampling process. To simplify notations, we denote the deviation of denoising output from the optimal counterpart as $d_1(\hatx_{t_{n}}) = \left\lVert r_{\bftheta}^{\star}(\hatx_{t_{n}}) - r_{\bftheta}(\hatx_{t_{n}}) \right\rVert_2$ and the distance between the optimal denoising output and the current position as $d_2(\hatx_{t_{n}}) = \left\lVert r_{\bftheta}^{\star}(\hatx_{t_{n}}) - \hatx_{t_{n}}\right\rVert_2$.
\begin{proposition}
    \label{prop:likelihood}
    In deterministic sampling with the Euler method, the sample likelihood is non-decreasing, \ie, $\forall n\in[1, N]$, we have $p_{h}(\hatx_{t_{n-1}})\ge p_{h}(\hatx_{t_n})$ and $p_{h}(r_{\bftheta}(\hatx_{t_n}))\ge p_{h}(\hatx_{t_n})$ with respect to the Gaussian KDE $p_{h}(\bfx)=(1/|\calI|)\sum_{i\in\calI}\calN(\bfx; \bfy_i, h^2\bfI)$ for any positive bandwidth $h$, under the assumption that all samples in the trajectory satisfy $d_1(\hatx_{t_{n}}) \le d_2(\hatx_{t_{n}})$.
\end{proposition}
A visual illustration is provided in Figure~\ref{fig:meanshift}. The assumption requires that the learned denoising output $r_{\bftheta}(\hatx_{t_{n}})$ falls within a sphere centered at the optimal denoising output $r_{\bftheta}^{\star}(\hatx_{t_{n}})$ with a radius of $d_2(\hatx_{t_{n}})$. This radius controls the maximum deviation of the learned denoising output and shrinks during the sampling process. In practice, such an assumption is relatively easy to satisfy for a well-trained diffusion model.  
Therefore, each sampling trajectory monotonically converges in terms of sample likelihood ($p_{h}(\hatx_{t_{n-1}})\ge p_{h}(\hatx_{t_n})$), while its coupled denoising trajectory converges even faster ($p_{h}(r_{\bftheta}(\hatx_{t_n}))\ge p_{h}(\hatx_{t_n})$). Given an empirical data distribution, Proposition~\ref{prop:likelihood} applies to any marginal distribution of the forward SDE $\{p_t(\bfx)\}_{t=0}^T$, each of which is a KDE with varying positive bandwidth $t$. 
Moreover, with an infinitesimal step size, Proposition~\ref{prop:likelihood} naturally extends to the continuous-time version. Finally, the standard monotone convergence property of mean shift is recovered when diffusion models are trained to optimality.
\begin{corollary}
	\label{corollary:meanshift}
	We have $p_{h}(\bfm(\hatx_{t_n}))\ge p_{h}(\hatx_{t_n})$, when $r_{\bftheta}(\hatx_{t_n})=r_{\bftheta}^{\star}(\hatx_{t_n})=\bfm (\hatx_{t_n})$. 
\end{corollary}

\begin{figure}[t]
	\centering
	\begin{subfigure}[b]{0.32\textwidth}
		\includegraphics[width=\textwidth]{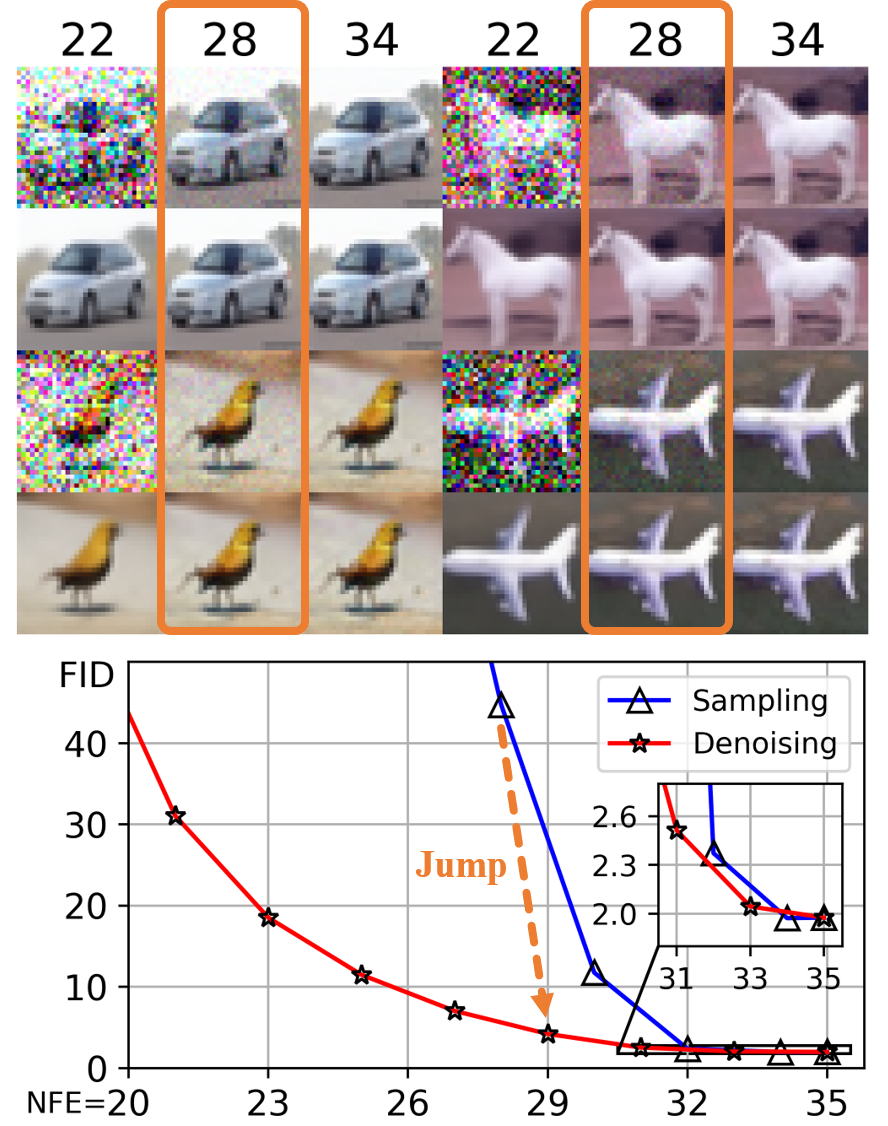}
		\caption{Unconditional generation (CIFAR-10). Total NFE $=$ 35.}
	\end{subfigure}
    \begin{subfigure}[b]{0.32\textwidth}
		\includegraphics[width=\textwidth]{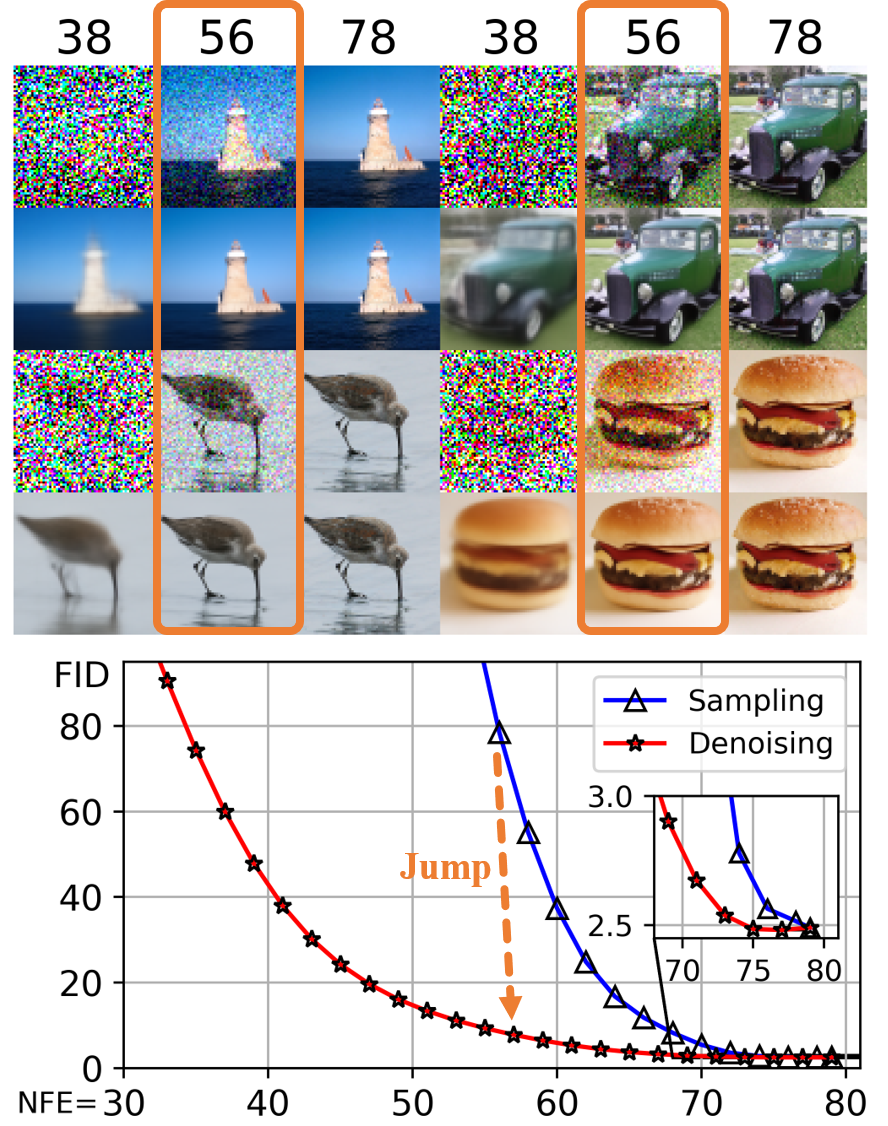}
        \caption{Class-conditional generation (ImageNet). Total NFE $=$ 79.}
	\end{subfigure}
	\begin{subfigure}[b]{0.321\textwidth}
		\includegraphics[width=\textwidth]{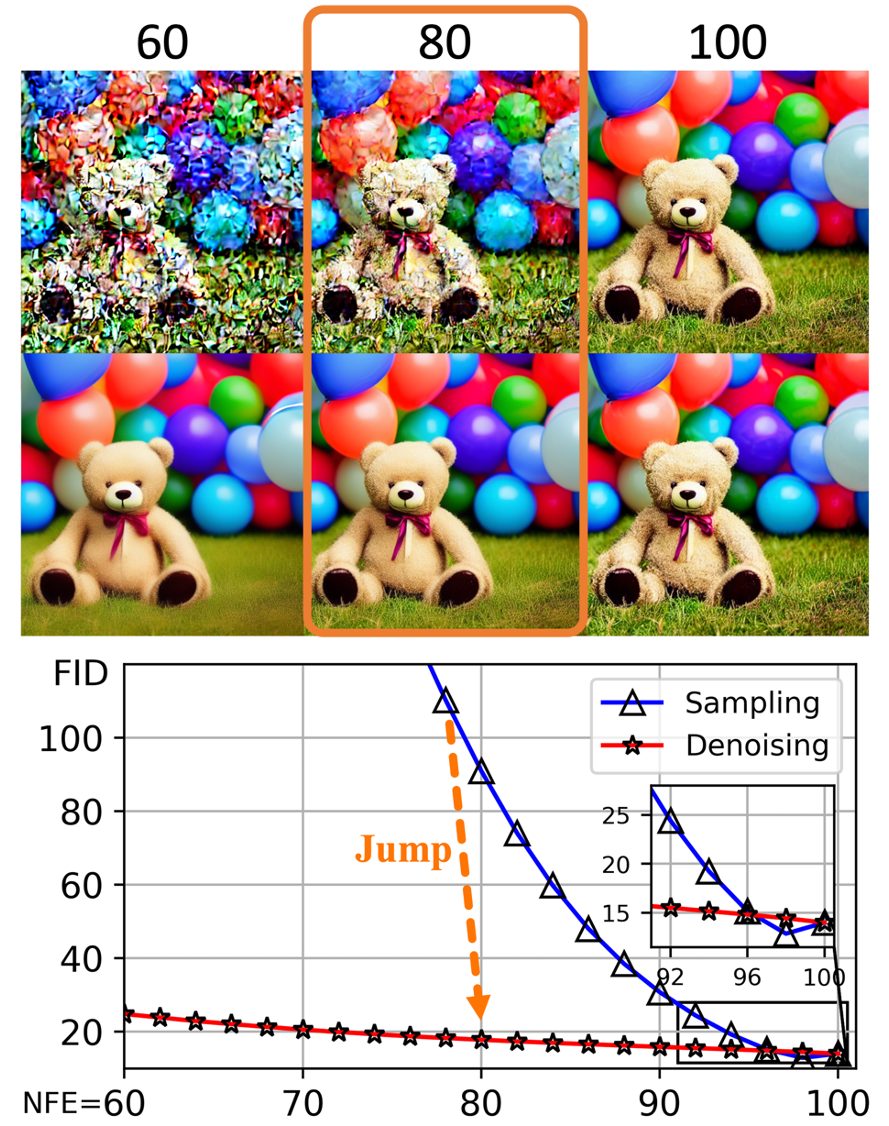}
		\caption{Text-conditional generation with SDv1.5. Total NFE $=$ 100.}
        \label{fig:fid_sd}
	\end{subfigure}
    \caption{Comparison of visual quality (top is sampling trajectory, bottom is denoising trajectory) and Fr\'echet Inception Distance (FID \citep{heusel2017gans}, lower is better) \textit{w.r.t.}\ the number of score function evaluations (NFEs). The denoising trajectory converges much faster than the sampling trajectory in terms of FID and visual quality. Figures (a)(b) are generated by EDM~\citep{karras2022edm}, and Figure (c) is generated by SDv1.5~\citep{rombach2022ldm}.}
	\label{fig:fid}
\end{figure}

Besides the monotone increase in sample likelihood, a similar trend is observed in image quality (Figure~\ref{fig:fid}). Both qualitative and quantitative results show that the denoising trajectory converges significantly faster than the sampling trajectory. This observation motivates a new technique, which we termed ``ODE-Jump''. The key idea is to directly transition from \textit{any} sample at \textit{any} time step of the sampling trajectory to its corresponding point on the denoising trajectory, and then returns the denoising output as the final synthetic result. Specifically, instead of following the full sequence $\hatx_{t_N} \rightarrow \dots \rightarrow\hatx_{t_{n}}\rightarrow \dots \rightarrow \hatx_{t_{0}}$, we modify it to $\hatx_{t_N} \rightarrow \dots \rightarrow \hatx_{t_{n}} \rightarrow r_{\bftheta}(\hatx_{t_{n}})$. This reduces the total NFE from $N$ to $N-n+1$, assuming one NFE per step. This technique is highly flexible and simple to implement. It only requires monitoring the visual quality of intermediate denoising samples to determine an appropriate time to terminate the remaining steps. As an example, consider the sampling process with SDv1.5 in Figure~\ref{fig:fid_sd}. By jumping from NFE=80 of the sampling trajectory to NFE=81 of the denoising trajectory, we obtain a substantial improvement in FID, while producing a visually comparable result to the final sample at NFE=100 with significantly fewer NFEs. Additional results are shown in Figure~\ref{fig:sd_ays}. Figure~\ref{fig:fid} also highlights the insensitivity of FID to subtle differences in image quality, a limitation also noted in previous work~\citep{kirstain2023pick,podell2024sdxl}.

\subsubsection{Global Properties}
\label{subsubsec:global_properties}
 


We then show that (1) the sampling trajectory acts as a linear-nonlinear-linear mode-seeking path, and (2) the trajectory statistics undergo a dramatic change during a short phase transition period.

In the optimal case, the denoising output, also referred to as the annealed mean vector, starts from a spurious mode (approximately the dataset mean), \ie, $r_{\bftheta}^{\star}(\bfx_t; t)\approx (1/|\calI|)\sum_{i\in\calI}\bfy_i$ when the bandwidth $\sigma_t$ is sufficiently large. 
Meanwhile, the sampling trajectory is initially located in an approximately uni-modal Gaussian distribution with a \textit{linear} score function: 
\begin{equation}
    \text{The first linear stage:} \quad \nablaxt \log p_t(\bfx_t)=\left(r_{\bftheta}^{\star}(\bfx_t; t) - \bfx_t\right)/\sigma_t^2 \approx - \bfx_t/\sigma_t^2.   
\end{equation}
This approximation holds for large $t$, since the dataset mean has negligible norm relative to $\bfx_t$ by the concentration of measure (Lemma~\ref{lemma:concentration}). This justifies heuristic methods that replace the learned score with the Gaussian analytic score at the first step~\citep{dockhorn2022genie,zhou2024fast,wang2024the}.
As $\sigma_t$ monotonically decreases during sampling, the number of modes in the Gaussian KDE
$\hat{p}_t(\bfx_t)=(1/|\calI|)\sum_{i\in\calI}\calN\left(\bfx_t;\bfy_i, \sigma_t^2\bfI\right)$ increases~\citep{silverman1981using}, and the underlying distribution surface gradually shifts from a simple Gaussian to a complex multi-modal form. In this intermediate stage, the score function appears highly data-dependent and \textit{nonlinear}, as multiple data points exert non-negligible influence. Finally, with a sufficiently small bandwidth $\sigma_t$, the sampling trajectory is attracted to a specific real-data mode, and the score function appears approximately \textit{linear} again, \ie, 
\begin{equation}
    \text{The second linear stage:} \quad \nablaxt \log p_t(\bfx_t)=\left(r_{\bftheta}^{\star}(\bfx_t; t) - \bfx_t\right)/\sigma_t^2 \approx (\bfy_{k}-\bfx_t)/\sigma_t^2,
\end{equation}
where $\bfy_k$ denotes the nearest data point to $\bfx_t$. In other words, the posterior distribution can again be well approximated by a Gaussian. 
This global linear-nonlinear-linear behavior allows the sampling trajectory to locate a true mode under mild conditions, similar to annealed mean shift~\citep{shen2005annealedms}. 
Intriguingly, the total trajectory length is guaranteed to be about $\sigma_T\sqrt{d}$ (Proposition~\ref{prop:eps_norm}), implying a shared structural property across trajectories originating from different initial conditions. 

\begin{figure}[t!]
    \centering
    \begin{subfigure}[t]{0.325\textwidth}
        \includegraphics[width=\columnwidth]{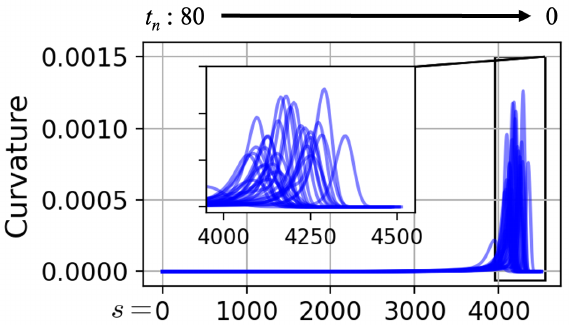}
        \caption{Curvature.}
        \label{fig:traj_stats_optimal_curvature}
    \end{subfigure}
    \begin{subfigure}[t]{0.32\textwidth}
        \includegraphics[width=\columnwidth]{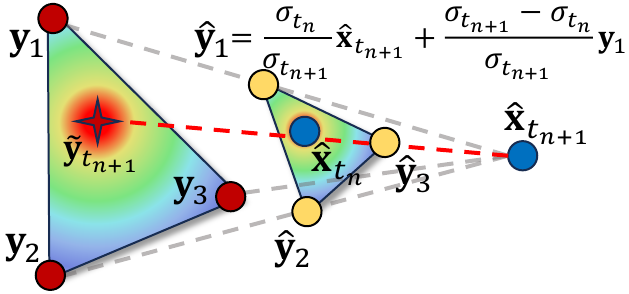}
        \caption{Bi-level convex combination.}
        \label{fig:bilevel_convex}
    \end{subfigure}
    \begin{subfigure}[t]{0.315\textwidth}
        \includegraphics[width=\columnwidth]{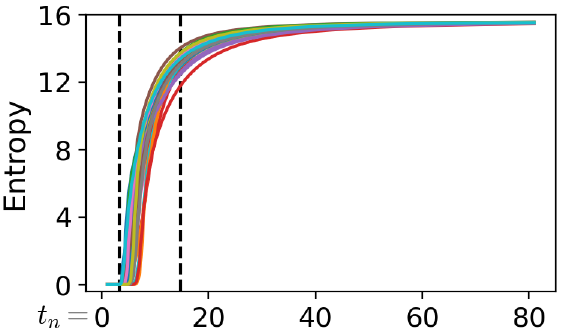}
        \caption{Shannon entropy.}
        \label{fig:entropy_cifar10}
    \end{subfigure}
    \begin{subfigure}[t]{0.325\textwidth}
        \includegraphics[width=\columnwidth]{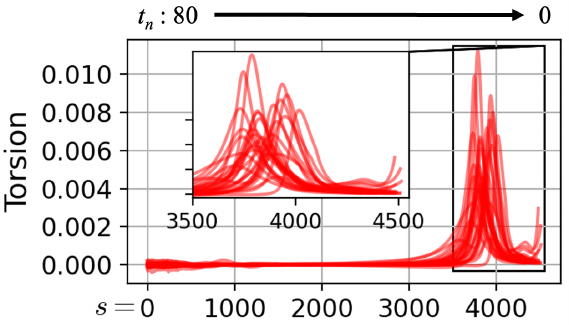}
        \caption{Torsion.}
        \label{fig:traj_stats_optimal_torsion}
    \end{subfigure}
    \begin{subfigure}[t]{0.32\textwidth}
        \includegraphics[width=\columnwidth]{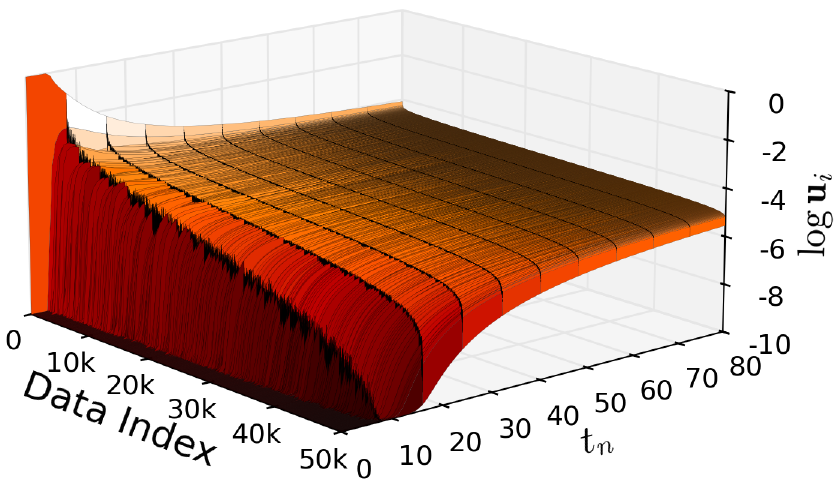}
        \caption{Evolution of coefficients.}
        \label{fig:u_per_sample_cifar10}
    \end{subfigure}
    \begin{subfigure}[t]{0.315\textwidth}
        \includegraphics[width=\columnwidth]{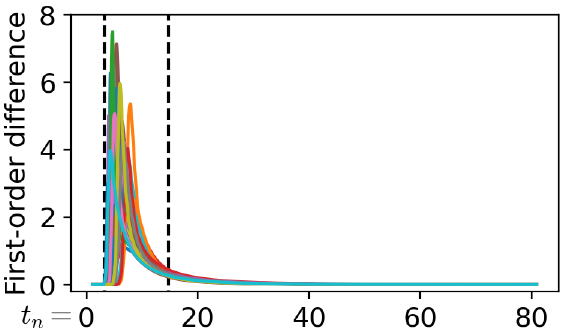}
        \caption{First-order difference.}
        \label{fig:entropy_diff_cifar10}
    \end{subfigure}
    \caption{Further analysis of deterministic sampling dynamics based on optimal denoising outputs, includes (a/d) the curvature and torsion functions; (b) an illustration of the bi-level convex combination used to infer the next position; (e) the evolution of convex combination coefficients; and (c/f) the corresponding Shannon entropy along the sampling trajectories.}
    \label{fig:optimal_phase_transition}
\end{figure}
A straightforward piece of quantitative evidence supporting the above analysis comes from the trajectory statistics of generated sampling paths based on optimal denoising outputs, as discussed in Section~\ref{subsec:three_dim}. The statistics presented in Figures~\ref{fig:traj_stats_optimal_curvature} and~\ref{fig:traj_stats_optimal_torsion} exhibit a distinct three-stage pattern, with the first and second linear stages characterized by near-zero curvature and torsion. Critical transition points can be readily identified by thresholding. For example, curvature values below $1e^{-5}$, which corresponds to average sampling times of $14.80$ or $3.44$, can be used to define linear trajectories in practice. We next delve into a more detailed analysis of the phase transition between the linear and nonlinear stages in the sampling dynamics. Given the current position $\hatx_{t_{n+1}}$ and its corresponding optimal denoising output $r^{\star}_{\bftheta}(\hatx_{t_{n+1}})$, the next position $\hatx_{t_n}$ predicted by the Euler method according to \eqref{eq:convex} becomes
\begin{equation}
    \label{eq:convex_new}
        \begin{aligned}
		\hatx_{t_{n}}
        &=\frac{\sigma_{t_n}}{\sigma_{t_{n+1}}} \hatx_{t_{n+1}} +  \frac{\sigma_{t_{n+1}}-\sigma_{t_n}}{\sigma_{t_{n+1}}} r^{\star}_{\bftheta}(\hatx_{t_{n+1}})\\
        &=\sum_{i}\underbrace{\frac{\exp \left(-\lVert \hatx_{t_{n+1}} - \bfy_i \rVert^2_2/2\sigma_t^2\right)}{\sum_{j}\exp \left(-\lVert \hatx_{t_{n+1}} - \bfy_j \rVert^2_2/2\sigma_t^2\right)}}_{\bfu_i(\hatx_{t_{n+1}})}
        \underbrace{\left(\frac{\sigma_{t_n}}{\sigma_{t_{n+1}}} \hatx_{t_{n+1}} +   \frac{\sigma_{t_{n+1}}-\sigma_{t_n}}{\sigma_{t_{n+1}}} \bfy_i\right)}_{\hat{\bfy}_i(\hatx_{t_{n+1}})}.
        \end{aligned}
\end{equation}
This implies that $\hatx_{t_n}$ lies within a convex hull, whose vertices $\hat{\bfy}_i$ are \textit{convex combinations} of the current position $\hatx_{t_{n+1}}$ and data points $\bfy_i$, with coefficients determined by the time schedule ($\sigma_{t_n}/\sigma_{t_{n+1}}$, $1-\sigma_{t_n}/\sigma_{t_{n+1}}$), as illustrated in Figure~\ref{fig:bilevel_convex}. In contrast, another convex combination, parameterized by coefficients $\bfu_i(\bfx_{t_{n+1}})$, quantifies the relative influence of individual data points and plays a central role in determining transition points within the linear-nonlinear-linear path. As shown in Figure~\ref{fig:u_per_sample_cifar10}, where a logarithmic scale is used with a small bias term $1e^{-10}$ for numerical stability, the evolution of coefficients begins approximately uniform at $1/50{,}000\approx 10^{-4.7}$ and gradually converges toward a specific data point. Note that the influence of different data points evolves differently: some decay monotonically, while others increase initially before declining. This behavior suggests the existence of hierarchical clusters in the dataset, potentially reflecting coarse-to-fine semantic structures in representing learning~\citep{bengio2013representation}. Moreover, there exists a non-eligible period (roughly the final $5\%$ of the trajectory in Figure~\ref{fig:u_per_sample_cifar10}) during which one data point already dominates the trajectory's trend. This phenomenon becomes more pronounced in Figures~\ref{fig:entropy_cifar10} and~\ref{fig:entropy_diff_cifar10}, where we introduce a Shannon entropy-based criterion, $\calH(\bfu) = - \sum_i \bfu_i(\bfx_{t_{n+1}})\log \bfu_i(\bfx_{t_{n+1}})$, and its temporal derivative $\partial \calH(\bfu)/ \partial t$ to quantify and visualize the evolving influence of data points throughout the sampling dynamics. Similar three distinct dynamical regimes and phase transitions during the generative diffusion process have also been observed in previous works~\citep{biroli2024dynamical,raya2024spontaneous}.

\subsection{Regularity Revisited under the Gaussian Data Assumption}
\label{subsec:theoretical_analysis_gaussian}

Although we have analyzed closed-form solutions for the optimal sampling dynamics under the empirical data distribution (Section~\ref{subsec:theoretical_analysis}), the resulting complex ODE hinders detailed theoretical analysis, particularly in the intermediate nonlinear regime. To gain further insight, we examine a simplified Gaussian data setting and demonstrate that trajectory regularity still emerges. These findings confirm that the observed structure is primarily an intrinsic property of deterministic sampling dynamics. 
\begin{figure}[t!]
    \begin{subfigure}[t]{0.48\textwidth}
        \includegraphics[width=\columnwidth]{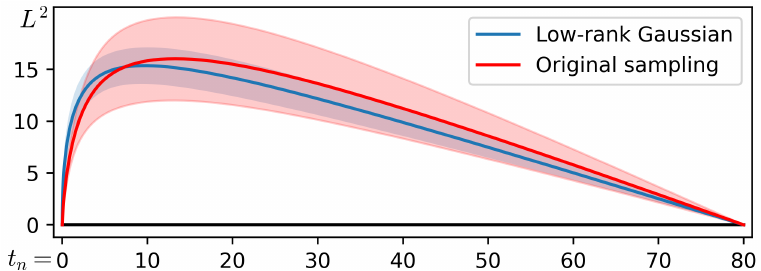}
        \caption{Low-rank Gaussian (1-D projection).}
        \label{fig:deviation_gaussian}
    \end{subfigure}
    \begin{subfigure}[t]{0.48\textwidth}
        \includegraphics[width=\columnwidth]{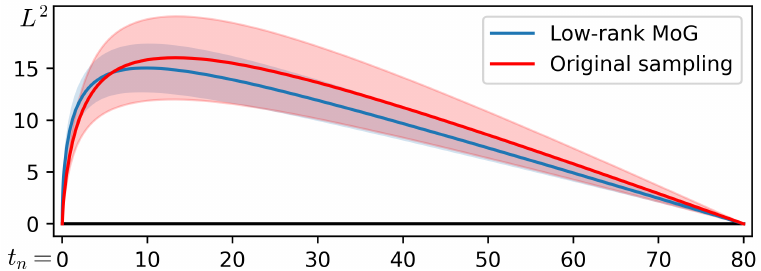}
        \caption{Low-rank mixture of Gaussians (1-D projection).}
        \label{fig:deviation_mog}
    \end{subfigure}
    \begin{subfigure}[t]{0.49\textwidth}
        \centering
        \includegraphics[width=\columnwidth]{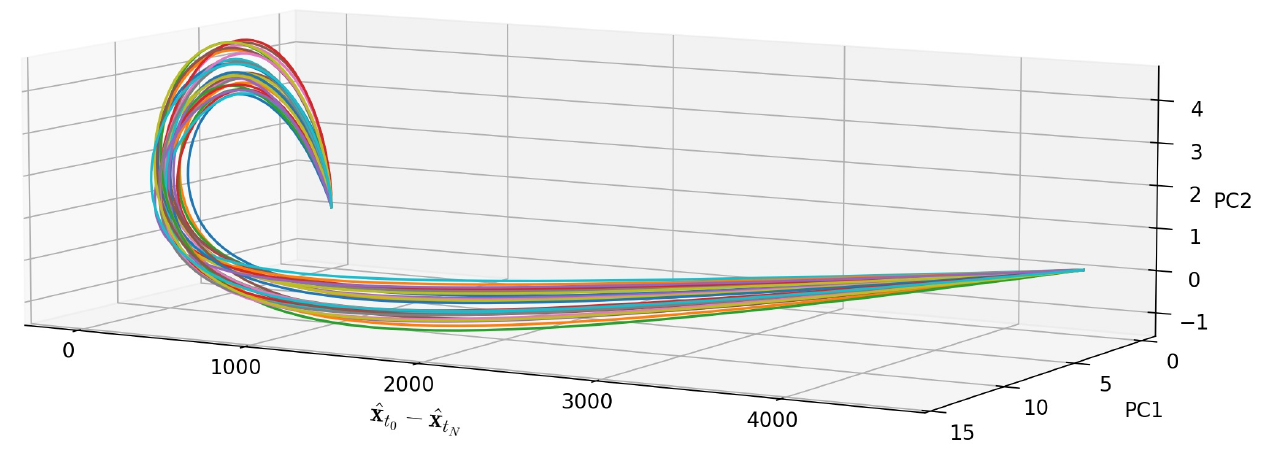}
        \caption{Low-rank Gaussian (3-D projection).}
        \label{fig:traj_3d_gaussian}
    \end{subfigure}
    \hfill
    \begin{subfigure}[t]{0.49\textwidth}
        \centering
        \includegraphics[width=\columnwidth]{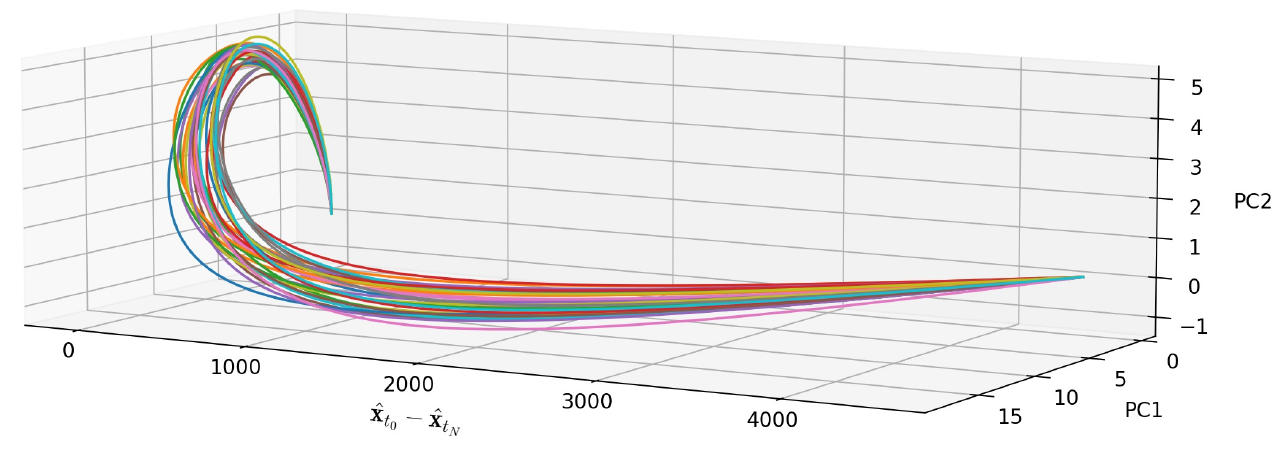}
        \caption{Low-rank mixture of Gaussians (3-D projection).}
        \label{fig:traj_3d_mog}
    \end{subfigure}
    \caption{Unconditional generation results on CIFAR-10 (32$\times$ 32). (a)-(b) One-dimensional trajectory deviation. (c)-(d) Three-dimensional trajectory visualization. Trajectory reconstruction and corresponding statistics are provided in Figures~\ref{fig:recon_gaussian} and~\ref{fig:recon_mog}.}
    \label{fig:devitation_traj_gassuian_mog}
\end{figure}

\begin{proposition}
    \label{prop:analytic_gaussian}
    Suppose the data distribution is Gaussian $p_d(\bfx)=\calN(\bfmu, \bfSigma)$, where $\bfmu \in \bbR^d$, $\bfSigma\in \bbR^{d\times d}$ is positive semi-definite (PSD) with $\text{rank}\,(\bfSigma)=r\ll d$. Let $\bfSigma=\bfU\bfLambda \bfU^T$ denote the singular value decomposition (SVD), where $\bfU\in \bbR^{d \times r}$ contains eigenvectors $\bfu_i$ as columns, and $\bfLambda\in \bbR^{r \times r}$ is diagonal with eigenvalues $\lambda_i$, $i\in [1,r]$. In this setting, the PF-ODE solution $\bfx_t$ can be decomposed into the final sample $\bfx_0$, a scaled reverse displacement vector $\bfx_T-\bfx_0$, and a trajectory residual $\Delta_k(t)$:
    \begin{equation}
        \begin{aligned}
            \bfx_t &= \bfx_0 + \frac{\sigma_t}{\sigma_T} (\bfx_T-\bfx_0) + \Delta_k(t),\qquad \Delta_k(t)=\sum_{k=1}^{r}\varphi_k(t)\bfu_k^T(\bfx_T-\bfmu)\bfu_k,\\
            \varphi_k(t)&=\sqrt{\frac{\lambda_k+\sigma_t^2}{\lambda_k+\sigma_T^2}}-\sqrt{\frac{\lambda_k}{\lambda_k+\sigma_T^2}}-\frac{\sigma_t}{\sigma_T}\left(\,\,1-\sqrt{\frac{\lambda_k}{\lambda_k+\sigma_T^2}}\,\,\right).
        \end{aligned}
    \end{equation}
\end{proposition}
The squared norm of the trajectory residual, $\|\Delta_k(t)\|_2^2$, approximates the one-dimensional trajectory deviation and almost surely attains a unique maximum for $t\in [\sigma_0, \sigma_T]$. Furthermore, it concentrates around its expectation $\bbE_{\bfx_T} \left[\|\Delta_k(t)\|_2^2\right]$. Proofs are deferred to Section~\ref{subsec:proof_analytic_gaussian}. Empirical verification is provided in Figure~\ref{fig:devitation_traj_gassuian_mog}, Figures~\ref{fig:recon_gaussian}-\ref{fig:recon_mog}, Figures~\ref{fig:phi_evolve}-\ref{fig:low_rank_gaussian_comparison}, where we consider two cases: (1) fitting the entire dataset with a single Gaussian distribution, and (2) fitting each class with an individual Gaussian, yielding a Gaussian mixture model for the full dataset. Both simplified Gaussian data settings exhibit similar trajectory regularity, as we have observed in Section~\ref{sec:trajectory_regularity}.
\section{Geometry-Inspired Time Scheduling}
\label{sec:trajectory_algorithm}

In this section, as a simple illustration, we propose a new technique inspired by the geometric regularity of deterministic sampling in diffusion models to accelerate sampling and enhance sample quality. This technique is compatible with any numerical solver-based sampler and model architecture, easy to implement, and incurs negligible computational overhead.

\subsection{Algorithm}
A deterministic ODE-based numerical solver such as the Euler~\citep{song2021sde} or Runge-Kutta~\citep{liu2022pseudo,zhang2023deis} relies on a pre-defined time schedule $\Gamma=\{t_0=\epsilon, \cdots, t_N=T\}$ in the sampling process. Typically, given the initial time $t_N$ and the final time $t_0$, the intermediate time steps from $t_1$ to $t_{N-1}$ are determined by heuristic strategies such as uniform, quadratic~\citep{song2021ddim}, log-SNR~\citep{lu2022dpm,lu2022dpmpp}, and polynomial functions~\citep{karras2022edm,song2023consistency}. In fact, the time schedule reflects our prior knowledge of the sampling trajectory shape. Under the constraint of the total {\em number of score function evaluations} (NFEs), an improved time schedule can reduce the local truncation error in each numerical step, and hopefully minimize the global truncation error. In this way, the sample quality generated by numerical methods could approach that of the exact solutions of the given empirical PF-ODE~\eqref{eq:epf_ode}. 

Our previous discussions in Section~\ref{sec:trajectory_regularity} identified each sampling trajectory as a simple low-dimensional ``boomerang'' curve. We thus leverage this geometric structure to re-allocate the intermediate timestamps according to the principle that assigning a larger time step size when the trajectory exhibits a relatively small curvature, while assigning a smaller time step size when the trajectory exhibits a relatively large curvature. Additionally, different trajectories share almost the same shape, regardless of the model architecture used or generation conditions, which helps us estimate the common structure of the sampling trajectory by using just a few ``warmup'' samples. We name our approach to achieve the above goal as \textit{geometry-inspired time scheduling} (GITS) and elaborate the details as follows. 

The allocation of the intermediate timestamps can be formulated as an \textit{integer programming problem} and solved using standard dynamic programming (DP) to search for an optimal time schedule~\citep{cormen2022introduction}.\footnote{We also tried the Branch and Bound algorithm~\citep{land1960automatic} and obtained similar results. Nevertheless, alternative approaches exist for determining the time schedule, such as using a trainable neural network~\citep{tong2025learning,frankel2025s4s}, by leveraging our discovered trajectory regularity.} We first define a searching space denoted as $\Gamma_{g}$, which is a fine-grained grid including all possible intermediate timestamps. Then, we measure the trajectory curvature by the local truncation errors. More precisely, we define the cost from the current position $\bfx_{t_i}$ to the next position $\bfx_{t_j}$ as the difference between an Euler step and the ground-truth prediction, \ie, $c_{t_i\rightarrow t_j}\coloneqq\calD(\hatx_{t_i\rightarrow t_j}, \bfx_{t_i\rightarrow t_j})$, where $t_i$ and $t_j$ are two intermediate timestamps from $\Gamma_{g}$ and $t_i>t_j$. According to the empirical PF-ODE~\eqref{eq:epf_ode}, the ground-truth prediction is calculated as
$\bfx_{t_i\rightarrow t_j}=\bfx_{t_i}+\int_{t_i}^{t_j}\bfeps_{\bftheta}(\bfx_t)\sigma^{\prime}_t \rmd t$, and the Euler prediction is calculated as
$\hatx_{t_i\rightarrow t_j}=\bfx_{t_i}+(\sigma_{t_j}-\sigma_{t_i})\bfeps_{\bftheta}(\hatx_{t_i})\sigma^{\prime}_{t_i}$. The cost function $\calD$ can be defined as the Euclidean distance in the original pixel space, or any other user-specified metric. Given all computed pairwise costs, which form a cost matrix, the problem reduces to a standard \textit{minimum-cost path problem} and can be solved with dynamic programming. 
Since the global truncation error is not equal to the accumulation of local truncation errors at each step, we introduce a hyperparameter $\gamma$, analogous to the discount factor used in reinforcement learning~\citep{sutton1998reinforcement}, to compensate for this effect. 

Dynamic programming is a fundamental concept widely used in computer science and many other fields~\citep{cormen2022introduction}. \citet{watson2021learning} was the first one leveraging dynamic programming to re-allocate the time schedule in diffusion models. However, our motivation differs significantly from that of this previous work. \citet{watson2021learning} exploited the fact that the evidence lower bound (ELBO) can be decomposed into separate KL terms and utilized DP to find the optimal discrete-time schedule that maximizes the training ELBO. However, this strategy was reported to worsen sample quality, as acknowledged by the authors. In contrast, we first discovered a strong trajectory regularity shared by all sampling trajectories, and then used several ``warmup'' samples to estimate the trajectory curvature to determine a more effective time schedule for the sampling of diffusion models. 



\begin{table}[t!]
    \caption{Sample quality comparison in terms of Fr\'echet Inception Distance (FID~\citep{heusel2017gans}, lower is better) on four datasets (resolutions ranging from $32\times32$ to $256\times256$). $\dagger$: Results reported by authors. More results are provided in Table~\ref{tab:fid-full}. 
    }
    \label{tab:fid-1}
    \centering
    \fontsize{8}{10}\selectfont
    \begin{tabular}{lcccc}
        \toprule
        \multirow{2}{*}{METHOD} & \multicolumn{4}{c}{NFE} \\
        \cmidrule{2-5}
        & 5 & 6 & 8 & 10 \\
        \midrule
        \multicolumn{5}{l}{\textbf{CIFAR-10 32$\times$32}~\citep{krizhevsky2009learning}} \\
        \midrule
        DDIM~\citep{song2021ddim}                        & 49.66 & 35.62 & 22.32 & 15.69 \\
        \rowcolor[gray]{0.9} DDIM + \ourName (\textbf{ours})           & 28.05 & 21.04 & 13.30 & 10.37 \\
        DPM-Solver-2~\citep{lu2022dpm}                   & -     & 60.00 & 10.30 & 5.01  \\
        DPM-Solver++(3M)~\citep{lu2022dpmpp}             & 24.97 & 11.99 & 4.54  & 3.00  \\
        DEIS-tAB3~\citep{zhang2023deis} & 14.39 & 9.40 & 5.55 & 4.09 \\
        UniPC~\citep{zhao2023unipc}                      & 23.98 & 11.14 & 3.99  & 2.89  \\
        AMED-Solver~\citep{zhou2024fast}                 & -     & 7.04  & 5.56  & 4.14  \\
        AMED-Plugin~\citep{zhou2024fast}                 & -     & 6.67  & 3.34  & \textbf{2.48}  \\
        iPNDM~\citep{zhang2023deis}                      & 13.59 & 7.05  & 3.69  & 2.77  \\
        \rowcolor[gray]{0.9} iPNDM + \ourName (\textbf{ours})          & \textbf{8.38} & \textbf{4.88} & \textbf{3.24} & \textbf{2.49} \\
        \midrule
        \multicolumn{5}{l}{\textbf{FFHQ 64$\times$64}~\citep{karras2019style}} \\
        \midrule
        DDIM~\citep{song2021ddim}                        & 43.93 & 35.22 & 24.39 & 18.37 \\
        \rowcolor[gray]{0.9} DDIM + \ourName (\textbf{ours})           & 29.80 & 23.67 & 16.60 & 13.06 \\
        DPM-Solver-2~\citep{lu2022dpm}                   & -     & 83.17 & 22.84 & 9.46  \\
        DPM-Solver++(3M)~\citep{lu2022dpmpp}             & 22.51 & 13.74 & 6.04  & 4.12  \\
        DEIS-tAB3~\citep{zhang2023deis} & 17.36 & 12.25 & 7.59 & 5.56 \\
        UniPC~\citep{zhao2023unipc}                      & 21.40 & 12.85 & 5.50  & 3.84  \\
        AMED-Solver~\citep{zhou2024fast}                 & -     & 10.28 & 6.90  & 5.49  \\
        AMED-Plugin~\citep{zhou2024fast}                 & -     & 9.54  & 5.28  & 3.66  \\    
        iPNDM~\citep{zhang2023deis}                      & 17.17 & 10.03 & 5.52  & 3.98  \\
        \rowcolor[gray]{0.9} iPNDM + \ourName (\textbf{ours})          & \textbf{11.22} & \textbf{7.00} & \textbf{4.52} & \textbf{3.62} \\
        \midrule
        \multicolumn{5}{l}{\textbf{ImageNet 64$\times$64}~\citep{russakovsky2015ImageNet}} \\
        \midrule
        DDIM~\citep{song2021ddim}                        & 43.81 & 34.03 & 22.59 & 16.72 \\
        \rowcolor[gray]{0.9} DDIM + \ourName (\textbf{ours})           & 24.92 & 19.54 & 13.79 & 10.83 \\
        DPM-Solver-2~\citep{lu2022dpm}                   & -     & 44.83 & 12.42 & 6.84  \\
        DPM-Solver++(3M)~\citep{lu2022dpmpp}             & 25.49 & 15.06 & 7.84  & 5.67  \\
        DEIS-tAB3~\citep{zhang2023deis} & 14.75 & 12.57 & 6.84 & 5.34 \\
        UniPC~\citep{zhao2023unipc}                      & 24.36 & 14.30 & 7.52  & 5.53  \\
        RES(M)$\dagger$~\citep{zhang2023improved}           & 25.10    & 14.32 & 7.44  & 5.12  \\
        AMED-Solver~\citep{zhou2024fast}                 & -     & 10.63 & 7.71  & 6.06  \\
        AMED-Plugin~\citep{zhou2024fast}                 & -     & 12.05 & 7.03  & 5.01 \\
        iPNDM~\citep{zhang2023deis}                      & 18.99 & 12.92 & 7.20  & 5.11  \\
        \rowcolor[gray]{0.9} iPNDM + \ourName (\textbf{ours})          & \textbf{10.79} & \textbf{8.43} & \textbf{5.82} & \textbf{4.48} \\
        \midrule
        \multicolumn{5}{l}{\textbf{LSUN Bedroom 256$\times$256}~\citep{yu2015lsun} (pixel-space)} \\
        \midrule
        DDIM~\citep{song2021ddim}                        & 34.34 & 25.25 & 15.71 & 11.42 \\
        \rowcolor[gray]{0.9} DDIM + \ourName (\textbf{ours})           & 22.04 & 16.54 & 11.20 & 9.04  \\
        DPM-Solver-2~\citep{lu2022dpm}                   & -     & 80.59 & 23.26 & 9.61  \\
        DPM-Solver++(3M)~\citep{lu2022dpmpp}             & 23.15 & 12.28 & 7.44  & 5.71  \\
        UniPC~\citep{zhao2023unipc}                      & 23.34 & 11.71 & 7.53  & 5.75  \\
        AMED-Solver~\citep{zhou2024fast}                 & -     & 12.75 & \textbf{6.95}  & 5.38  \\
        AMED-Plugin~\citep{zhou2024fast}                 & -     & 11.58 & 7.48  & 5.70  \\
        iPNDM~\citep{zhang2023deis}                      & 26.65 & 20.73 & 11.78 & 5.57  \\
        \rowcolor[gray]{0.9} iPNDM + \ourName (\textbf{ours})          & \textbf{15.85} & \textbf{10.41} & \textbf{7.31} & \textbf{5.28} \\
        \bottomrule
    \end{tabular}
    \vspace{-1em}
\end{table}

\begin{table}[t!]
    \caption{Image generation results using Stable Diffusion v1.5 (two NFEs per sampling step). 
    }
    \label{tab:fid-2}
    \centering
    \fontsize{8}{10}\selectfont
    \begin{tabular}{lcccc}
        \toprule
        \multirow{2}{*}{METHOD} & \multicolumn{4}{c}{Step} \\
        \cmidrule{2-5}
        & 5 & 6 & 7 & 8 \\
        \midrule
        DPM-Solver++(2M)~\citep{lu2022dpmpp}                 & 16.80 & 15.43 & 14.88 & 14.65 \\
        \rowcolor[gray]{0.9} DPM-Solver++(2M) + \ourName (\textbf{ours})  & \textbf{15.53} & \textbf{13.18} & \textbf{12.32} & \textbf{12.17}\\
        \bottomrule
    \end{tabular}
\end{table}

\subsection{Experimental Results}
\label{subsec:experiments}
We adhere to the setup and experimental designs of the EDM framework~\citep{karras2022edm,song2023consistency}, with $f(t) = 0$, $g(t)=\sqrt{2t}$, and $\sigma_t = t$. Under this parameterization, the forward VE-SDE is expressed as $\rmdx_t =\sqrt{2t} \,  \rmd \bfw_t$, while the corresponding empirical PF-ODE is formulated as $\rmdx_t / \rmd t = (\bfx_t - r_{\bftheta}\left(\bfx_t; t\right))/t $. The temporal domain is segmented using a polynomial function $t_n=(t_0^{1/\rho}+\frac{n}{N}(t_N^{1/\rho}-t_0^{1/\rho}))^{\rho}$, where $t_0=0.002$, $t_N=80$, $n\in [0,N]$, and $\rho=7$. 
We initiate the dynamic programming experiments with $256$ ``warmup'' samples randomly selected from Gaussian noise to create a more refined grid, and then calculate the associated cost matrix. 
The ground-truth predictions are generated by iPNDM~\citep{zhang2023deis}, which employs a fourth-order multistep Runge-Kutta method with a lower-order warming start, using the polynomial time schedule with 60 NFEs. This yields a grid size of $|\Gamma_g|=61$. The default classifier-free guidance scale of 7.5 is used for Stable Diffusion (SDv1.5). We follow the standard FID and CLIP Score evaluation protocol for SDv1.5, using the reference statistics and 30k sampled captions from the MS-COCO validation set~\citep{lin2014microsoft}. For other datasets, we compute FID based on 50k generated samples~\citep{heusel2017gans}. 
All reported results for evaluated methods are obtained based on our developed open-source toolbox: \url{https://github.com/zju-pi/diff-sampler}.

\begin{figure}[t!]
    \centering
    \begin{subfigure}[t]{0.24\columnwidth}
        \centering
        \includegraphics[width=\columnwidth]{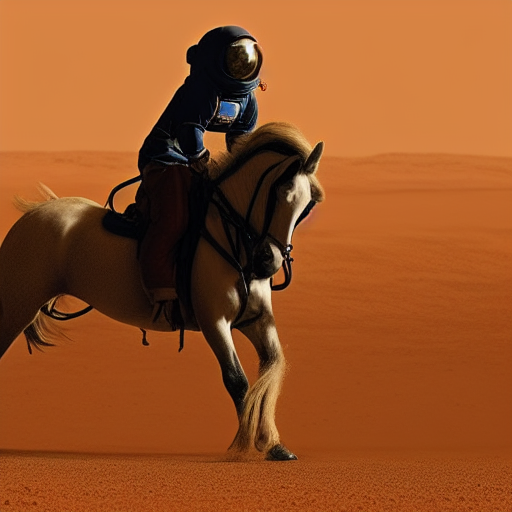}
    \end{subfigure}
    \begin{subfigure}[t]{0.24\columnwidth}
        \centering
        \includegraphics[width=\columnwidth]{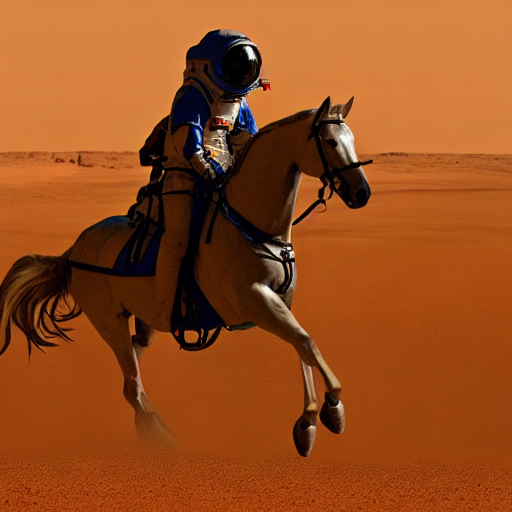}
    \end{subfigure}
    \begin{subfigure}[t]{0.24\columnwidth}
        \centering
        \includegraphics[width=\columnwidth]{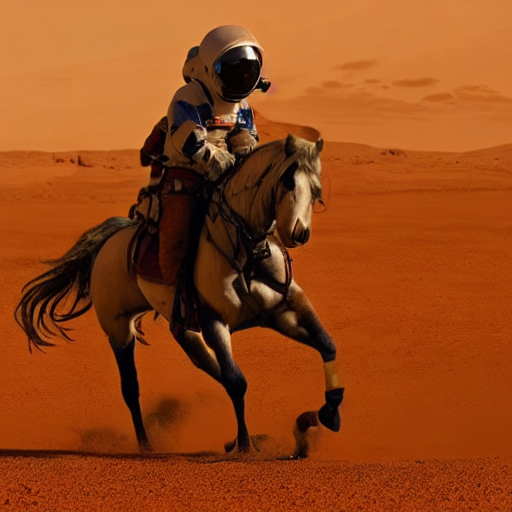}
    \end{subfigure}
    \begin{subfigure}[t]{0.24\columnwidth}
        \centering
        \includegraphics[width=\columnwidth]{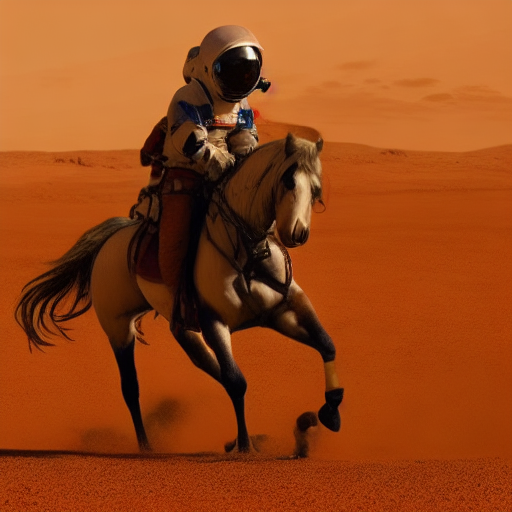}
    \end{subfigure}
    
    \vspace{0.15em}
    
    \begin{subfigure}[t]{0.24\columnwidth}
        \centering
        \includegraphics[width=\columnwidth]{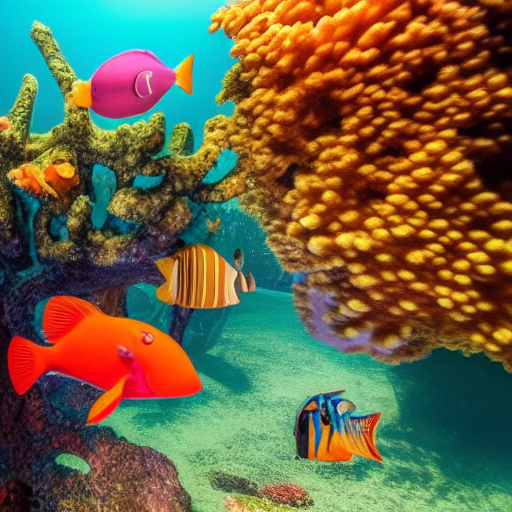}
    \end{subfigure}
    \begin{subfigure}[t]{0.24\columnwidth}
        \centering
        \includegraphics[width=\columnwidth]{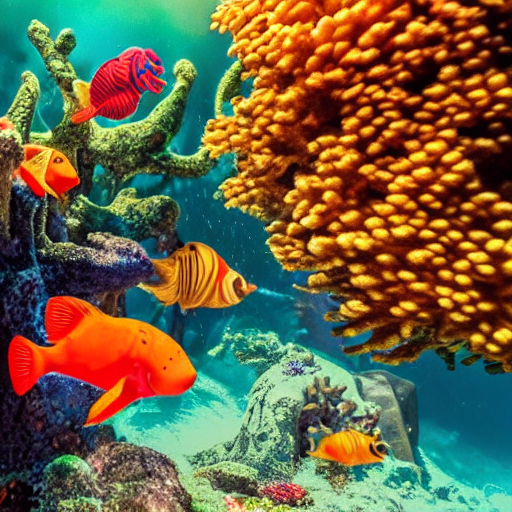}
    \end{subfigure}
    \begin{subfigure}[t]{0.24\columnwidth}
        \centering
        \includegraphics[width=\columnwidth]{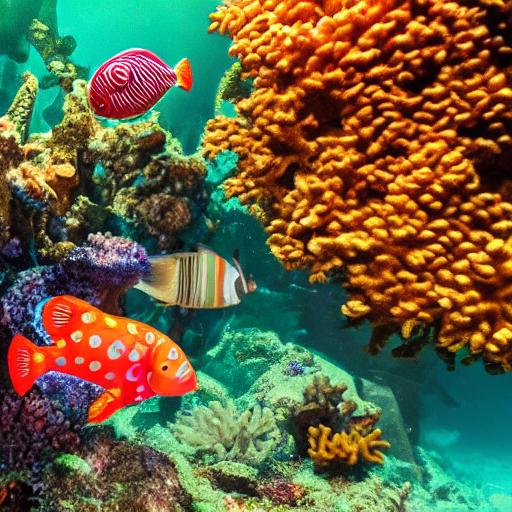}
    \end{subfigure}
    \begin{subfigure}[t]{0.24\columnwidth}
        \centering
        \includegraphics[width=\columnwidth]{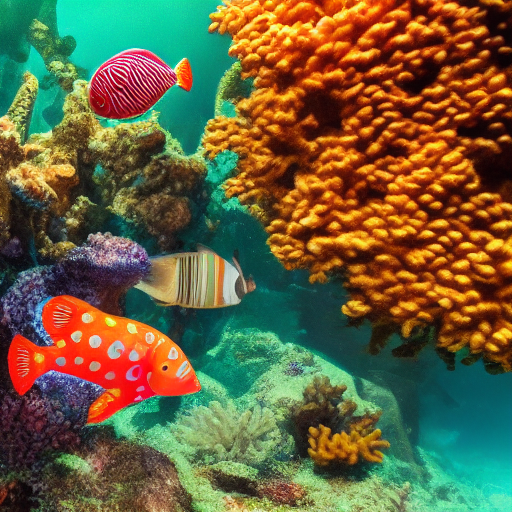}
    \end{subfigure}
    
    \vspace{0.15em}
    
    \begin{subfigure}[t]{0.24\columnwidth}
        \centering
        \includegraphics[width=\columnwidth]{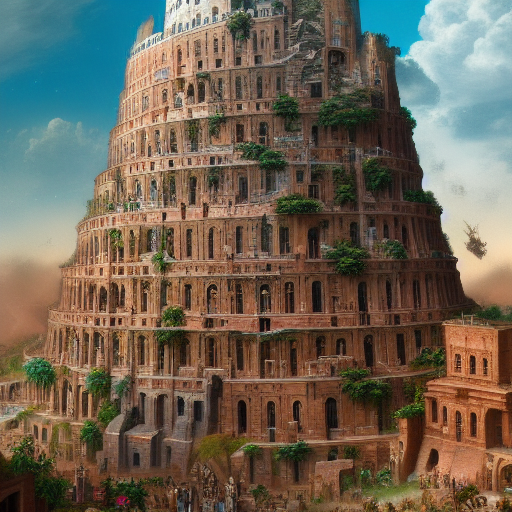}
        \caption{Uniform.}
    \end{subfigure}
    \begin{subfigure}[t]{0.24\columnwidth}
        \centering
        \includegraphics[width=\columnwidth]{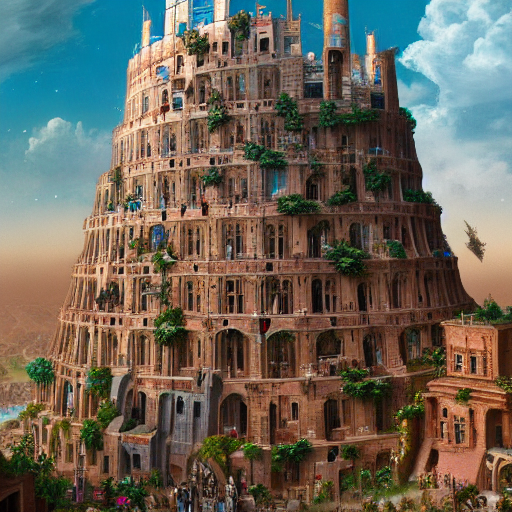}
        \caption{AYS.}
    \end{subfigure}
    \begin{subfigure}[t]{0.24\columnwidth}
        \centering
        \includegraphics[width=\columnwidth]{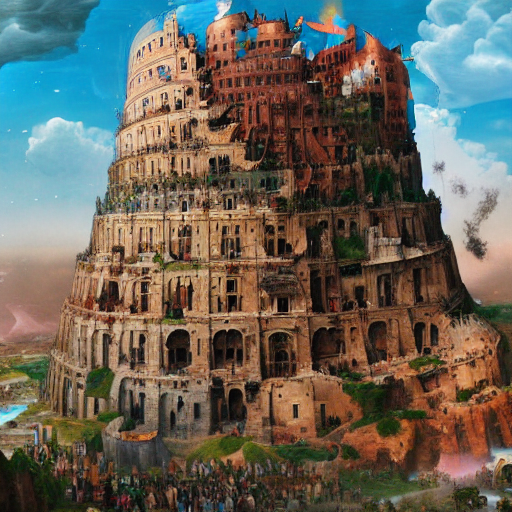}
        \caption{\ourName.}
    \end{subfigure}
    \begin{subfigure}[t]{0.24\columnwidth}
        \centering
        \includegraphics[width=\columnwidth]{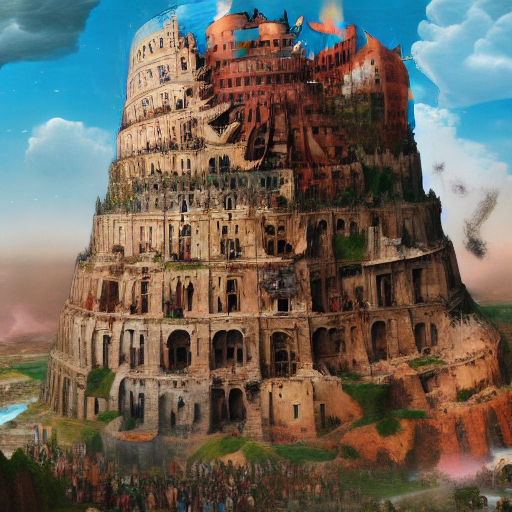}
        \caption{\ourName+Jump(30\%).}
        \vspace{0.5em}
    \end{subfigure}
    
    \fontsize{8}{10}\selectfont
    \begin{tabular}{lc>{\columncolor[gray]{0.9}}cccc}
        \toprule
        \multirow{2}{*}{GITS} & \multicolumn{5}{c}{Step} \\
        \cmidrule{2-6}
        & 6 & \textbf{7} & 8 & 9 & 10 \\
        \midrule
        Sampling trajectory                 & 56.86/\underline{26.49} & 24.52/\underline{28.58} & 14.15/\underline{29.56} & 11.44/\underline{29.95} & \textbf{12.01}/\underline{\textbf{30.11}} \\
        Denoising trajectory & 20.55/\underline{29.77} & \textbf{14.48}/\underline{\textbf{29.97}} & 13.16/\underline{30.06} & 12.45/\underline{30.09} & \textbf{12.01}/\underline{\textbf{30.11}} \\
        \bottomrule
    \end{tabular}
    \caption{\textit{Top:} Visual comparison of samples generated by SDv1.5 using 10-step DPM-Solver++(2M) under various time schedules: (a) uniform, (b) AYS-optimized~\citep{sabour2024align}, and (c) \ourName-optimized. (d) Results from GITS+Jump, which further reduces the number of steps by 30\%, are also presented for comparison. \textbf{Bottom}: FID and CLIP Scores (underlined) for GITS along the trajectories are reported.}
    \label{fig:sd_ays}
    \vspace{-0.5em}
\end{figure}

\textbf{Image generation.} As shown in Tables~\ref{tab:fid-1}-\ref{tab:fid-2}, our simple time re-allocation strategy based on iPNDM~\citep{zhang2023deis} consistently outperforms all existing ODE-based accelerated sampling methods with a significant margin, especially in the few NFE cases. In particular, all time schedules in these datasets are searched based on the Euler method, \ie, DDIM~\citep{song2021ddim}, but they are directly applicable to high-order methods such as iPNDM~\citep{zhang2023deis}. The trajectory regularity we uncovered guarantees that the schedule determined through 256 ``warmup'' samples is effective across all generated content. Furthermore, the experimental results suggest that identifying this trajectory regularity enhances our understanding of the mechanisms of diffusion models. This understanding opens avenues for developing tailored time schedules for more efficient diffusion sampling.
Note that we did not adopt the analytical first step (AFS) that replaces the first numerical step with an analytical Gaussian score to save one NFE, proposed in~\citep{dockhorn2022genie} and later used in~\citep{zhou2024fast}, as we found AFS is particularly effective only for datasets with low-resolution images. DPM-Solver-2~\citep{lu2022dpm} and AMED-Solver/Plugin~\citep{zhou2024fast} are thus not applicable with NFE $=5$ (marked as ``-'') in Table~\ref{tab:fid-1}. Ablation studies on AFS are provided in Table~\ref{tab:fid-full}. 

A concurrent work named AYS was recently proposed to optimize time schedules for sampling by minimizing the mismatch between the true backward SDE and its linear approximation, utilizing tools from stochastic calculus~\citep{sabour2024align}. In contrast, our \ourName exploits the strong trajectory regularity inherent in diffusion models and yields time schedules via dynamic programming, requiring only a small number of ``warmup'' samples. Our method also gets rid of the time-consuming Monte-Carlo computation in AYS~\citep{sabour2024align} and therefore is several orders of magnitude faster. In Figure~\ref{fig:sd_ays}, we compare image samples generated under different time schedules, using the publicly released colab code and its default setting.\footnote{\url{https://research.nvidia.com/labs/toronto-ai/AlignYourSteps/}} The text prompts used are ``a photo of an astronaut riding a horse on mars'' (1st row); ``a whimsical underwater world inhabited by colorful sea creatures and coral reefs'' (2nd row); ``a digital illustration of the Babel tower, 4k detailed, trending in artstation, fantasy vivid colors'' (3rd row). The evaluated FID results for each schedule are 14.28 (uniform), 12.48 (AYS), and 12.01 (\ourName). Besides, building on the significantly faster convergence of the denoising trajectory compared to the sampling trajectory, as discussed in Section~\ref{subsubsec:local_properties}, we propose ``GITS-Jump'' to further reduce sampling cost by 30\% (from 10-step to 7-step), almost without degradation in image quality.

\textbf{Time schedule comparison.} From Table~\ref{tab:time_schedule}, we can see that time schedules considerably affect the image generation performance. Compared with existing handcrafted schedules, the schedule we found better fits the underlying trajectory structure in the sampling of diffusion models and achieves smaller truncation errors with improved sample quality. 

\textbf{Running time.} Our strategy is highly efficient and incurs a very low computational cost, without requiring access to the real dataset. The procedure starts with a small number of initial ``warmup'' samples, followed by executing the given ODE-solver with both fine-grained and coarse-grained steps to construct the cost matrix for dynamic programming. Such a computation is performed only \textit{once} per dataset, and it yields optimal time schedules for different NFE budgets simultaneously, thanks to the optimal substructure property~\citep{cormen2022introduction}. As reported in Table~\ref{tab:consumed_time}, the entire algorithm takes less than or approximately one minute on datasets such as CIFAR-10, FFHQ, and ImageNet $64\times64$, and around 10 to 15 minutes for larger datasets such as LSUN Bedroom and LAION (Stable Diffusion), when evaluated on an NVIDIA A100 GPU.

\begin{table}[t]
    \caption{The comparison of FID results on CIFAR-10 across different time schedules.}
    \label{tab:time_schedule}
    \centering
    \fontsize{8}{10}\selectfont
    \begin{tabular}{lcccc}
        \toprule
        \multirow{2}{*}{TIME SCHEDULE} & \multicolumn{4}{c}{NFE} \\
        \cmidrule{2-5}
        & 5 & 6 & 8 & 10 \\
        \midrule
        DDIM-uniform                & 36.98 & 28.22 & 19.60 & 15.45 \\
        DDIM-logsnr                 & 53.53 & 38.20 & 24.06 & 16.43 \\
        DDIM-polynomial             & 49.66 & 35.62 & 22.32 & 15.69 \\
        \rowcolor[gray]{0.9} DDIM + \ourName (\textbf{ours})           & \textbf{28.05} & \textbf{21.04} & \textbf{13.30} & \textbf{10.09} \\
        \midrule
        iPNDM-uniform               & 17.34 & 9.75  & 7.56  & 7.35  \\
        iPNDM-logsnr                & 19.87 & 10.68 & 4.74  & 2.94  \\
        iPNDM-polynomial            & 13.59 & 7.05  & 3.69  & 2.77  \\
        \rowcolor[gray]{0.9} iPNDM + \ourName (\textbf{ours})          & \textbf{8.38} & \textbf{4.88} & \textbf{3.24} & \textbf{2.49} \\
        \bottomrule
    \end{tabular}
\end{table}

\begin{table}[t!]
    \caption{Time (in seconds) used at different stages of \ourName. ``warmup'' samples are generated using 60 NFE, and the NFE budget for dynamic programming is set to 10.}
    \label{tab:consumed_time}
    \centering
    \fontsize{8}{10}\selectfont
    \begin{tabular}{lcccc}
        \toprule
        \multirow{2}{*}{DATASET} & sample & cost & dynamic & total \\
        & generation & matrix & programming & time (s) \\
        \midrule
        CIFAR-10 $32\times 32$            & 27.47  & 5.29   & 0.015  & 32.78  \\
        FFHQ $64\times 64$            & 51.90  & 10.88  & 0.016  & 62.79  \\
        ImageNet $64\times 64$        & 71.77  & 13.28  & 0.016  & 85.07  \\
        LSUN Bedroom        & 517.63 & 122.13 & 0.015  & 639.78 \\
        LAION (SDv1.5)     & 877.62 & 24.00  & 0.016  & 901.62 \\
        \bottomrule
    \end{tabular}
\end{table}

\begin{table}[t!]
    \caption{Ablation study on the grid size of the dynamic programming-based time scheduling. 
    }
    \label{tab:ablation_teacher}
    \centering
    \fontsize{8}{10}\selectfont
    \begin{tabular}{lccccccc}
        \toprule
        \multirow{2}{*}{GRID SIZE} & \multicolumn{6}{c}{NFE BUDGET} \\
        \cmidrule{2-8}
        & 4 & 5 & 6 & 7 & 8 & 9 & 10 \\
        \midrule
        11       & 20.88 & 10.15 & 5.11 & 4.63 & 3.16 & 2.78 & 2.77  \\
        21       & 16.22 & 9.87  & 4.83 & 3.76 & 3.39 & 3.20 & 2.81  \\
        41       & 15.34 & 9.34  & 4.83 & 5.54 & 3.01 & 2.66 & 2.53  \\
        \rowcolor[gray]{0.9} \textbf{61} (default)      & 15.10 & 8.38  & 4.88 & 5.11 & 3.24 & 2.70 & 2.49  \\
        81       & 15.74 & 8.57  & 5.09 & 5.38 & 3.10 & 2.93 & 2.38  \\
        101      & 15.03 & 8.72  & 5.02 & 5.19 & 3.12 & 2.81 & 2.41  \\
        \midrule
        iPNDM & 24.82 & 13.59 & 7.05 & 5.08 & 3.69 & 3.17 & 2.77  \\
        \bottomrule
    \end{tabular}
\end{table}

\begin{table*}[t!]
	\caption{Ablation study on the ``warmup'' sample size. $\dagger$ indicates that a unique time schedule is searched for each of the 50k generated samples. This special case is more time-consuming while achieving similar results, owing to the strong trajectory regularity.}
	\label{tab:warmup_size}
	\centering
	\fontsize{8}{10}\selectfont
	\begin{tabular}{lcccc>{\columncolor[gray]{0.9}}cccc}
		\toprule
		\multirow{2}{*}{NFE} & \multicolumn{8}{c}{SAMPLE SIZE} \\
		\cmidrule{2-9}
		& 1$\dagger$ & 16 & 64 & 128 & \textbf{256} & 512 & 1024 & 2048 \\
		\midrule
		5  & 9.25 & 9.55$\pm$0.75 & 9.57$\pm$0.97 & 9.21$\pm$0.44 & 8.84$\pm$0.30 & 8.81$\pm$0.04 & 8.89$\pm$0.11 & 8.88$\pm$0.12 \\
		6  & 5.12 & 5.36$\pm$0.61 & 5.16$\pm$0.28 & 4.99$\pm$0.18 & 5.03$\pm$0.25 & 5.20$\pm$0.27 & 5.01$\pm$0.19 & 4.92$\pm$0.08 \\
		8  & 3.13 & 3.25$\pm$0.13 & 3.22$\pm$0.08 & 3.28$\pm$0.10 & 3.27$\pm$0.11 & 3.30$\pm$0.11 & 3.29$\pm$0.08 & 3.33$\pm$0.10 \\
		10 & 2.41 & 2.46$\pm$0.11 & 2.46$\pm$0.05 & 2.45$\pm$0.05 & 2.46$\pm$0.04 & 2.45$\pm$0.04 & 2.44$\pm$0.05 & 2.44$\pm$0.05 \\
		\bottomrule
	\end{tabular}
	\label{tab:reb_ablation}
\end{table*}

\textbf{Ablation studies.} We provide ablation studies on the number of ``warmup'' sample sizes and the grid size used for generating the fine-grained sampling trajectory in Tables~\ref{tab:ablation_teacher} and~\ref{tab:warmup_size}, respectively. The default experiments are conducted using iPNDM+GITS with the coefficient $\gamma=1.15$ on CIFAR-10. We also provide a sensitivity analysis of the coefficient in Table~\ref{tab:fid-full}. It is shown that the number of ``warmup'' samples is not a critical hyper-parameter, but reducing it generally increases the variance, as shown in Table~\ref{tab:warmup_size}. Due to subtle differences among sampling trajectories (see Figure~\ref{fig:traj_3d}), we recommend utilizing a reasonable number of ``warmup'' samples to determine the optimal time schedule, such that this time schedule works well for all the generated samples. 

\section{Related Work and Discussions}
\label{sec:related}

The popular variance-exploding (VE) SDEs~\citep{song2019ncsn,song2021sde} are taken as our main examples for analysis, which are equivalent to their variance-preserving (VP) counterparts according to It\^{o}'s lemma (see Lemma~\ref{lemma:ito_lemma} and Appendix~\ref{subsec:proof_lemma_ito}). The equivalence has been established in their corresponding PF-ODE (rather than SDE) forms by using the change-of-variable formula~(see Proposition 1 of \citep{song2021ddim} and Proposition 3 of \citep{zhang2023deis}).~\citet{karras2022edm} also presented a series of operational steps to reframe different models within a single framework (see Appendix C of~\citep{karras2022edm}). The equivalence guarantees the wide applicability of our conclusions, even though we focus on VE-SDEs. Besides, 
instead of training a noise-conditional score model with denoising score matching~\citep{vincent2011dsm,song2019ncsn,song2021maximum} or training a noise-prediction model to estimate the added noise in each step~\citep{ho2020ddpm,song2021ddim,nichol2021improved,vahdat2021lsgm,bao2022analytic}, we follow~\citep{kingma2021vdm,karras2022edm} and train a denoising model that predicts the reconstructed data from its corrupted version. With the help of simplified empirical PF-ODE~\eqref{eq:epf_ode}, we can characterize an implicit denoising trajectory, draw inspiration from classical non-parametric mean shift~\citep{fukunaga1975estimation,cheng1995mean,comaniciu2002mean}, and derive various trajectory properties. 

Denoising trajectories have been observed since the renaissance of diffusion models (see Figure 6 of~\citep{ho2020ddpm}) and later in Figure 3 of~\citep{kwon2023diffusion}, but they have not been systematically investigated, perhaps due to the indirect model parameterization. \citet{karras2022edm} were the first to note that the denoising output reflects the tangent of the sampling trajectory, consistent with our Corollary~\ref{cor:jump}. However, their work did not formulate it in differential-equation form nor examine how it controls the evolution of the sampling trajectory. In fact, \citet{karras2022edm} mentioned this property to argue that the sampling trajectory of \eqref{eq:epf_ode} is approximately linear, owing to the slow variation of the denoising output, and validated this intuition using a one-dimensional toy example. In contrast, we establish the equivalence of linear diffusion models and provide an in-depth analysis of high-dimensional sampling trajectories with real data, highlighting their intrinsically low dimensionality and pronounced geometric regularity. 

The mathematical foundations of the closed-form solution for the denoising score matching objective, or equivalently, the denoising autoencoder objective, were established more than half a century ago under the framework of \textit{empirical Bayes}~\citep{robbins1956eb}; see, for instance, Chapter 1 of~\citep{efron2010large}. Perhaps the earliest appearance of the closed-form solution~\eqref{eq:optimal} for a finite dataset within the literature of diffusion models is in Appendix B.3 of~\citep{karras2022edm}, where it was included for completeness and not connected to kernel density estimation (KDE) or any application. Subsequent works explicitly adopted the KDE-based interpretation (or, optimal denoising output) to analyze memorization and generalization in generative diffusion models~\citep{yi2023generalization,chen2023geometric,gu2023memorization,kadkhodaie2023generalization,scarvelis2023closed,li2024understanding,niedoba2024towards,kamb2024analytic}, listed here chronologically by their arXiv release dates. The early arXiv version\footnote{\url{https://arxiv.org/abs/2305.19947}} of our paper~\citep{chen2023geometric} was among the first, or at least concurrent with these studies. Importantly, our unique contribution lies in leveraging this well-established analytical formula to provide theoretical guarantees for the observed trajectory regularity and to extract additional insights from this approximate model in the context of diffusion-based generative models (Section~\ref{subsec:theoretical_analysis}). 

The trajectory regularity revealed in this paper is presented as an independent scientific discovery, supported by comprehensive empirical and theoretical analysis designed to reveal, characterize, and understand these principles. It is not intended merely as a prerequisite for specific algorithms; rather, it provides an intuitive yet grounded perspective on the underlying mechanics of diffusion models and helps explain the success of many widely used heuristic methods. (I) The observation that sampling trajectories follow a simple curvature and torsion function clarifies, for instance, why large steps can be safely taken at the beginning of sampling~\citep{dockhorn2022genie,zhou2024fast} without incurring significant truncation errors, and why polynomial time schedules outperform uniform schedules during sampling. Moreover, training efficiency improves when a larger computational training budget is allocated to the intermediate non-linear region of the trajectory and fewer to the near-linear regions~\citep{karras2022edm,chen2023importance,hang2024improved}, considering the trajectory shape. While previous work largely converged on these effective time/noise schedules through trial-and-error search~\citep{lu2022dpm,karras2022edm,sabour2024align,chen2023importance,hang2024improved}. (II) Our geometric perspective also provides a theoretical justification for the common heuristic of disabling classifier-free guidance~\citep{ho2022classifier,karras2024guiding,kynkaanniemi2024applying} at the beginning or end of the sampling process with minimal performance degradation~\citep{kynkaanniemi2024applying,castillo2025adaptive}. This phenomenon arises naturally, since the intermediate nonlinear region strongly influences trajectory orientation, whereas the early and late linear regions contribute little.

Finally, we describe a potential application of the discovered trajectory regularity for accelerating the sampling process. Different from most existing methods focusing on developing improved ODE-solvers~\citep{song2021ddim,karras2022edm,liu2022pseudo,lu2022dpm,zhang2023deis,zhao2023unipc,zhou2024fast} while selecting time schedules through a handcrafted or empirical tuning, we leverage the trajectory regularity of deterministic sampling dynamics to more effectively allocate discretized time steps. Our method achieves acceleration by several orders of magnitude compared with distillation-based sampling methods~\citep{luhman2021knowledge,salimans2022progressive,zheng2023fast,song2023consistency,kim2024consistency,zhou2024simple}. Although \citet{watson2021learning} were the first to employ dynamic programming for optimizing time schedules based on the decomposable nature of the ELBO objective, their method was shown to degrade sample quality. Moreover, while various theoretical studies have explored convergence analysis and score estimation of diffusion models, none of them have examined the trajectory-level properties that govern the sampling dynamics~\citep{de2022convergence,pidstrigach2022score,lee2023convergence,chen2023approximation,chen2023restoration}.

\section{Conclusion}
We reveal that a strong trajectory regularity consistently emerges in the deterministic sampling dynamics of diffusion-based generative models, regardless of the model implementation or generated content. This regularity is explained by characterizing and analyzing the implicit denoising trajectory, particularly its behavior under kernel density estimate-based data modeling. These insights into the underlying trajectory structure lead to an accelerated sampling method that enhances image synthesis quality with negligible computational overhead. We hope that the empirical and theoretical findings presented in this paper contribute to a deeper understanding of diffusion models and inspire further research into more efficient training paradigms and faster sampling algorithms.

\textbf{Future works}. We aim to explore deeper geometric regularities in sampling trajectories, characterize more precise structural patterns, and identify new applications inspired by these insights. Several promising directions are outlined below:
\begin{itemize}
    \item The geometric regularity of sampling trajectories analyzed in this paper may have potential connections to the behavior of random walk paths simulated in the forward process of diffusion models. In the limit of infinite dimensions and trajectory length, random walk-based trajectories exhibit similarly intriguing low-dimensional structures, with the explained variance taking an analytic form. Furthermore, their projections onto PCA subspaces follow Lissajous curves~\citep{moore2018high,antognini2018pca}. Extending existing theoretical results from the forward diffusion process to the backward sampling process remains an open problem. 
    \item A distinct three-stage pattern emerges in the sampling dynamics when using optimal score functions. Concurrently with our earlier manuscripts~\citep{chen2023geometric,chen2024trajectory}, \cite{biroli2024dynamical} introduced concepts from statistical physics, such as symmetry breaking and phase transitions, to characterize sampling dynamics. They provided analytic solutions for critical points in a simplified setting (two well-separated Gaussian mixture classes), and discussed the trade-off between generalization and memorization. It is particularly intriguing to bridge these theoretical insights with realistic diffusion models, especially incorporating conditional signals into the framework. 
    \item In practice, sampling in diffusion-based generative models is typically performed using general-purpose numerical solvers, sometimes augmented with learned solver coefficients or sampling schedules in a data-driven way~\citep{tong2025learning,frankel2025s4s}. Our findings reveal that each integral curve of the gradient field defined by a diffusion model lies within an extremely low-dimensional subspace embedded in the high-dimensional data space, with a regular trajectory shape shared across all initial conditions. While we present a preliminary attempt to exploit this structure, further investigation in this direction holds great promise.
\end{itemize}

\clearpage
\appendix

\setcounter{tocdepth}{2}
\tableofcontents
\allowdisplaybreaks
\clearpage

\section*{Appendix}
\section{Further readings}
\label{sec:further_readings}

\subsection{Details about Popular Linear SDEs}
\label{subsec:details_linear_SDEs}

In the literature, two specific forms of linear SDEs are widely used in large-scale diffusion models~\citep{ramesh2022hierarchical,saharia2022photorealistic,rombach2022ldm,balaji2022ediffi,peebles2023scalable,podell2024sdxl,xie2025sana}, namely, the variance-preserving (VP) SDE and the variance-exploding (VE) SDE~\citep{song2021sde,karras2022edm}. They correspond to the continuous versions of previously established models, \ie, DDPMs~\citep{ho2020ddpm,nichol2021improved} and NCSNs~\citep{song2019ncsn,song2020improved}, respectively. Next, we demonstrate that the original notations of VP-SDE, VE-SDE~\citep{song2021sde}, including recently proposed flow matching-based generative models~\citep{lipman2023flow,liu2023flow,albergo2023stochastic,neklyudov2023action,heitz2023iterative,esser2024scaling} can be recovered by properly setting the coefficients $s_t$ and $\sigma_t$:
\begin{itemize}
    \item VP-SDEs~\citep{ho2020ddpm,nichol2021improved,song2021ddim,song2021sde}: By setting $s_t=\sqrt{\alpha_t}$, $\sigma_t=\sqrt{(1-\alpha_t)/\alpha_t}$, $\beta_t = -\rmd \log \alpha_t/\rmd t$, and $\alpha_t\in (0, 1]$ as a decreasing sequence with $\alpha_0=1, \alpha_T\approx 0$, we have the transition kernel $p_{0t}(\bfz_t|\bfz_0)=\calN(\bfz_t; \sqrt{\alpha_t}\bfz_0, (1-\alpha_t) \bfI)$, or equivalently,
	\begin{equation}
        \bfz_t = \sqrt{\alpha_t}\ \bfz_0 + \sqrt{1-\alpha_t}\ \bfeps_t, \quad \bfeps_t\sim \calN(\mathbf{0}, \bfI), 
	\end{equation}
    with the forward linear SDE $\rmd \bfz_t = -\frac{1}{2}\beta_t\bfz_t \, \rmd t +  \sqrt{\beta_t}\rmd \bfw_t$.
	\item VE-SDEs~\citep{song2019ncsn,song2021sde}: By setting $s_t=1$, and $\sigma_0\approx 0$, $\sigma_T\gg 1$ for an increasing sequence $\sigma_t$, we have the transition kernel $p_{0t}(\bfz_t|\bfz_0)=\calN(\bfz_t; \bfz_0, \sigma_t^2 \bfI)$, or equivalently,
	\begin{equation}
        \bfz_t = \bfz_0 + \sigma_t\bfeps_t, \quad \bfeps_t \sim \calN(\mathbf{0}, \bfI),
	\end{equation}
    with the forward linear SDE $\rmd \bfz_t =  \sqrt{\rmd \sigma^2_t/\rmd t}\ \rmd \bfw_t$.
    \item A typical flow matching-based instantiation~\citep{lipman2023flow,liu2023flow,esser2024scaling} defines the transition kernel directly without relying on the forward linear SDE, with $s_t=1-t/T$, $\sigma_t=t/(T-t)$, \ie, $p_{0t}(\bfz_t|\bfz_0)=\calN(\bfz_t; (1-t/T)\bfz_0, (t/T)^2 \bfI)$, or equivalently,
    \begin{equation}
        \bfz_t=(1-t/T)\bfz_0 + t/T \bfeps_t, \quad \bfeps_t \sim \calN(\mathbf{0}, \bfI).
    \end{equation}
\end{itemize}

\subsection{Details about Score Matching}
\label{subsec:details_score_matching}

The score function $\nablazt \log p_t(\bfz_t)$ ~\citep{hyvarinen2005estimation,lyu2009interpretation}, which can be estimated with nonparametric score matching via \textit{kernel density estimation}, implicit score matching via \textit{integration by parts} formula~\citep{hyvarinen2005estimation}, sliced score matching via \textit{Hutchinson's trace estimator}~\citep{song2019sliced}, or more typically, denoising score matching (DSM) via \textit{mean squared regression}~\citep{vincent2011dsm,song2019ncsn,karras2022edm}. The DSM objective function of training a score-estimation model $s_{\bftheta}(\bfz_t; t)$ is 
\begin{equation}
    \calL_{\DSM}\left(\bftheta; \lambda(t)\right):=\int_{0}^{T} \lambda(t)\bbE_{\bfz_0 \sim p_d} \bbE_{\bfz_t \sim p_{0t}(\bfz_t|\bfz_0)} \lVert \overbrace{\nabla_{\bfz_t}\log p_{\bftheta}(\bfz_t; t}^{s_{\bftheta}(\bfz_t; t)}) - \nabla_{\bfz_t} \log p_{0t}(\bfz_t|\bfz_0)  \rVert^2_2 \rmd t.    
\end{equation}
The weighting function $\lambda(t)$ across different noise levels reflects our preference for visual quality or density estimation during model training \citep{song2021maximum,kim2022soft}. The optimal estimator $s_{\bftheta}^{\star}(\bfz_t; t)$ equals $\nablazt \log p_t(\bfz_t)$, and therefore we can use the converged score-estimation model as an effective proxy for the ground-truth score function. Lemma~\ref{lemma:pre_posterior} shows that we can also estimate the conditional expectation $\bbE(\bfz_0|\bfz_t)$ instead, typically using a denoising autoencoder (DAE)~\citep{vincent2008extracting,bengio2013generalized}. In fact, the mathematical essentials of the deep connection between DAE and DSM were established more than half a century ago under the framework of \textit{empirical Bayes}~\citep{robbins1956eb}; see, for example, Chapter 1 of a textbook~\citep{efron2010large}, or technical details given in Appendix~A of our early manuscript~\citep{chen2023geometric}. 

\subsection{Details about Numerical Approximation}
\label{subsec:details_numerical}
Given the empirical PF-ODE~\eqref{eq:pre_epf_ode}, 
generally, we have two formulas to calculate the exact solution from the current position $\bfz_{t_{n+1}}$ to the next position $\bfz_{t_{n}}$ ($t_0\leq t_{n}<t_{n+1}\leq t_N$) in the ODE-based sampling to obtain the sampling trajectory from $t_N$ to $t_0$. One is the direct integral from $t_{n+1}$ to $t_n$ 
\begin{equation}
	\label{eq:pre_exact_direct}
    \begin{aligned}
        \bfz_{t_n} = \bfz_{t_{n+1}}+\int_{t_{n+1}}^{t_n} \frac{\rmd \bfz_t}{\rmd t}\rmd t
        =\bfz_{t_{n+1}}+\int_{t_{n+1}}^{t_n}\left(\frac{\rmd \log s_t}{\rmd t}\bfz_t - \frac{1}{2} s^2_t\frac{\rmd \sigma^2_t}{\rmd t} \nablazt \log p_t(\bfz_t)\right),
    \end{aligned}
\end{equation}
and another leverages the semi-linear structure in the PF-ODE to derive the following equation with the \textit{variant of constants} formula~\citep{lu2022dpm,zhang2023deis}
\begin{equation}
    \label{eq:pre_exact_semi}    
    \begin{aligned}
    \bfz_{t_n} 
    &= \exp\left(\int_{t_{n+1}}^{{t_n}} f(t) \rmd t\right)\bfz_{t_{n+1}} -
    \int_{t_{n+1}}^{{t_n}} \left( \exp\left(\int_{t}^{{t_n}} f(r) \rmd r\right) \, \frac{g^2(t)}{2} \nablazt \log p_{t}(\bfz_t) \right) \rmd t\\
    &= \frac{s_{t_n}}{s_{t_{n+1}}}\bfz_{t_{n+1}} -
    s_{t_n}\int_{t_{n+1}}^{{t_n}} \left( s_t\sigma_t\sigma_t^{\prime}\nablazt \log p_t(\bfz_t)\right) \rmd t.
    \end{aligned}
\end{equation}
The above integrals, whether in \eqref{eq:pre_exact_direct} or \eqref{eq:pre_exact_semi}, involving the score function parameterized by a neural network, are generally intractable. Therefore, deterministic sampling in diffusion models centers on approximating these integrals with numerical methods in each step. In practice, various sampling strategies inspired by classic numerical methods have been proposed to solve the backward PF-ODE~\eqref{eq:pre_epf_ode}, including the Euler method~\citep{song2021ddim}, Heun's method~\citep{karras2022edm}, Runge-Kutta method~\citep{song2021sde,liu2022pseudo,lu2022dpm}, and linear multistep method~\citep{liu2022pseudo,lu2022dpmpp,zhang2023deis,zhao2023unipc,zhou2024fast,chen2024trajectory}. 
\section{Proofs}
\label{sec:proofs}


In this section, we provide detailed proofs of the lemmas, propositions, and theorems presented in the main content. 

\subsection{Proof of Lemma~\ref{lemma:pre_posterior}}
\label{subsec:proof_lemma_posterior}

\addtocounter{lemma}{-4}
\begin{lemma}
    Let the clean data be $\bfz_0\sim p_d$, and consider a transition kernel that adds Gaussian noise to the data, $p_{0t}(\bfz_t | \bfz_0)=\calN\left(\bfz_t; s_t\bfz_0, \, s^2_t\sigma^2_t\bfI\right)$. Then the score function is related to the posterior expectation by
    \begin{equation}
        \nablazt \log p_t(\bfz_t)=\left(s_t\sigma_t\right)^{-2}\left(s_t\bbE(\bfz_0|\bfz_t)-\bfz_t\right),    
    \end{equation}
    or equivalently, by linearity of expectation, 
    \begin{equation}
        \nablazt \log p_t(\bfz_t)=-\left(s_t\sigma_t\right)^{-1}\bbE_{p_{t0}(\bfz_0|\bfz_t)} \bfeps_t, \quad \quad \bfeps_t = \left(s_t\sigma_t\right)^{-1}(\bfz_t - s_t \bfz_0). 
    \end{equation} 
\end{lemma}

\begin{proof}
	We take derivative of $p_t(\bfz_t)=\int p_d(\bfz_0)p_{0t}(\bfz_t|\bfz_0)\rmd \bfz_0$ with respect to $\bfz_t$,
 \begin{equation}   
 \begin{aligned}
        \nablazt p_t(\bfz_t)
        &=\int \frac{\left(s_t\bfz_0-\bfz_t\right)}{s_t^2\sigma_t^2}	
        p_d(\bfz_0)p_{0t}(\bfz_t|\bfz_0)\rmd \bfz_0
        \\
        s_t^2\sigma_t^2\nablazt p_t(\bfz_t)&=\int s_t\bfz_0 p_d(\bfz_0)p_{0t}(\bfz_t|\bfz_0)\rmd \bfz_0 -\bfz_t p_t(\bfz_t)\\
        s_t^2\sigma_t^2\frac{\nablazt p_t(\bfz_t)}{p_t(\bfz_t)}&=s_t\int \bfz_0 p_t(\bfz_0|\bfz_t)\rmd \bfz_0 -\bfz_t \\
        \nablazt \log p_t(\bfz_t)&=\left(s_t\sigma_t\right)^{-2}\left(s_t\bbE(\bfz_0|\bfz_t)-\bfz_t\right).
    \end{aligned}
\end{equation} 
We further have 
\begin{equation}
    \begin{aligned}
        \nablazt \log p_t(\bfz_t)&=\left(s_t\sigma_t\right)^{-2}\left(s_t\bbE(\bfz_0|\bfz_t)-\bfz_t\right)=\left(s_t\sigma_t\right)^{-1}\bbE\left(\frac{s_t\bfz_0-\bfz_t}{s_t\sigma_t}|\bfz_t\right)\\
        &=-\left(s_t\sigma_t\right)^{-1}\bbE_{p_{t0}(\bfz_0|\bfz_t)}\bfeps_t,    
    \end{aligned}
\end{equation}
by linearity of expectation, where $\bfz_t = s_t \bfz_0 + s_t\sigma_t \bfeps_t$, $\bfeps_t \sim \calN(\mathbf{0}, \bfI)$ according to the transition kernel \eqref{eq:pre_kernel}.
\end{proof}

\subsection{Proof of Lemma~\ref{lemma:pre_dae}}
\label{subsec:proof_lemma_dae}

\begin{lemma}
	The optimal estimator $r_{\bftheta}^{\star}\left(\bfz_t; t\right)$ for the denoising autoencoder objective, also known as the Bayesian least squares estimator or minimum mean square error (MMSE) estimator, is given by $\bbE\left(\bfz_0|\bfz_t\right)$.
\end{lemma}
\begin{proof}
	The solution is easily obtained by setting the derivative of $\calL_{\DAE}$ equal to zero.
    For each noisy sample $\bfz_t$, we have
	\begin{equation}
		\begin{aligned}
			\nabla_{r_{\bftheta}\left(\bfz_t; t\right)} \calL_{\DAE}&=\mathbf{0}\\
			\int{p_{t0}(\bfz_0|\bfz_t)}\left(r_{\bftheta}^{\star}\left(\bfz_t; t\right) - \bfz_0\right)\rmdz_0
            &=\mathbf{0}\\
			\int{p_{t0}(\bfz_0|\bfz_t)}r_{\bftheta}^{\star}\left(\bfz_t; t\right) \rmdz_0&=\int{p_{t0}(\bfz_0|\bfz_t)}\bfz_0\rmdz_0\\	
			r_{\bftheta}^{\star}\left(\bfz_t; t\right)&= \bbE\left(\bfz_0 | \bfz_t\right).
		\end{aligned}
	\end{equation}
\end{proof}

\subsection{Proof of Lemma~\ref{lemma:pre_optimal_denoiser}}
\label{subsec:proof_lemma_optimal}

\begin{lemma}
    Let $\calD \coloneqq \{\bfy_{i} \in \mathbb{R}^d\}_{i \in \calI}$ denote a dataset of $|\calI|$ i.i.d.\ data points drawn from $p_d$. When training a denoising autoencoder with the empirical data distribution $\hat{p}_{d}$, the optimal denoising output is a convex combination of original data points, namely 
    \begin{equation}
    	\begin{aligned}
            r_{\bftheta}^{\star}(\bfz_t; t)=\min_{r_{\bftheta}} \bbE_{\bfy \sim \hat{p}_d} \bbE_{\bfz_t \sim p_{0t}(\bfz_t|\bfy)} \lVert r_{\bftheta}(\bfz_t; t) - \bfy  \rVert^2_2
    		=\sum_{i} \frac{\exp \left(-\lVert \bfz_t - \bfy_i \rVert^2_2/2\sigma_t^2\right)}{\sum_{j}\exp \left(-\lVert \bfz_t - \bfy_j \rVert^2_2/2\sigma_t^2\right)} \bfy_i,
    	\end{aligned}
    \end{equation}
    where $\hat{p}_{d}(\bfy)$ is the sum of multiple \textit{Dirac delta functions}, \ie, $\hat{p}_{d}(\bfy)=(1/|\calI|)\sum_{i\in\calI}\delta(\lVert \bfy - \bfy_i\rVert)$.
\end{lemma}
\begin{proof}    
Based on Lemmas~\ref{lemma:pre_posterior} and~\ref{lemma:pre_dae}, and the Gaussian kernel density estimate (KDE) $\hat{p}_t(\bfz_t)=\int p_{0t}(\bfz_t|\bfy)\hat{p}_{d}(\bfy) = \frac{1}{|\calI|}\sum_{i}\calN\left(\bfz_t;\bfy_i, \sigma_t^2\bfI\right)$, the optimal denoising output is
\begin{equation}
    \begin{aligned}
        r_{\bftheta}^{\star}\left(\bfz_t; t\right) =
        \bbE\left(\bfz_0|\bfz_t\right)&=\bfz_t + \sigma_t^2\nablazt \log \hat{p}_t(\bfz_t)\\
        &=\bfz_t + \sigma_t^2 \sum_{i}\frac{\nablazt\calN\left(\bfz_t;\bfy_i, \sigma_t^2\bfI\right)}{\sum_{j}\calN\left(\bfz_t;\bfy_j, \sigma_t^2\bfI\right)}\\
        &=\bfz_t + \sigma_t^2 \sum_{i}\frac{\calN\left(\bfz_t;\bfy_i, \sigma_t^2\bfI\right)}{\sum_{j}\calN\left(\bfz_t;\bfy_j, \sigma_t^2\bfI\right)}\left( \frac{\bfy_i-\bfz_t}{\sigma_t^2}\right)\\
        &=\bfz_t + \sum_{i}\frac{\calN\left(\bfz_t;\bfy_i, \sigma_t^2\bfI\right)}{\sum_{j}\calN\left(\bfz_t;\bfy_j, \sigma_t^2\bfI\right)}\left(\bfy_i-\bfz_t\right)\\ 
        &=\sum_{i}\frac{\calN\left(\bfz_t;\bfy_i, \sigma_t^2\bfI\right)}{\sum_{j}\calN\left(\bfz_t;\bfy_j, \sigma_t^2\bfI\right)}\bfy_i\\
        &=\sum_{i} \frac{\exp \left(-\lVert \bfz_t - \bfy_i \rVert^2_2/2\sigma_t^2\right)}{\sum_{j}\exp \left(-\lVert \bfz_t - \bfy_j \rVert^2_2/2\sigma_t^2\right)} \bfy_i,
    \end{aligned}
\end{equation}
where each weight is calculated based on the time-scaled and normalized Euclidean distance between the input $\bfz_t$ and each data point $\bfy_i$ and the sum of coefficients equals one.
\end{proof}

\subsection{Proof of Lemma~\ref{lemma:ito_lemma}}
\label{subsec:proof_lemma_ito}
\begin{lemma}
	The linear diffusion process defined as \eqref{eq:pre_new_forword_sde} can be transformed into its VE counterpart with the change of variables $\bfx_t=\bfz_t / s_t$, keeping the SNR function unchanged.
\end{lemma}
\begin{proof}
    We adopt the change-of-variables formula $\bfx_t=\bfphi(t, \bfz_t)=\bfz_t / s_t$ with $\bfphi:[0, T]\times\bbR^n \rightarrow \bbR^n$, and denote the $i$-th dimension of $\bfz_t$, $\bfx_t$ and $\bfw_t$ as $\bfz_t[i]$, $\bfx_t[i]$, and $\bfw_t[i]$ respectively; $\bfphi = [\phi_1, \cdots, \phi_i, \cdots, \phi_n]^T$ with a twice differentiable scalar function $\phi_i(t, z)=z/s_t$ of two real variables $t$ and $z$. Since each dimension of $\bfz_t$ is independent, we can apply It\^{o}'s lemma~\citep{oksendal2013stochastic} to each dimension with $\phi_i(t, \bfz_t[i])$ separately. We have
    \begin{equation}
        \frac{\partial \phi_i}{\partial t}=-\frac{z}{s_t^2}\frac{\rmd s_t}{\rmd t}, \quad
        \frac{\partial \phi_i}{\partial z}=\frac{1}{s_t}, \quad
        \frac{\partial^2 \phi_i}{\partial z^2}=0, \quad
        \rmd \bfz_t[i]=\frac{\rmd \log s_t}{\rmd t}\bfz_t[i]\rmd t + s_t\sqrt{\frac{\rmd \sigma^2_t}{\rmd t}} \rmd \bfw_t[i],
    \end{equation}
    then
    \begin{equation}
    \begin{aligned}
        \rmd \phi_i(t, \bfz_t[i])&=\left(\frac{\partial \phi_i}{\partial t}+ f(t)\bfz_t[i]\frac{\partial \phi_i}{\partial z}+\frac{g^2(t)}{2}\frac{\partial^2 \phi_i}{\partial z^2}\right) \rmd t + g(t)\frac{\partial \phi_i}{\partial z}\rmd \bfw_t[i]\\
        &=\left(\frac{\partial \phi_i}{\partial t}+ \frac{g^2(t)}{2}\frac{\partial^2 \phi_i}{\partial z^2}\right) \rmd t 
        +\frac{\partial \phi_i}{\partial z}\rmd \bfz_t[i]\\
        \rmd \bfx_t[i] & =-\frac{\bfz_t[i]}{s_t}\frac{\rmd \log s_t}{\rmd t}\rmd t+ \frac{1}{s_t}\left(\frac{\rmd \log s_t}{\rmd t}\bfz_t[i]\rmd t + s_t\sqrt{\frac{\rmd \sigma^2_t}{\rmd t}} \rmd w_t\right)\\
        \rmd \bfx_t[i] &= \sqrt{\rmd \sigma^2_t/\rmd t}\ \rmd \bfw_t[i] \quad \Rightarrow \quad
        \rmd \bfx_t = \sqrt{\rmd \sigma^2_t/\rmd t}\ \rmd \bfw_t,
    \end{aligned}
    \end{equation}
with the initial condition $\bfx_0=\bfz_0 \sim p_d$. Since $\sigma_t$ in the above VE-SDE ($\bfx$-space) is exactly the same as that used in the original SDE ($\bfz$-space, \eqref{eq:pre_new_forword_sde}), the SNR remains unchanged. 
\end{proof}
We also establish the connections between their score functions and sampling behaviors. Similarly, we have the score function ($t\in [0,T]$) 
\begin{equation}
	\begin{aligned}
        \nablaxt \log p_t(\bfx_t)
        &=\nablaxt \log \int \calN\left(\bfx_t; \bfx_0, \sigma_t^2\bfI\right)p_d(\bfx_0) \rmdx_0 \\ 
		&=\nabla_{\bfz_t/s_t} \log \int \calN\left(\bfz_t/s_t ; \bfx_0, \sigma_t^2\bfI\right)p_d(\bfx_0) \rmdx_0 \\ 	
        &=s_t\nablazt \log \int s_t^{d}\calN\left(\bfz_t; s_t\bfz_0, s_t^2\sigma_t^2\bfI\right)p_d(\bfz_0) \rmdz_0
		=s_t\nablazt \log p_t(\bfz_t).
	\end{aligned}
 \end{equation}

\begin{corollary}
    With the same numerical method, the results obtained by solving \eqref{eq:pre_exact_direct} or \eqref{eq:pre_exact_semi} are not equal in the general cases. But they become exactly the same by first transforming the formulas into the $\bfx$-space and then perform numerical approximation.  
\end{corollary}
\begin{corollary}
    \label{cor:equal}
	With the same numerical method, the result obtained by solving \eqref{eq:pre_exact_semi} in the $\bfz$-space is exactly the same as the result obtained by solving \eqref{eq:pre_exact_direct} or \eqref{eq:pre_exact_semi} in the $\bfx$-space.
\end{corollary}
\begin{proof}
    Given the sample $\bfz_{t_n}$ obtained by solving the equation \eqref{eq:pre_exact_semi} in $\bfz$-space starting from $\bfz_{t_{n+1}}$, we demonstrate that $\bfz_{t_n}/s_{t_n}$ is exactly equal to sampling with the equation~\eqref{eq:pre_exact_direct} in $\bfx$-space starting from $\bfx_{t_{n+1}}=\bfz_{t_{n+1}}/s_{t_{n+1}}$ to $\bfx_{t_n}$. We have
\begin{equation}
	\begin{aligned}
        \bfz_{t_n} &= s_{t_n}\left(\frac{\bfz_{t_{n+1}}}{s_{t_{n+1}}}-\int_{t_{n+1}}^{t_n} s_t\sigma_t\sigma_t^{\prime} \nablazt \log p_t(\bfz_t)\rmd t \right)\\
        &= s_{t_n}\left(\bfx_{t_{n+1}}+\int_{t_{n+1}}^{t_n} - \sigma_t\nablaxt \log p_t(\bfx_t)\sigma_t^{\prime} \rmd t\right)
		= s_{t_n}\left(\bfx_{t_{n+1}}+\int_{t_{n+1}}^{t_n} \frac{\rmd \bfx_t}{\rmd t}\rmd t\right) = s_{t_{n}}\bfx_{t_n}.
	\end{aligned}
\end{equation}
\end{proof}

\subsection{Proof of Proposition~\ref{prop:convex}}
\label{subsec:proof_generalized_denoising_output}
All following proofs are conducted in the context of a VE-SDE $\rmdx_t =\sqrt{2t} \, \rmd \bfw_t$, \ie, $\sigma_t=t$ for notation simplicity, and the sampling trajectory always starts from $\hatx_{t_N}\sim \calN(\mathbf{0}, T^2\bfI)$ and ends at $\hatx_{t_0}$. 
The PF-ODEs of the sampling trajectory and denoising trajectory are 
$\frac{\rmdx_t}{\rmd t} = \bfeps_{\bftheta}(\bfx_t; t) = \frac{\bfx_t - r_{\bftheta}(\bfx_t; t)}{t}$, and $\frac{\rmd r_{\bftheta}(\bfx_t; t)}{\rmd t} = - t \frac{\rmd^2 \bfx_t}{\rmd t^2}$, respectively. 


\addtocounter{proposition}{-4}
\begin{proposition}
	Given the probability flow ODE~\eqref{eq:epf_ode} and a current position $\hatx_{t_{n+1}}$, $n\in[0, N-1]$ in the sampling trajectory, the next position $\hatx_{t_{n}}$ predicted by a $k$-th order Taylor expansion with the time step size $\sigma_{t_{n+1}}-\sigma_{t_n}$ is 
	\begin{equation}
        \begin{aligned}
		\hatx_{t_{n}}&=\frac{\sigma_{t_n}}{\sigma_{t_{n+1}}} \hatx_{t_{n+1}} +  \frac{\sigma_{t_{n+1}}-\sigma_{t_n}}{\sigma_{t_{n+1}}} \calR_{\bftheta}(\hatx_{t_{n+1}}),
        \end{aligned}
	\end{equation}
which is a convex combination of $\hatx_{t_{n+1}}$ and the generalized denoising output
\begin{equation}
    \calR_{\bftheta}(\hatx_{t_{n+1}})=r_{\bftheta}(\hatx_{t_{n+1}})- \sum_{i=2}^{k}\frac{1}{i!}\frac{\rmd^{(i)} \bfx_t}{\rmd \sigma_t^{(i)}}\Big|_{\hatx_{t_{n+1}}}\sigma_{t_{n+1}}(\sigma_{t_n} - \sigma_{t_{n+1}})^{i-1}.
\end{equation}
In particular, we have $\calR_{\bftheta}(\hatx_{t_{n+1}})=r_{\bftheta}(\hatx_{t_{n+1}})$ for the Euler method ($k=1$), and $\calR_{\bftheta}(\hatx_{t_{n+1}})=r_{\bftheta}(\hatx_{t_{n+1}})+\frac{\sigma_{t_{n}}-\sigma_{t_{n+1}}}{2}\frac{\rmd r_{\bftheta}(\hatx_{t_{n+1}})}{\rmd \sigma_t}$ for second-order numerical methods ($k=2$). 
\end{proposition}
\begin{proof}
    The $k$-th order Taylor expansion at $\hatx_{t_{n+1}}$ is
    \begin{equation}
    \begin{aligned}
    \hatx_{t_{n}}
    &= \sum_{i=0}^{k}\frac{1}{i!}\frac{\rmd^{(i)} \bfx_t}{\rmd t^{(i)}}\Big|_{\hatx_{t_{n+1}}}(t_n - t_{n+1})^{i}\\
    &=\hatx_{t_{n+1}} +  (t_n - t_{n+1})\frac{\rmdx_t}{\rmd t}\Big|_{\hatx_{t_{n+1}}}
    +\sum_{i=2}^{k}\frac{1}{i!}\frac{\rmd^{(i)} \bfx_t}{\rmd t^{(i)}}\Big|_{\hatx_{t_{n+1}}}(t_n - t_{n+1})^i\\
    &=\hatx_{t_{n+1}} +  \frac{t_n - t_{n+1}}{t_{n+1}} \left(\hatx_{t_{n+1}} - r_{\bftheta}(\hatx_{t_{n+1}})\right)
    + \sum_{i=2}^{k}\frac{1}{i!}\frac{\rmd^{(i)} \bfx_t}{\rmd t^{(i)}}\Big|_{\hatx_{t_{n+1}}}(t_n - t_{n+1})^{i}\\
    &=\frac{{t_n}}{{t_{n+1}}} \hatx_{t_{n+1}} +  \frac{{t_{n+1}}-{t_n}}{{t_{n+1}}} \calR_{\bftheta}(\hatx_{t_{n+1}}), 
    \end{aligned}
\end{equation}
where $\calR_{\bftheta}(\hatx_{t_{n+1}})=r_{\bftheta}(\hatx_{t_{n+1}})- \sum_{i=2}^{k}\frac{1}{i!}\frac{\rmd^{(i)} \bfx_t}{\rmd t^{(i)}}\Big|_{\hatx_{t_{n+1}}}{t_{n+1}}(t_n - t_{n+1})^{i-1}$. As for the first-order approximation ($k=1$), we have $\calR_{\bftheta}(\hatx_{t_{n+1}})=r_{\bftheta}(\hatx_{t_{n+1}})$. As for the second-order approximation ($k=2$), we have
\begin{equation}
    \label{eq:second-order}
    \calR_{\bftheta}(\hatx_{t_{n+1}})=r_{\bftheta}(\hatx_{t_{n+1}})-\frac{t_{n+1}(t_{n}-t_{n+1})}{2}\frac{\rmd^{2}\bfx_t }{\rmd t^2}\Big|_{\hatx_{t_{n+1}}}=r_{\bftheta}(\hatx_{t_{n+1}})+\frac{t_{n}-t_{n+1}}{2} \frac{\rmd r_{\bftheta}(\hatx_{t_{n+1}})}{\rmd t}.
\end{equation} 
\end{proof}

\begin{corollary}
	The denoising output $r_{\bftheta}(\hatx_{t_{n+1}})$ reflects the prediction made by a single Euler step from $\hatx_{t_{n+1}}$ with the time step size $t_{n+1}$.
\end{corollary}
\begin{proof}
	The prediction of such an Euler step equals to
    \begin{equation}
        \hatx_{t_{n+1}} + (0 - {t_{n+1}}) \left(\hatx_{t_{n+1}} - r_{\bftheta}(\hatx_{t_{n+1}})\right)/{t_{n+1}}=r_{\bftheta}(\hatx_{t_{n+1}}).    
    \end{equation}
\end{proof}

\subsection{Proof of Proposition~\ref{prop:eps_norm}}
\label{subsec:proof_eps_norm}
\begin{lemma}[see Section 3.1 in \citep{vershynin2018high}]
    \label{lemma:concentration}
    Given a high-dimensional isotropic Gaussian noise $\bfz \sim \calN(\mathbf{0} ; {\sigma^2\bfI_d})$, $\sigma>0$, we have $\bbE \left\lVert \bfz \right\rVert^2=\sigma^2 d$, and with high probability, $\bfz$ stays within a ``thin spherical shell'': $\lVert \bfz \rVert = \sigma \sqrt{d} \pm O(1)$. 
\end{lemma}
\begin{proof}
	We denote $\bfz_i$ as the $i$-th dimension of random variable $\bfz$, then the expectation and variance is $\bbE \left[\bfz_i\right]=0$, $\bbV \left[\bfz_i\right]=\sigma^2$, respectively. The fourth central moment is $\bbE\left[ \bfz_i^4\right] = 3{\sigma ^4}$. Additionally, 
    \begin{equation}
        \begin{aligned}
            \bbE \left[\bfz_i^2\right]=\bbV \left[\bfz_i\right]+\bbE\left[\bfz_i\right]^2
            &=\sigma^2,\qquad
            \bbE \left[\lVert \bfz \rVert^2\right]=\bbE \left[\sum_{i=1}^{d}\bfz_i^2\right]=\sum_{i=1}^{d}\bbE \left[\bfz_i^2\right]=\sigma^2 d, \\
            \bbV\left[{\lVert \bfz  \rVert^2} \right] 
            &= \mathbb{E} \left[ {\lVert \bfz  \rVert}^4\right] - \left(\bbE\left[
    		  \lVert \bfz  \rVert^2 \right] \right)^2 
    		= 2d \sigma ^4,
            \end{aligned}
    \end{equation}
	Then, we have 
	\begin{equation}
		\begin{aligned}
		\bbE \left[ 
		\lVert \bfx + \bfz \rVert^2 - 
		\lVert \bfx \rVert^2
		\right] 
		& = \bbE \left[ 
		\lVert \bfz \rVert^2 + 2 \bfx^T\bfz 
		\right]
        =
        \bbE \left[ 
		\lVert \bfz \rVert^2
		\right]=\sigma^2 d. 
	\end{aligned}	
	\end{equation}

Furthermore, the standard deviation of $\lVert \bfz  \rVert^2$ is ${\sigma ^2}\sqrt {2d}$, which means
\begin{equation}
	\begin{aligned}
		\lVert \bfz  \rVert^2 = {\sigma ^2}d \pm {\sigma ^2}\sqrt {2d}  = {\sigma ^2}d \pm O\left( \sqrt{d} \right), \qquad
		\lVert \bfz \rVert \hspace{0.4em} = \sigma \sqrt{d} \pm O\left( 1 \right),
	\end{aligned}
\end{equation}
holds with high probability.
\end{proof}

\begin{proposition}
    The magnitude of $\bfeps_{\bftheta}^{\star}(\bfx_t; t)$ concentrates around $\sqrt{d}$, where $d$ denotes the data dimension. Consequently, the total length of the sampling trajectory is approximately $\sigma_T \sqrt{d}$, where $\sigma_T$ denotes the maximum noise level.
\end{proposition}
\begin{proof}
We next provide a sketch of proof. Suppose the data distribution lies in a smooth real low-dimensional manifold with the intrinsic dimension as $m$. According to the \textit{Whitney embedding theorem}~\citep{whitney1936differentiable}, it can be smoothly embedded in a real $2m$ Euclidean space. We then decompose each $\bfeps_{\bftheta}^{\star}\in \bbR^{d}$ vector as $\bfeps_{\bftheta, \parallel}^{\star}$ and $\bfeps_{\bftheta, \perp}^{\star}$, which are parallel and perpendicular to the $2m$ Euclidean space, respectively. Therefore, we have $\lVert \bfeps_{\bftheta}^{\star} \rVert_2=\lVert \bfeps_{\bftheta, \parallel}^{\star}+\bfeps_{\bftheta, \perp}^{\star}\rVert_2
\geq\lVert\bfeps_{\bftheta, \perp}^{\star}\rVert_2\approx\sqrt{d-2m}$. 
\begin{figure*}[t]
    \centering
    \begin{subfigure}[t]{0.45\textwidth}
        \centering
        \includegraphics[width=\columnwidth]{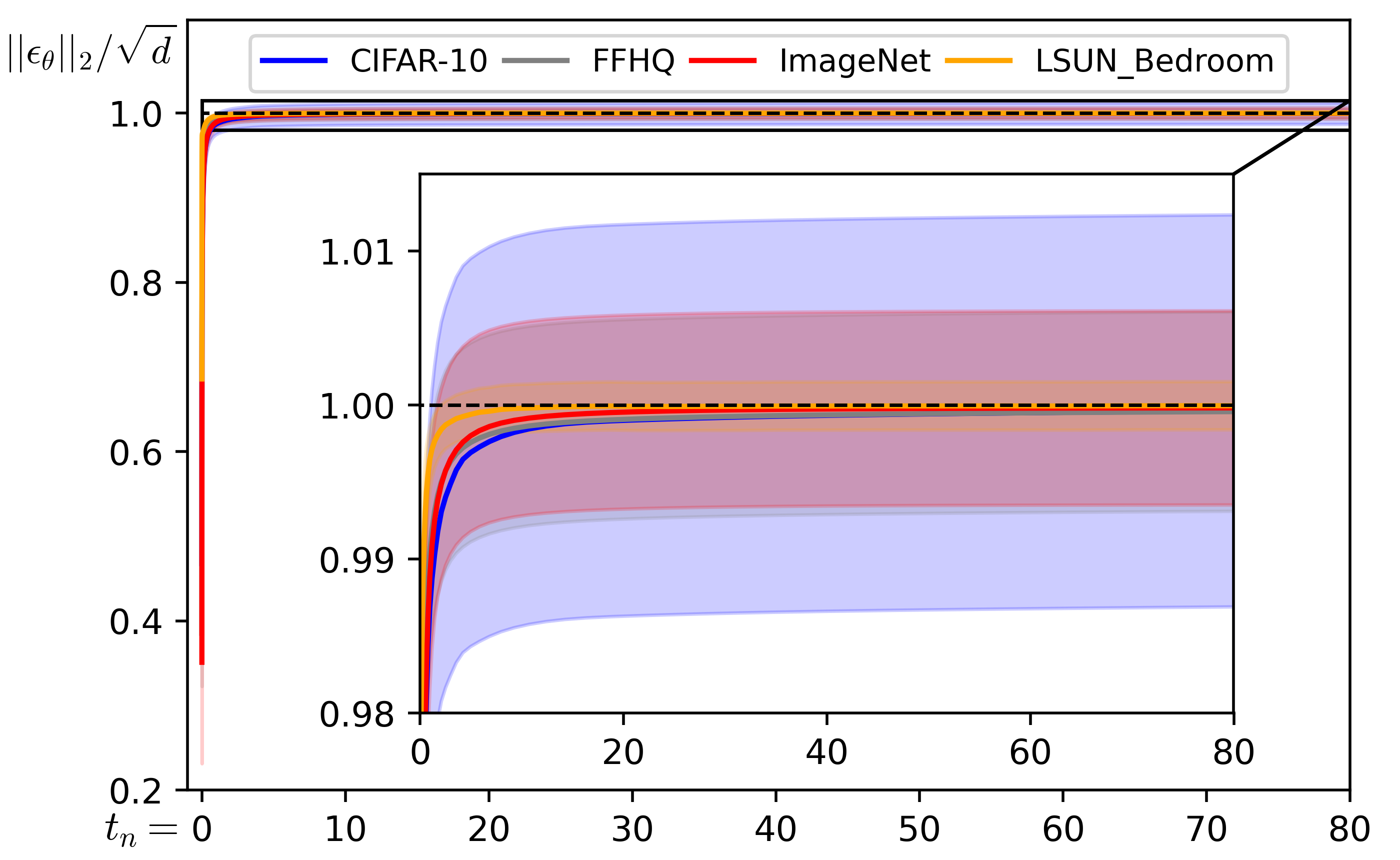}
        \caption{The $L^2$ norm of $\bfeps_{\bftheta}$.}
        \label{fig:eps}
    \end{subfigure}
    \quad
    \begin{subfigure}[t]{0.45\textwidth}
        \centering
        \includegraphics[width=\columnwidth]{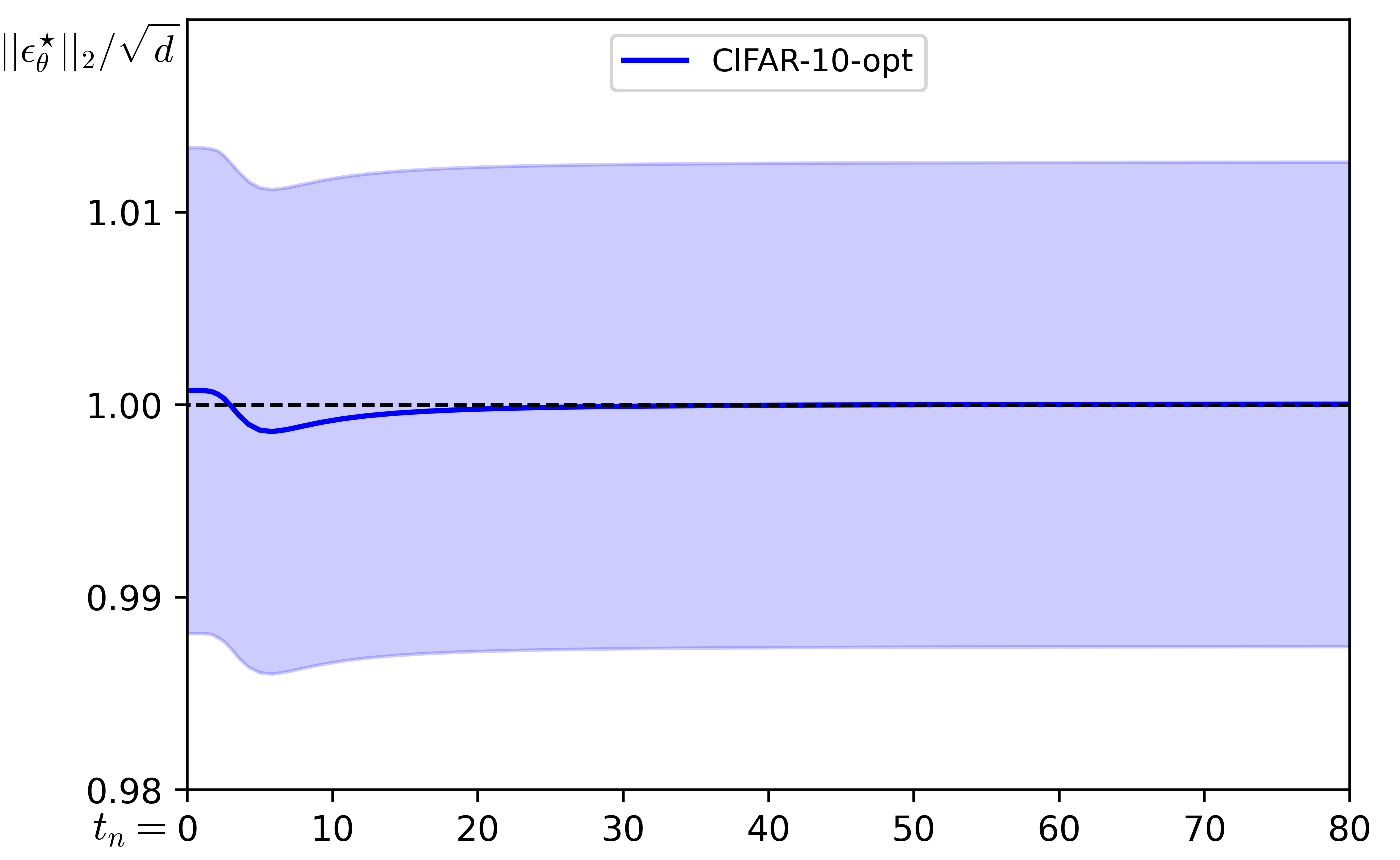}
        \caption{The $L^2$ norm of $\bfeps_{\bftheta}^{\star}$.}
        \label{fig:eps_opt}
    \end{subfigure}
    \caption{The optimal noise prediction satisfies $\lVert \bfeps_{\bftheta}^{\star}\rVert_2\approx \sqrt{d}$ throughout the entire sampling process, as guaranteed by theoretical results. The actual noise prediction $\lVert \bfeps_{\bftheta}\rVert_2$ also remains $\sqrt{d}$ for most time steps, but exhibits a noticeable shrinkage in the final stage, when the time step approaches zero. This norm shrinkage almost does not affect the trajectory length, as the discretized time steps in the final stage are extremely small.}
    \label{fig:traj_len}
\end{figure*}

We provide a upper bound for the $\lVert \bfeps_{\bftheta}^{\star} \rVert_2$ below
\begin{equation}
    \begin{aligned}
        \bbE_{p_t(\bfx_t)}\lVert \bfeps_{\bftheta}^{\star} \rVert_2
        &=\bbE_{p_t(\bfx_t)}\lVert \frac{\bfx_t - r_{\bftheta}^{\star}(\bfx_t)}{\sigma_t} \rVert_2
        =\bbE_{p_t(\bfx_t)}\lVert \frac{\bfx_t - \bbE(\bfx_0|\bfx_t)}{\sigma_t} \rVert_2\\
        &=\bbE_{p_t(\bfx_t)}\lVert \bbE(\frac{\bfx_t - \bfx_0}{\sigma_t}|\bfx_t) \rVert_2
        =\bbE_{p_t(\bfx_t)}\lVert \bbE_{p_{t0}(\bfx_0|\bfx_t)}\bfeps\rVert_2\\
        &\leq\bbE_{p_t(\bfx_t)}\bbE_{p_{t0}(\bfx_0|\bfx_t)} \lVert \bfeps\rVert_2=\bbE_{p_0(\bfx_0)}\bbE_{p_{0t}(\bfx_t|\bfx_0)} \lVert \bfeps\rVert_2\\
        &\approx \sqrt{d} \qquad\qquad (\text{concentration of measure, Lemma~\ref{lemma:concentration}}). 
    \end{aligned}
\end{equation}
Additionally, the variance of $\lVert \bfeps_{\bftheta}^{\star} \rVert_2$ is relatively small.
\begin{equation}
    \begin{aligned}
        \mathrm{Var}_{p_t(\bfx_t)}\lVert \bfeps_{\bftheta}^{\star} \rVert_2
        &=\mathrm{Var}_{p_t(\bfx_t)}\lVert \bbE_{p_{t0}(\bfx_0|\bfx_t)}\bfeps\rVert_2
        = \bbE_{p_t(\bfx_t)}\lVert \bbE_{p_{t0}(\bfx_0|\bfx_t)}\bfeps\rVert_2^2 - \left[\bbE_{p_t(\bfx_t)}\lVert \bbE_{p_{t0}(\bfx_0|\bfx_t)}\bfeps\rVert_2\right]^2\\
        &\leq \bbE_{p_t(\bfx_t)}\bbE_{p_{t0}(\bfx_0|\bfx_t)}\lVert \bfeps\rVert_2^2 - (d-2m)
        = \bbE_{p_0(\bfx_0)}\bbE_{p_{0t}(\bfx_0|\bfx_t)}\lVert \bfeps\rVert_2^2 - (d-2m)\\
        &= d - (d-2m)\\
        &= 2m.
    \end{aligned}
\end{equation}
Therefore, the standard deviation of $\lVert \bfeps_{\bftheta}^{\star} \rVert_2$ is upper bounded by $\sqrt{2m}$.
Since $d\gg m$, we can conclude that in the optimal case, the magnitude of vector field is approximately constant, \ie, $\lVert \bfeps_{\bftheta}^{\star} \rVert_2\approx \sqrt{d}$. The total sampling trajectory length is
$\sum_{n=0}^{N-1} (\sigma_{t_{n+1}} - \sigma_{t_{n}}) \lVert \bfeps_{\bftheta}^{\star}(\bfx_{t_{n+1}}) \rVert_2\approx \sigma_{T} \sqrt{d}$. Empirical verification is provided in Figure~\ref{fig:traj_len}.
\end{proof}

\begin{table*}[t!]
    \caption{PCA ratios computed by directly performing dimension reduction of the original sampling trajectory. This differs from PCA ratios in the main context that measures the reconstruction of the orthogonal complement of the displacement vector of each trajectory.}
    \label{tab:sampler}
    \centering
    \resizebox{0.9\textwidth}{!}{
    \begin{tabular}{llll}
    \toprule
    \multicolumn{1}{l}{\bf Dataset} & Top-1 & Top-2 & Top-3 \\
    \midrule
    CIFAR-10 ($32\times 32$)        & 99.99729$\pm$0.00135\% & 99.99998$\pm$0.00018\% & 99.999994$\pm$0.000027\% \\
    FFHQ ($64\times 64$)            & 99.99843$\pm$0.00095\% & 99.99999$\pm$0.00016\% & 99.999994$\pm$0.000024\% \\ 
    ImageNet ($64\times 64$)        & 99.99805$\pm$0.00121\% & 99.99998$\pm$0.00019\% & 99.999994$\pm$0.000033\% \\
    SDv1.5, MS-COCO ($64\times 64$) & 99.94101$\pm$0.01289\% & 99.99432$\pm$0.00150\% & 99.999195$\pm$0.000475\% \\
    LSUN Bedroom ($256\times 256$)  & 99.99953$\pm$0.00083\% & 99.99999$\pm$0.00013\% & 99.999994$\pm$0.000025\% \\
    \bottomrule
    \end{tabular}}
\end{table*}

\subsection{Proof of Proposition~\ref{prop:likelihood}}
\label{subsec:proof_likelihood}
\begin{proposition}
    In deterministic sampling with the Euler method, the sample likelihood is non-decreasing, \ie, $\forall n\in[1, N]$, we have $p_{h}(\hatx_{t_{n-1}})\ge p_{h}(\hatx_{t_n})$ and $p_{h}(r_{\bftheta}(\hatx_{t_n}))\ge p_{h}(\hatx_{t_n})$ in terms of the Gaussian KDE $p_{h}(\bfx)=(1/|\calI|)\sum_{i\in\calI}\calN(\bfx; \bfy_i, h^2\bfI)$ with any positive bandwidth $h$, assuming all samples in the sampling trajectory satisfy $d_1(\hatx_{t_{n}}) \le d_2(\hatx_{t_{n}})$.
\end{proposition}
\begin{proof}
	We first prove that given a random vector $\bfv$ falling within a sphere centered at the optimal denoising output $r_{\bftheta}^{\star}(\hatx_{t_n})$ with a radius of $\left\lVert r_{\bftheta}^{\star}(\hatx_{t_n})- \hatx_{t_n}\right\rVert_2$, \textit{i.e.}, $\left\lVert r_{\bftheta}^{\star}(\hatx_{t_n}) - \hatx_{t_n} \right\rVert_2 \ge \left\lVert\bfv  \right\rVert_2$, the sample likelihood is non-decreasing from $\hatx_{t_n}$ to  $r_{\bftheta}^{\star}(\hatx_{t_n})+\bfv$, \textit{i.e.}, $p_{h}(r_{\bftheta}^{\star}(\hatx_{t_n}) + \bfv)\ge p_{h}(\hatx_{t_n})$. Then, we provide two settings for $\bfv$ to finish the proof.
	
	The increase of sample likelihood from $\hatx_{t_n}$ to  $r_{\bftheta}^{\star}(\hatx_{t_n})+\bfv$ in terms of $p_{h}(\bfx)$ is
	\begin{equation}
		\begin{aligned}
			&\quad\, p_h(r_{\bftheta}^{\star}(\hatx_{t_n}) + \bfv) - p_h(\hatx_{t_n})=\frac{1}{|\calI|}\sum_{i}
			\left[\calN\left(r_{\bftheta}^{\star}(\hatx_{t_n}) + \bfv;\bfy_i, h^2\bfI\right)-\calN\left(\hatx_{t_n};\bfy_i, h^2\bfI\right)\right]\\
			&\overset{\text{(i)}}{\ge} \frac{1}{2h^2|\calI|}\sum_{i}\calN\left(\hatx_{t_n};\bfy_i, h^2\bfI\right)\left[\lVert \hatx_{t_n} - \bfy_i \rVert^2_2 
			-\lVert r_{\bftheta}^{\star}(\hatx_{t_n}) + \bfv - \bfy_i \rVert^2_2 \right]\\
			&=\frac{1}{2h^2|\calI|}\sum_{i}\calN\left(\hatx_{t_n};\bfy_i, h^2\bfI\right)\left[\lVert \hatx_{t_n}\rVert^2_2   - 2\hatx_{t_n}^{T}\bfy_i -\lVert r_{\bftheta}^{\star}(\hatx_{t_n}) + \bfv \rVert^2_2 + 2\left(r_{\bftheta}^{\star}(\hatx_{t_n}) + \bfv\right)^{T}\bfy_i  \right]\\
			&\overset{\text{(ii)}}{=}\frac{1}{2h^2|\calI|}\sum_{i}\calN\left(\hatx_{t_n};\bfy_i, h^2\bfI\right)\left[\lVert \hatx_{t_n}\rVert^2_2 - 2\hatx_{t_n}^{T}r_{\bftheta}^{\star}(\hatx_{t_n}) -\lVert r_{\bftheta}^{\star}(\hatx_{t_n}) + \bfv \rVert^2_2 + 2\left(r_{\bftheta}^{\star}(\hatx_{t_n}) + \bfv\right)^{T} r_{\bftheta}^{\star}(\hatx_{t_n}) \right]\\
			&=\frac{1}{2h^2|\calI|}\sum_{i}\calN\left(\hatx_{t_n};\bfy_i, h^2\bfI\right)\left[\lVert \hatx_{t_n}\rVert^2_2 - 2\hatx_{t_n}^T r_{\bftheta}^{\star}(\hatx_{t_n})+\lVert r_{\bftheta}^{\star}(\hatx_{t_n})\rVert^2_2 - \lVert \bfv \rVert^2_2 \right]\\
			&=\frac{1}{2h^2|\calI|}\sum_{i}\calN\left(\hatx_{t_n};\bfy_i, h^2\bfI\right)\left[\lVert r_{\bftheta}^{\star}(\hatx_{t_n}) - \hatx_{t_n}\rVert^2_2  - \lVert \bfv \rVert^2_2\right] \ge 0, 
		\end{aligned}
	\end{equation}
	where 
	(i) uses the definition of convex function $f(\bfx_2)\ge f(\bfx_1)+f^{\prime}(\bfx_1)(\bfx_2-\bfx_1)$ with $f(\bfx)=\exp \left(-\frac{1}{2}\lVert\bfx\rVert^2_2\right)$, $\bfx_1=\left(\hatx_{t_n}-\bfy_i \right)/h$ and $\bfx_2=\left(r_{\bftheta}^{\star}(\hatx_{t_n}) + \bfv-\bfy_i \right)/h$; 
	(ii) uses the relationship between two consecutive steps $\hatx_{t_n}$ and $r_{\bftheta}^{\star}(\hatx_{t_n})$ in mean shift, \ie,
	\begin{equation}
		r_{\bftheta}^{\star}(\hatx_{t_n})=\bfm(\hatx_{t_n})=\sum_{i} \frac{\exp \left(-\lVert \hatx_{t_n} - \bfy_i \rVert^2_2/2h^2\right)}{\sum_{j}\exp \left(-\lVert \hatx_{t_n} - \bfy_j \rVert^2_2/2h^2\right)} \bfy_i,
	\end{equation}
	which implies the following equation also holds
	\begin{equation}
		\sum_i \calN\left(\hatx_{t_n};\bfy_i, h^2\bfI\right) \bfx_i = \sum_i \calN\left(\hatx_{t_n};\bfy_i, h^2\bfI\right) r_{\bftheta}^{\star}(\hatx_{t_n}).
	\end{equation}
	Since $\left\lVert r_{\bftheta}^{\star}(\hatx_{t_n}) - \hatx_{t_n} \right\rVert_2 \ge \left\lVert\bfv  \right\rVert_2$, or equivalently, $\left\lVert r_{\bftheta}^{\star}(\hatx_{t_n}) - \hatx_{t_n} \right\rVert^2_2 \ge \left\lVert\bfv  \right\rVert^2_2$, we conclude that the sample likelihood monotonically increases from $\hatx_{t_n}$ to $r_{\bftheta}^{\star}(\hatx_{t_n})+\bfv$ unless $\hatx_{t_n}=r_{\bftheta}^{\star}(\hatx_{t_n})+\bfv$, in terms of the kernel density estimate $p_{h}(\bfx)=\frac{1}{|\calI|}\sum_{i}\calN(\bfx; \bfy_i, h^2\bfI)$ with any positive bandwidth $h$. 
	We next provide two settings for $\bfv$, which trivially satisfy the condition $\left\lVert r_{\bftheta}^{\star}(\hatx_{t_n}) - \hatx_{t_n} \right\rVert_2 \ge \left\lVert\bfv  \right\rVert_2$
	, and have the following corollaries:
	\begin{itemize}
		\item $p_{h}(r_{\bftheta}(\hatx_{t_n}))\ge p_{h}(\hatx_{t_n})$, when $\bfv=r_{\bftheta}(\hatx_{t_n})-r_{\bftheta}^{\star}(\hatx_{t_n})$. 
		\item $p_{h}(\hatx_{t_{n-1}})\ge p_{h}(\hatx_{t_n})$, when  $\bfv=r_{\bftheta}(\hatx_{t_n})-r_{\bftheta}^{\star}(\hatx_{t_n})+\frac{t_{n-1}}{t_n}\left(\hatx_{t_n}-r_{\bftheta}(\hatx_{t_n})\right)$. 
	\end{itemize}
\end{proof}

\subsection{Proof of Proposition~\ref{prop:analytic_gaussian}}
\label{subsec:proof_analytic_gaussian}

\begin{proposition}
    Suppose the data distribution is Gaussian $p_d(\bfx)=\calN(\bfmu, \bfSigma)$, where $\bfmu \in \bbR^d$, $\bfSigma\in \bbR^{d\times d}$ is positive semi-definite (PSD) with $\text{rank}\,(\bfSigma)=r\ll d$. Let $\bfSigma=\bfU\bfLambda \bfU^T$ denote the singular value decomposition (SVD), where $\bfU\in \bbR^{d \times r}$ contains eigenvectors $\bfu_i$ as columns, and $\bfLambda\in \bbR^{r \times r}$ is diagonal with eigenvalues $\lambda_i$, $i\in [1,r]$. In this setting, the PF-ODE solution $\bfx_t$ can be decomposed into the final sample $\bfx_0$, a scaled reverse displacement vector $\bfx_T-\bfx_0$, and a trajectory residual $\Delta_k(t)$:
    \begin{equation}
        \begin{aligned}
            \bfx_t &= \bfx_0 + \frac{\sigma_t}{\sigma_T} (\bfx_T-\bfx_0) + \Delta_k(t),\qquad \Delta_k(t)=\sum_{k=1}^{r}\varphi_k(t)\bfu_k^T(\bfx_T-\bfmu)\bfu_k,\\
            \varphi_k(t)&=\sqrt{\frac{\lambda_k+\sigma_t^2}{\lambda_k+\sigma_T^2}}-\sqrt{\frac{\lambda_k}{\lambda_k+\sigma_T^2}}-\frac{\sigma_t}{\sigma_T}\left(\,\,1-\sqrt{\frac{\lambda_k}{\lambda_k+\sigma_T^2}}\,\,\right).
        \end{aligned}
    \end{equation}
\end{proposition}

\begin{proof}
    The score function is $\nablaxt \log p_t(\bfx_t) = (\bfSigma+\sigma_t^2\bfI)^{-1}(\bfmu -\bfx_t)$, and the corresponding linear PF-ODE is $\rmd \bfx_t/\rmd \sigma_t =\sigma_t(\bfSigma+\sigma_t^2\bfI)^{-1}(\bfx_t-\bfmu)$. 
    The solution of $\bfx_t$
    \begin{equation}
        \begin{aligned}
            \bfx_t
            &=\bfmu+\exp\left(\int_T^t \sigma_s\,(\bfSigma+\sigma_s^2\bfI)^{-1}  \rmd s\right)(\bfx_T-\bfmu)\\
            &\overset{\text{(i)}}{=}\bfmu+\exp\left(\int_T^t \frac{1}{\sigma_s}\left[\mathbf{I}-\mathbf{U} \operatorname{diag}\left[\frac{\lambda_{k}}{\lambda_{k}+\sigma_s^{2}}\right] \mathbf{U}^{T}\right] \rmd s\right)(\bfx_T-\bfmu)\\
            &\overset{\text{(ii)}}{=}\bfmu+\exp\left(\int_T^t \bfQ \begin{bmatrix}
            \operatorname{diag}\!\Big(\frac{\sigma_s}{\lambda_k+\sigma_s^2}\Big) & 0 \\
            0 & \tfrac{1}{\sigma_s} \bfI
            \end{bmatrix}
            \bfQ^T\rmd s\right)(\bfx_T-\bfmu)\\
            &\overset{\text{(iii)}}{=}\bfmu+\bfQ \exp\left(\int_T^t \begin{bmatrix}
            \operatorname{diag}\!\Big(\frac{\sigma_s}{\lambda_k+\sigma_s^2}\Big) & 0 \\
            0 & \tfrac{1}{\sigma_s} \bfI
            \end{bmatrix}
        \rmd s\right)\bfQ^T(\bfx_T-\bfmu)\\
        &= \bfmu + \frac{\sigma_t}{\sigma_T}(\bfI-\bfU\bfU^T)(\bfx_T-\bfmu)+\sum_{k=1}^{r} \sqrt{\frac{\lambda_k+\sigma_t^2}{\lambda_k+\sigma_T^2}}\bfu_k\bfu_k^T(\bfx_T-\bfmu)\\
        &= \bfmu + \frac{\sigma_t}{\sigma_T}(\bfx_T-\bfmu) +  (1- \frac{\sigma_t}{\sigma_T})\left( \sum_{k=1}^{r} \sqrt{\frac{\lambda_k}{\lambda_k+\sigma_T^2}}\bfu_k\bfu_k^T(\bfx_T-\bfmu) \right) +\Delta_k(t)\\
        &=\bfx_0+ \frac{\sigma_t}{\sigma_T}(\bfx_T-\bfx_0) + \Delta_k(t),
        \end{aligned}
    \end{equation}
where (i) applies the Woodbury identity, (ii) denotes $\bfQ=[\bfU, \bfU_\perp]\in \bbR^{d\times d}$, $\bfQ\bfQ^T=\bfI$, (iii) exchanges the order of operators given commuting matrices (simultaneous diagonalization). The trajectory deviation $\Delta_k(t)$:
\begin{equation}
    \begin{aligned}
        \Delta_k(t)
        &=\bfU\operatorname{diag}\left[-\frac{\sigma_t}{\sigma_T}+\sqrt{\frac{\lambda_k+\sigma_t^2}{\lambda_k+\sigma_T^2}}+(\frac{\sigma_t}{\sigma_T}-1)\sqrt{\frac{\lambda_k}{\lambda_k+\sigma_T^2}}\right]\bfU^T(\bfx_T-\bfmu)\\
        &=\sum_{k=1}^{r}\varphi_k(t)\bfu_k^T(\bfx_T-\bfmu)\bfu_k, \quad \varphi_k(t)=-\frac{\sigma_t}{\sigma_T}+\sqrt{\frac{\lambda_k+\sigma_t^2}{\lambda_k+\sigma_T^2}}+(\frac{\sigma_t}{\sigma_T}-1)\sqrt{\frac{\lambda_k}{\lambda_k+\sigma_T^2}}.
    \end{aligned}
\end{equation}
\end{proof}

Since $\bfu_i^T\bfu_j=0$ for $i\neq j$, we have the squared norm of the trajectory residual as follows
\begin{equation}
    h(t)\coloneqq \|\Delta_k(t)\|_2^2=(\sum_{i=1}^{r}\varphi_i(t)\bfu_i\bfu_i^T(\bfx_T-\bfmu))(\sum_{j=1}^{r}\varphi_j(t)\bfu_j\bfu_j^T(\bfx_T-\bfmu))=\sum_{k=1}^{r}    \varphi_k(t)^2(\bfu_k^T(\bfx_T-\bfmu))^2.
\end{equation}
Since $\bfx_T \sim \calN(\bfmu, \bfSigma+\sigma_T^2\bfI)$, then 
    \begin{equation}
        \bfU^T(\bfx_T-\bfmu) \sim \calN(\textbf{0}, \bfU^T(\bfSigma+\sigma_T^2\bfI)\bfU)=\calN(\textbf{0}, \bfU^T(\bfSigma+\sigma_T^2\bfI_n)\bfU)=\calN(\textbf{0}, \bfLambda+\sigma_T^2\bfI_d).
    \end{equation}
\begin{equation}
    h(t) = \sum_{k=1}^{r}\varphi_k(t)^2 v_k^2, \quad v_k \sim \calN(0, \lambda_k+\sigma_T^2).
\end{equation}
We denote $s_k(t)= \sqrt{\frac{\lambda_k+\sigma_t^2}{\lambda_k+\sigma_T^2}}$, then $s_k(T)=1$, and 
\begin{equation}
    s_k^{\prime}(t)=\frac{\sigma_t}{(\lambda_k+\sigma_T^2)\,s_k(t)},\quad s_k^{\prime\prime}(t)=\frac{\lambda_k}{(\lambda_k+\sigma_t^2)^{3/2}\sqrt{\lambda_k+\sigma_T^2}}=\frac{\lambda_k}{(\lambda_k+\sigma_T^2)s_k(t)^3}>0 \qquad(\lambda_k>0),
\end{equation}
\begin{equation}
    \begin{aligned}
        \varphi_k(t)&=s_k(t)-\frac{\sigma_t}{\sigma_T}+(\frac{\sigma_t}{\sigma_T}-1)s_k(0)=s_k(t)-s_k(0)-\frac{\sigma_t}{\sigma_T}(1-s_k(0)),\\
        \varphi_k^{\prime}(t)&=\frac{\sigma_t}{(\lambda_k+\sigma_T^2)s_k(t)}-\frac{1-s_k(0)}{\sigma_T}=s^{\prime}_k(t)-\frac{1-s_k(0)}{\sigma_T},\\
        \varphi_k^{\prime\prime}(t)&=s^{\prime\prime}_k(t)>0 \quad \text{for each} \quad k.
    \end{aligned}
\end{equation}
Therefore, $\varphi$ is a strict convex function for $t\in[0,T]$ and must have \textit{one unique minimum}. By setting $\varphi_k^{\prime}(t_{\text{min}})=0$, we have $t_{\text{min}} = \sqrt{\frac{\sqrt{\lambda_k(\lambda_k+\sigma_T^2)-\lambda_k}}{2}}$, and
\begin{equation}
    \begin{aligned}
        s_k(t_{\min})= \frac{\sqrt{(\sqrt{\lambda_k})(\sqrt{\lambda_k}+\sqrt{\lambda_k+\sigma_T^2})}}{\sqrt{2} \sqrt{\lambda_k+\sigma_T^2}}, \quad
        \varphi_k(t_{\min})=\sqrt{\frac{\lambda_k}{\lambda_k+\sigma_T^2}}\;\Bigg(\sqrt{\frac{2\sqrt{\lambda_k}}{\sqrt{\lambda_k}+\sqrt{\lambda_k+\sigma_T^2}}}\;-\;1\Bigg).
    \end{aligned}
\end{equation}
The visualization of $\varphi_k(t)^2$ and $h_k(t)$ with respect to $k$-th eigenvalue ($k\in [1, 1000]$) and $t$ ($t\in [0,80]$) are provided in Figures~\ref{fig:phi_evolve} and~\ref{fig:h_evolve}. Since $\Delta_k(t)$ is approximately orthogonal to the displacement vector $\bfx_0-\bfx_T$, the differences between trajectory deviations and trajectory residuals are minor, as shown in Figure~\ref{fig:low_rank_gaussian_comparison}.

\begin{figure}[t!]
    \centering
    \begin{subfigure}[t]{0.32\textwidth}
        \includegraphics[width=\columnwidth]{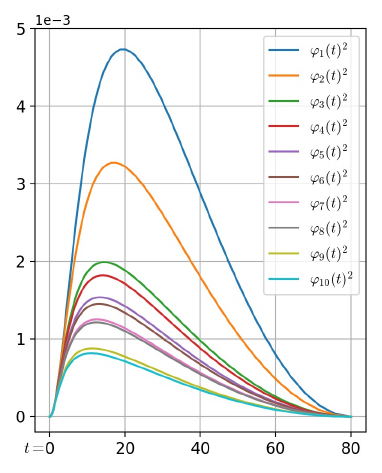}
    \end{subfigure}
    \begin{subfigure}[t]{0.32\textwidth}
        \includegraphics[width=\columnwidth]{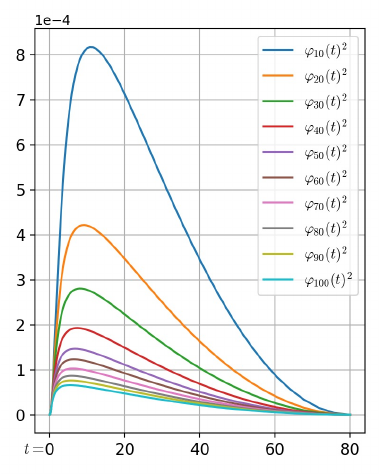}
    \end{subfigure}
    \begin{subfigure}[t]{0.32\textwidth}
        \includegraphics[width=\textwidth]{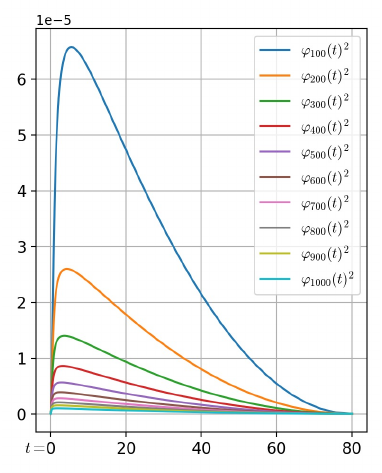}
    \end{subfigure}
    \caption{Variation of $\varphi_k(t)^2$ with respect to $k$ and $t$.}
    \label{fig:phi_evolve}
\end{figure}
\begin{figure}[t!]
    \centering
    \begin{subfigure}[t]{0.32\textwidth}
        \includegraphics[width=\columnwidth]{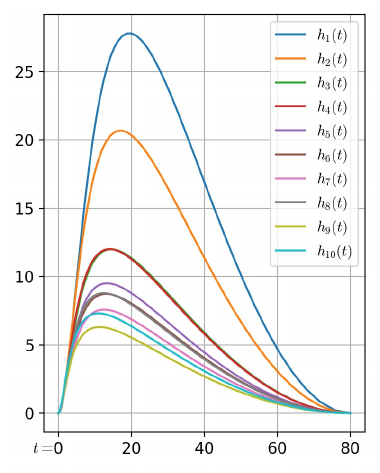}
    \end{subfigure}
    \begin{subfigure}[t]{0.32\textwidth}
        \includegraphics[width=\columnwidth]{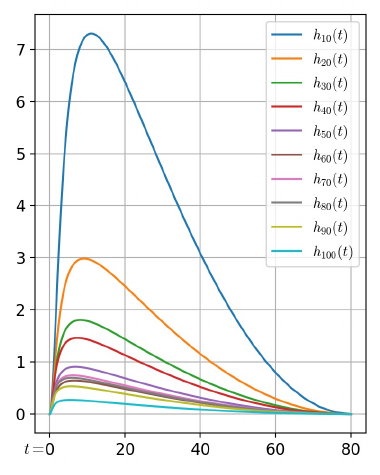}
    \end{subfigure}
    \begin{subfigure}[t]{0.32\textwidth}
        \includegraphics[width=\textwidth]{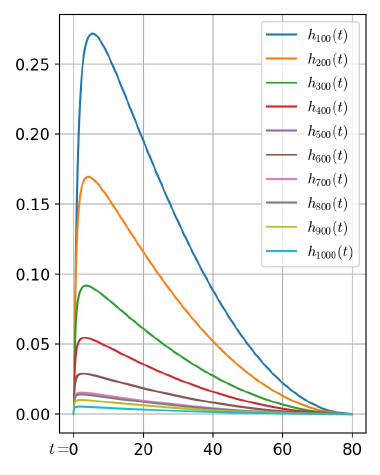}
    \end{subfigure}
    \caption{Variation of $h_k(t)=\varphi_k(t)^2(\mathbf{u}_k^T(\mathbf{x}_T-\mathbf{\mu}))^2$ with respect to $k$ and $t$.}
    \label{fig:h_evolve}
\end{figure}
\begin{figure}[t!]
    \centering
    \includegraphics[width=0.5\textwidth]{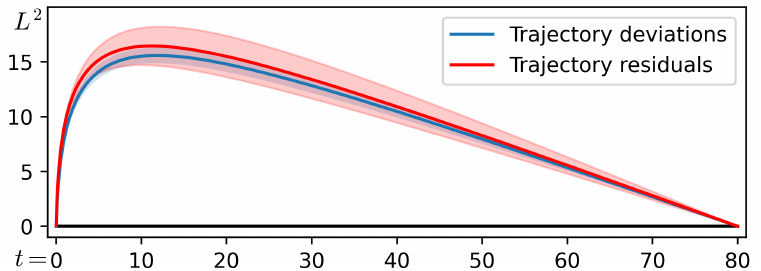}
    \caption{Comparison between trajectory deviations and trajectory residuals using low-rank Gaussian approximation.}
    \label{fig:low_rank_gaussian_comparison}
\end{figure}

\section{Additional Results}

\subsection{Visualization of Sampling Trajectories}
\label{subsec:appendix_visual}

Figures~\ref{fig:deviation_ffhq} and \ref{fig:deviation_bedroom} provide more experiments about the 1-D trajectory projection on FFHQ and LSUN Bedroom. Figure~\ref{fig:recon_ffhq} provides more results about Multi-D projections on FFHQ. Figures~\ref{fig:recon_gaussian} and~\ref{fig:recon_mog} provide more results on CIFAR-10 using a Gaussian or mixture of Gaussians model.
Figure~\ref{fig:euler_final} visualizes more generated samples on three datasets. 

\begin{figure}[t!]
    \begin{subfigure}[t]{0.48\textwidth}
        \includegraphics[width=\columnwidth]{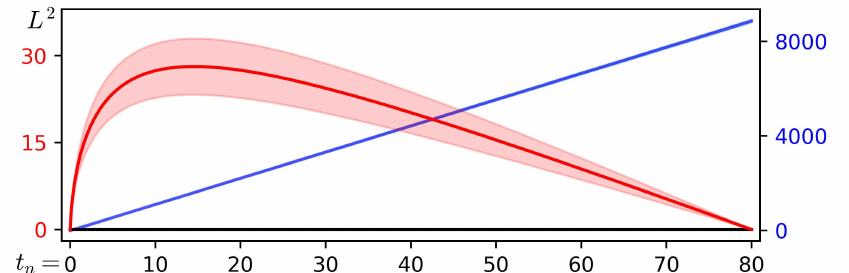}
        \caption{FFHQ (64$\times$64).}
        \label{fig:deviation_ffhq}
    \end{subfigure}
    \begin{subfigure}[t]{0.48\textwidth}
        \includegraphics[width=\columnwidth]{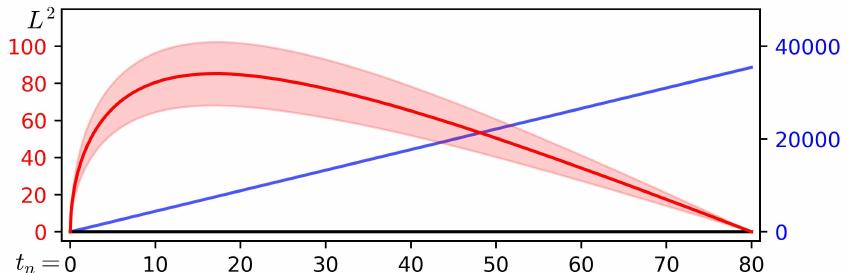}
        \caption{LSUN Bedroom (256$\times$256).}
        \label{fig:deviation_bedroom}
    \end{subfigure}
    \caption{Trajectory deviation (1-D projection) on unconditional generation.}
    \label{fig:app_traj_deviation}
\end{figure}
\vspace{-0.5em}

\begin{figure}[t!]
    \centering
    \begin{subfigure}[t]{0.51\textwidth}
        \includegraphics[width=\columnwidth]{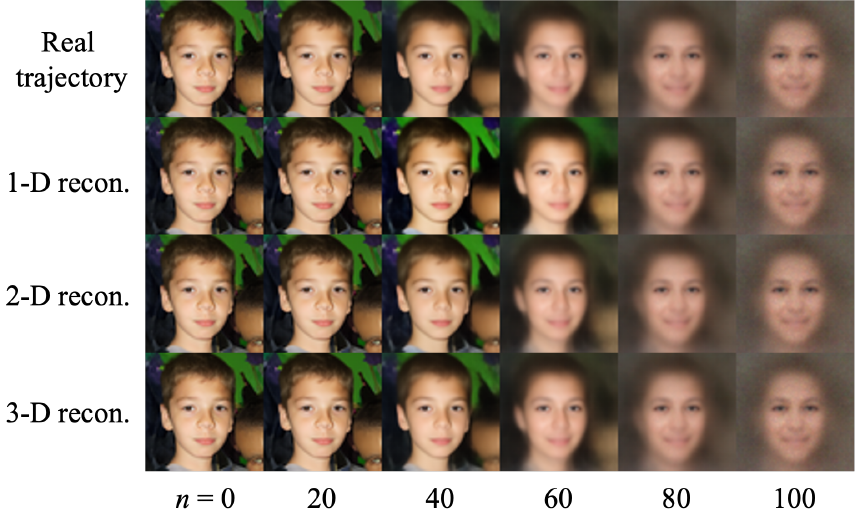}
        \caption{Visual Comparison.}
        \label{fig:recon_visual_ffhq}
    \end{subfigure}
    \begin{subfigure}[t]{0.255\textwidth}
        \includegraphics[width=\columnwidth]{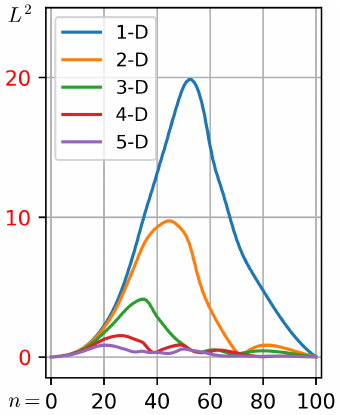}
        \caption{$L^2$ error.}
        \label{fig:recon_l2_ffhq}
    \end{subfigure}
    \begin{subfigure}[t]{0.20\textwidth}
        \includegraphics[width=\textwidth]{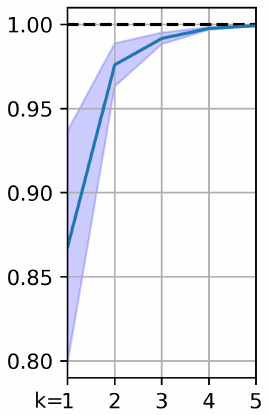}
        \caption{PCA ratio.}
        \label{fig:recon_pca_ffhq}
    \end{subfigure}
    \begin{subfigure}[t]{0.7\textwidth}
        \includegraphics[width=\columnwidth]{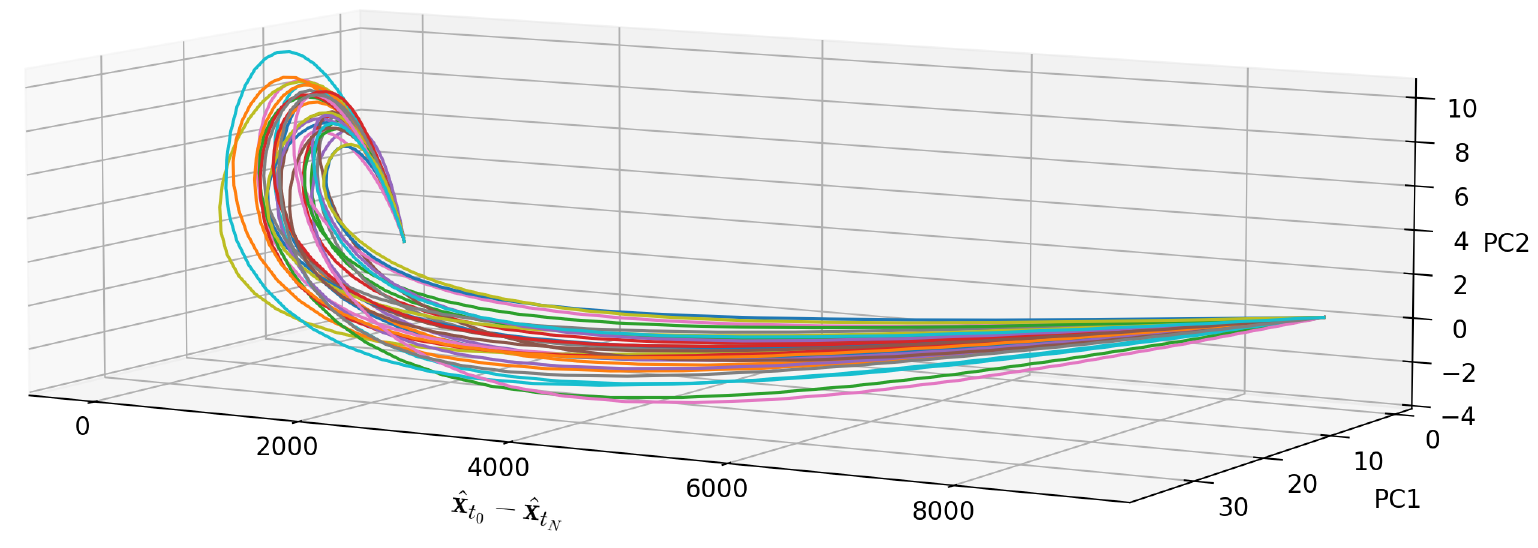}
        \caption{Unconditional generation, FFHQ (64$\times$ 64)}
        \label{fig:traj_3d_ffhq}
    \end{subfigure}
    \begin{subfigure}[t]{0.29\textwidth}
        \includegraphics[width=\columnwidth]{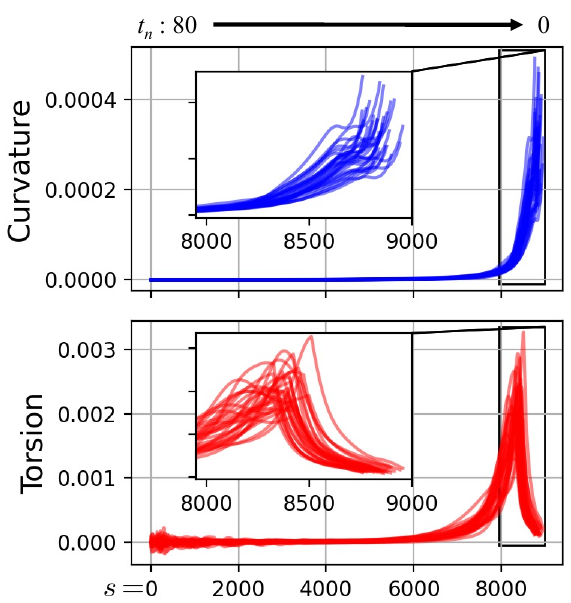}
        \caption{Curvature/Torsion.}
        \label{fig:traj_stats_ffhq}
    \end{subfigure}
    \caption{Trajectory reconstruction, visualization and statistics on FFHQ. Figure (a) is generated by EDM~\citep{karras2022edm}.}
    \label{fig:recon_ffhq}
\end{figure}

\begin{figure}[t!]
    \centering
    \begin{subfigure}[t]{0.51\textwidth}
        \includegraphics[width=\columnwidth]{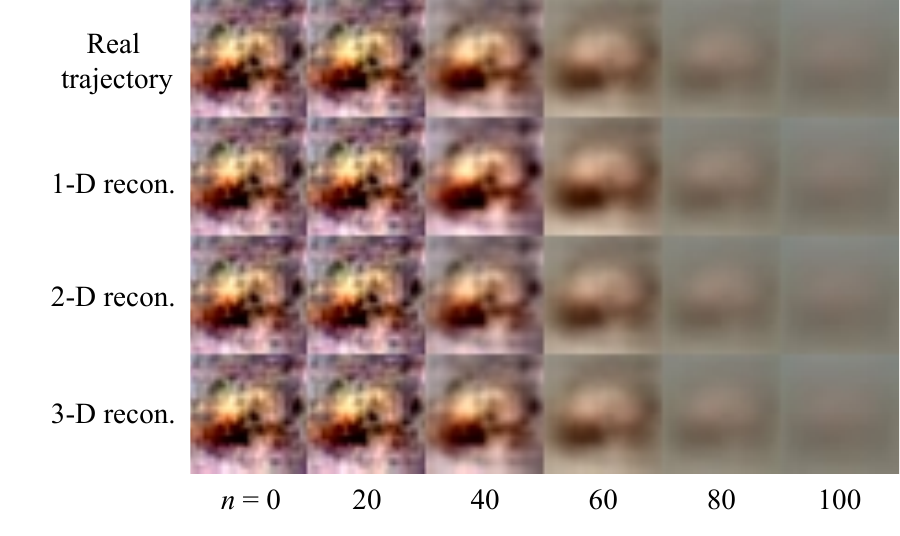}
        \caption{Visual Comparison.}
        \label{fig:recon_visual_gaussian}
    \end{subfigure}
    \begin{subfigure}[t]{0.265\textwidth}
        \includegraphics[width=\columnwidth]{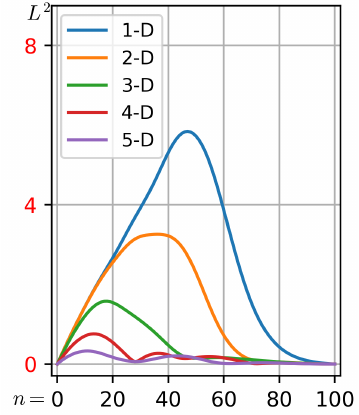}
        \caption{$L^2$ error.}
        \label{fig:recon_l2_gaussian}
    \end{subfigure}
    \begin{subfigure}[t]{0.20\textwidth}
        \includegraphics[width=\textwidth]{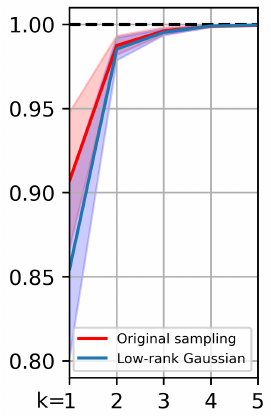}
        \caption{PCA ratio.}
        \label{fig:recon_pca_gaussian}
    \end{subfigure}
    \caption{Trajectory reconstruction, visualization and statistics on CIFAR-10 using low-rank Gaussian.}
    \label{fig:recon_gaussian}
\end{figure}

\begin{figure}[t!]
    \centering
    \begin{subfigure}[t]{0.51\textwidth}
        \includegraphics[width=\columnwidth]{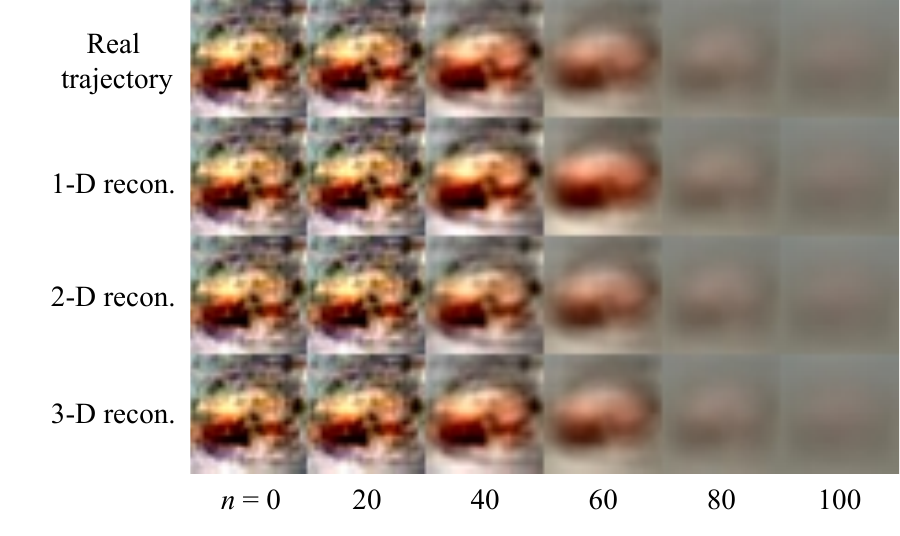}
        \caption{Visual Comparison.}
        \label{fig:recon_visual_mog}
    \end{subfigure}
    \begin{subfigure}[t]{0.265\textwidth}
        \includegraphics[width=\columnwidth]{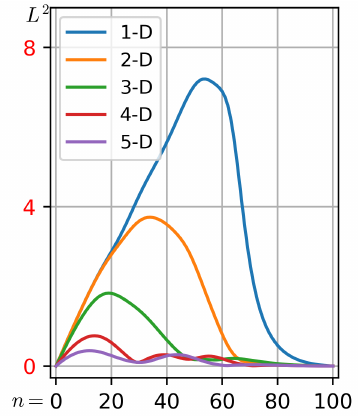}
        \caption{$L^2$ error.}
        \label{fig:recon_l2_mog}
    \end{subfigure}
    \begin{subfigure}[t]{0.20\textwidth}
        \includegraphics[width=\textwidth]{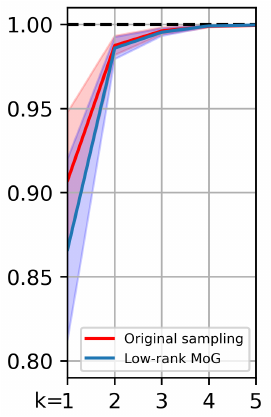}
        \caption{PCA ratio.}
        \label{fig:recon_pca_mog}
    \end{subfigure}
    \caption{Trajectory reconstruction, visualization and statistics on CIFAR-10 using low-rank mixture of Gaussians.}
    \label{fig:recon_mog}
\end{figure}

\begin{figure}[t!]
    \centering
    \begin{subfigure}[t]{0.45\columnwidth}
        \centering
        \includegraphics[width=\columnwidth]{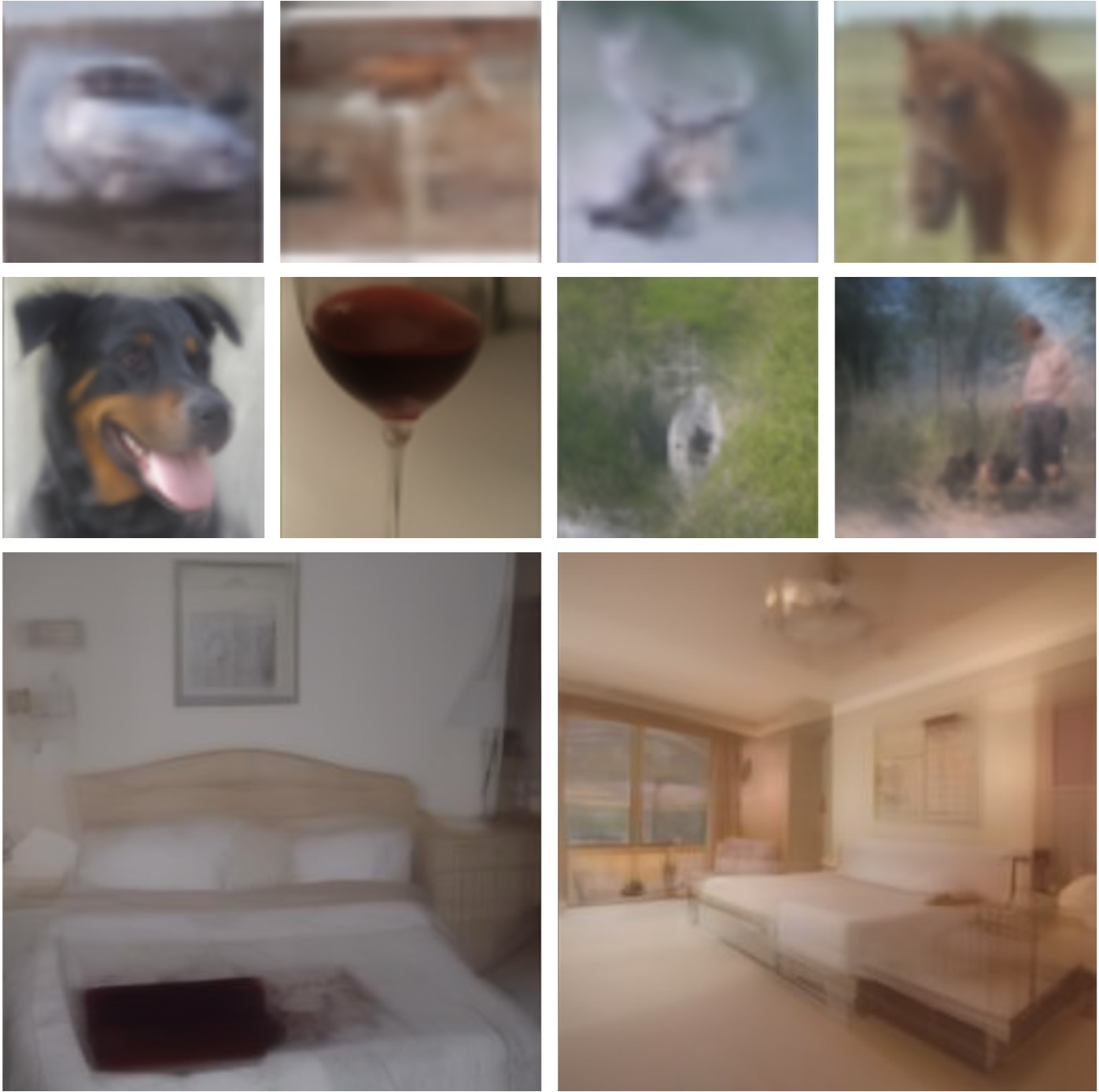}
        \caption{DDIM, NFE = 5.}
    \end{subfigure}
    \quad
    \begin{subfigure}[t]{0.45\columnwidth}
        \centering
        \includegraphics[width=\columnwidth]{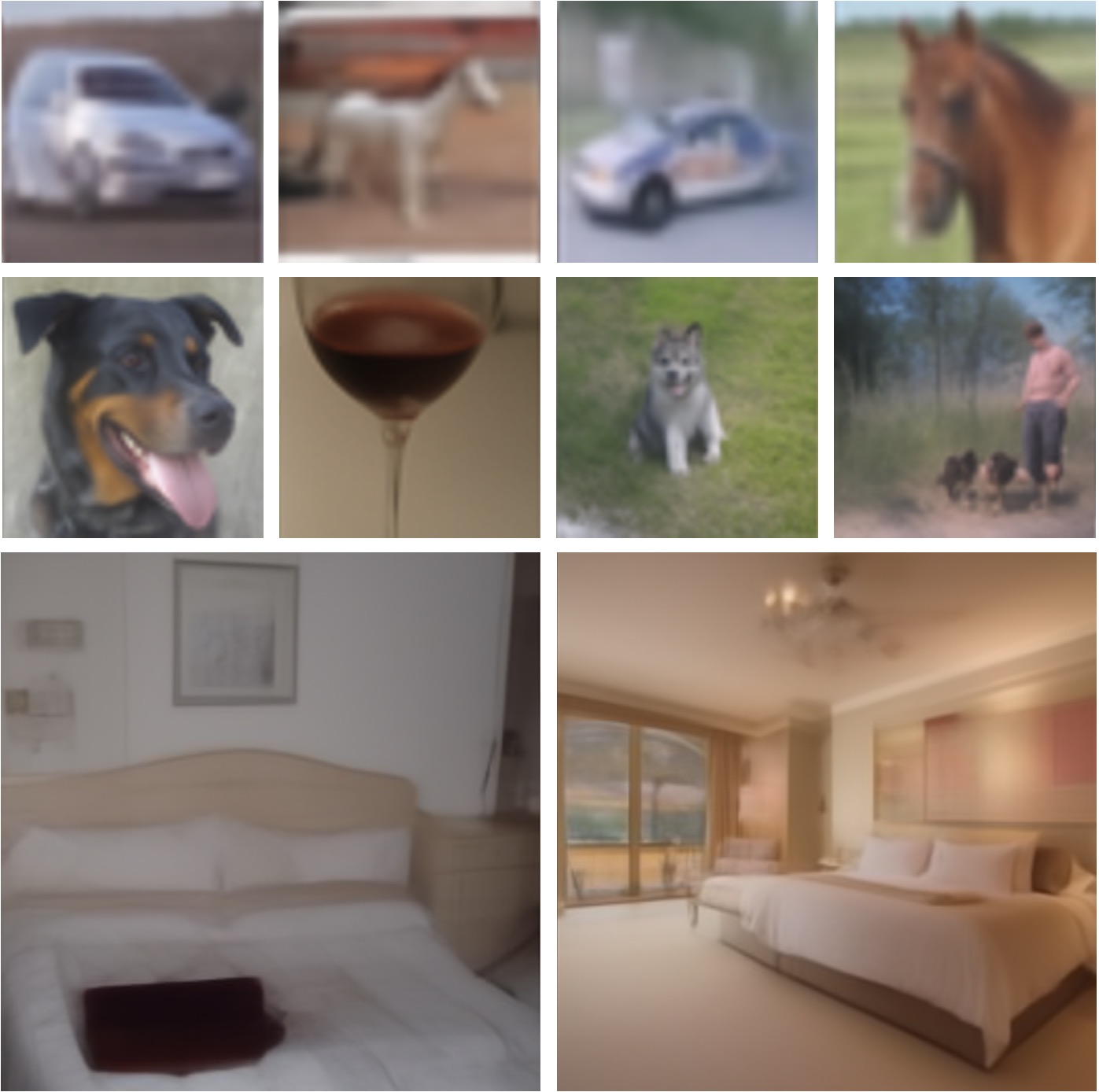}
        \caption{DDIM + \ourName, NFE = 5.}
    \end{subfigure}
    \caption{The visual comparison of samples generated by DDIM and DDIM + \ourName (1st row: CIFAR-10, 2nd row: ImageNet $64\times 64$, 3rd row: LSUN Bedroom). Figures are generated by EDM~\citep{karras2022edm}. 
    }
    \label{fig:euler_final}
\end{figure}

\subsection{Diagnosis of Score Deviation}
\label{subsec:score_deviation}

In this section, we simulate four new trajectories based on the optimal denoising output $r_{\bftheta}^{\star}(\cdot)$ to monitor the score deviation from the optimum. We denote the \textit{optimal sampling trajectory} as $\{\hatx_{t_n}^{\star}\}_{n=0}^{N}$, where we generate samples as the standard sampling trajectory $\{\hatx_{t_n}\}_{n=0}^{N}$ with the same time schedule $\Gamma=\{t_0=\epsilon, \cdots, t_N=T\}$, but adopt optimal denoising output $r^{\star}_{\bftheta}(\cdot)$ rather than denoising output $r_{\bftheta}(\cdot)$ for score estimation. The other three trajectories are simulated by tracking the (optimal) denoising output of each sample in $\{\hatx_{t}^{\star}\}$ or $\{\hatx_{t}\}$, and designated as $\{r_{\bftheta}(\hatx_{t}^{\star})\}$, $\{r^{\star}_{\bftheta}(\hatx_{t}^{\star})\}$, $\{r_{\bftheta}^{\star}(\hatx_{t})\}$. 
According to \eqref{eq:convex} and $t_0=0$, we have $\hatx_{t_0}^{\star}=r_{\bftheta}^{\star}(\hatx_{t_1}^{\star})$, and similarly, $\hatx_{t_0}=r_{\bftheta}(\hatx_{t_1})$. As $t\rightarrow 0$, $r^{\star}_{\bftheta}(\hatx_{t}^{\star})$ and $r_{\bftheta}^{\star}(\hatx_{t})$ serve as the approximate nearest neighbors of $\hatx_{t}^{\star}$ and $\hatx_{t}$ to the real data, respectively.

We calculate the deviation of denoising output to quantify the score deviation across all time steps using the $L^2$ distance, though they should differ by a factor $t^2$, and have the following observation: 
\textit{The learned score is well-matched to the optimal score in the large-noise region, otherwise, they may diverge or almost coincide depending on different regions. }
In fact, our learned score has to moderately diverge from the optimum to guarantee the generative ability. Otherwise, the ODE-based sampling reduces to an approximate (single-step) annealed mean shift for global mode-seeking, and simply replays the dataset. 
As shown in Figure~\ref{fig:diagnosis}, the nearest sample of $\hatx_{t_0}^{\star}$ to the real data is almost the same as itself, which indicates the optimal sampling trajectory has a very limited ability to synthesize novel samples. Empirically, score deviation in a small region is sufficient to bring forth a decent generative ability. 

From the comparison of $\{r_{\bftheta}(\hatx_t^{\star})\}$, $\{r_{\bftheta}^{\star}(\hatx_t^{\star})\}$ sequences in Figures~\ref{fig:diagnosis} and~\ref{fig:score_deviation}, we can clearly see that \textit{along the optimal sampling trajectory}, the deviation between the learned denoising output $r_{\bftheta}(\cdot)$ and its optimal counterpart $r_{\bftheta}^{\star}(\cdot)$ behaves differently in three successive regions:
The deviation starts off as almost negligible (about $10<t\le 80$), gradually increases (about $3<t\le10$), and then drops down to a low level once again (about $0\le t\le3$).
This phenomenon was also validated by a recent work \citep{xu2023stable} with a different perspective. 
We further observe that \textit{along the sampling trajectory}, this phenomenon disappears and the score deviation keeps increasing (see $\{r_{\bftheta}(\hatx_t)\}$, $\{r_{\bftheta}^{\star}(\hatx_t)\}$ sequences in Figures~\ref{fig:diagnosis} and~\ref{fig:score_deviation}). Additionally, samples in the latter half of $\{r_{\theta}^{\star}(\hatx_t)\}$ appear almost the same as the nearest sample of $\hatx_{t_0}$ to the real data, as shown in Figure~\ref{fig:diagnosis}. This indicates that our score-based model strives to explore novel regions, and synthetic samples in the sampling trajectory are quickly attracted to a real-data mode but do not fall into it.

\begin{figure*}[t]
	\centering
	\begin{subfigure}[t]{\textwidth}
		\centering
		\includegraphics[width=0.95\textwidth]{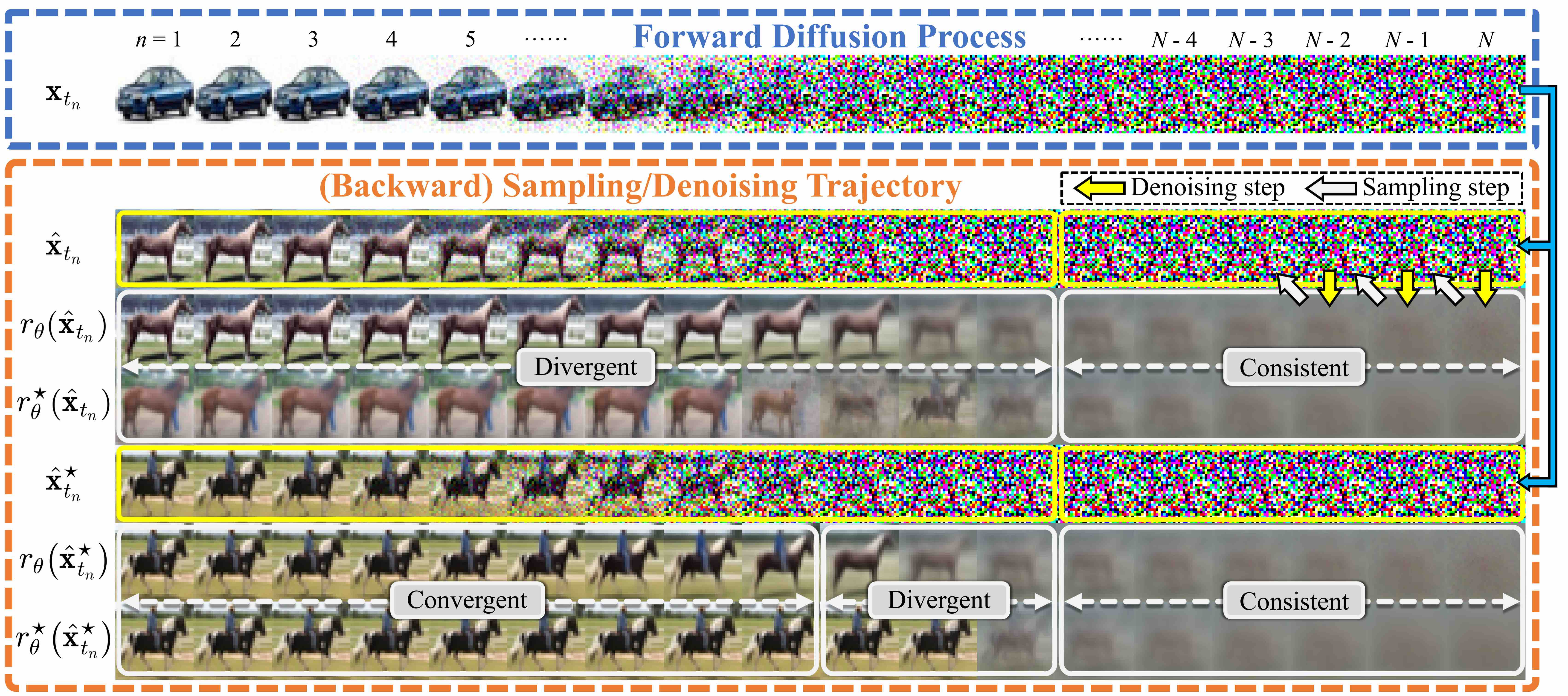}
	\end{subfigure}
	\hfil
	\begin{subfigure}[t]{\textwidth}
		\centering
		\includegraphics[width=0.96\textwidth]{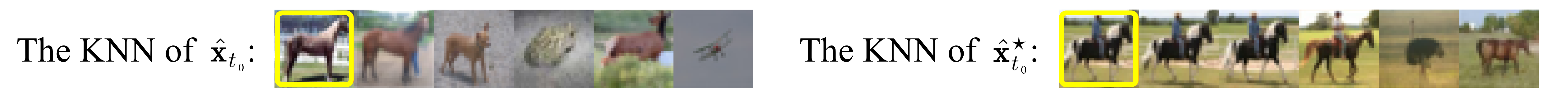}
	\end{subfigure}
	\caption{\textit{Top}: We visualize a forward diffusion process of a randomly-selected image to obtain its encoding $\hatx_{t_N}$ (first row) and simulate multiple trajectories starting from this encoding (other rows). 
		\textit{Bottom}: The k-nearest neighbors (k=5) of $\hatx_{t_0}$ and $\hatx_{t_0}^{\star}$ to real samples in the dataset. Figures are generated by EDM~\citep{karras2022edm}.
	}
	\label{fig:diagnosis}
\end{figure*}
\begin{figure*}[t!]
	\centering
	\includegraphics[width=0.96\textwidth]{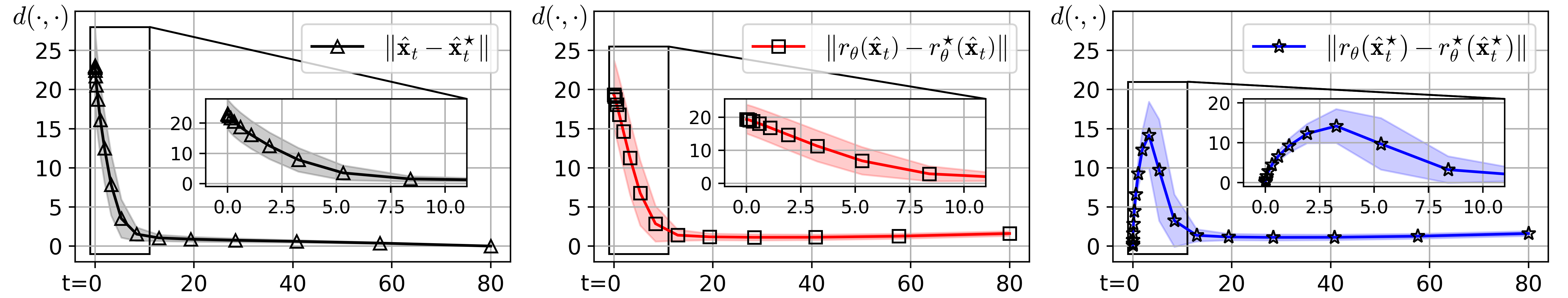}
	\caption{The deviation (Euclidean distance) of outputs from their corresponding optima.} 
	\label{fig:score_deviation}
\end{figure*}

\begin{table}[t!]
    \caption{Sample quality in terms of FID~\citep{heusel2017gans} on four datasets (resolutions ranging from $32\times32$ to $256\times256$). $\dagger$: After obtaining the DP schedule, we could further optimize the first time step with AFS, using the same ``warmup'' samples. 
    The default setting in our main submission does not use AFS and keeps the coefficient in dynamic programming as 1.1 for LSUN Bedroom and 1.15 otherwise. Although the performance can be further improved by carefully tuning the coefficient and using AFS as shown below.}
    \label{tab:fid-full}
    \vskip 0.1in
    \centering
    \fontsize{8}{10}\selectfont
   	\resizebox{0.95\textwidth}{!}{
    \begin{tabular}{lcccccccccc}
        \toprule
        \multirow{2}{*}{METHOD} & \multirow{2}{*}{Coeff} & \multirow{2}{*}{AFS$\dagger$} & \multicolumn{8}{c}{NFE} \\
        \cmidrule{4-11}
        & & & 3 & 4 & 5 & 6 & 7 & 8 & 9 & 10 \\
        \midrule
        \multicolumn{11}{l}{ \textbf{CIFAR-10 32$\times$32}~\citep{krizhevsky2009learning}} \\
        \midrule
        DDIM~\citep{song2021ddim}          & -    & \usym{2613}      & 93.36 & 66.76 & 49.66 & 35.62 & 27.93 & 22.32 & 18.43 & 15.69 \\
        \rowcolor[gray]{0.9} DDIM + GITS  & 1.10 & \usym{2613}      & 88.68 & 46.88 & 32.50 & 22.04 & 16.76 & 13.93 & 11.57 & 10.09 \\
        \rowcolor[gray]{0.9} DDIM + GITS  (default) & 1.15 & \usym{2613}      & 79.67 & 43.07 & 28.05 & 21.04 & 16.35 & 13.30 & 11.62 & 10.37 \\
        \rowcolor[gray]{0.9} DDIM + GITS  & 1.20 & \usym{2613}      & 77.22 & 43.16 & 29.06 & 22.69 & 18.91 & 14.22 & 12.03 & 11.38 \\
        iPNDM~\citep{zhang2023deis}        & -    & \usym{2613}      & 47.98 & 24.82 & 13.59 & 7.05 & 5.08 & 3.69 & 3.17 & 2.77 \\
        \rowcolor[gray]{0.9} iPNDM + GITS & 1.10 & \usym{2613}      & 51.31 & 17.19 & 12.90 & 5.98 & 6.62 & 4.36 & 3.59 & 3.14 \\
        \rowcolor[gray]{0.9} iPNDM + GITS (default) & 1.15 & \usym{2613}      & 43.89 & 15.10 & 8.38  & 4.88 & 5.11 & 3.24 & 2.70 & 2.49 \\
        \rowcolor[gray]{0.9} iPNDM + GITS & 1.20 & \usym{2613}      & 42.06 & 15.85 & 9.33  & 7.13 & 5.95 & 3.28 & 2.81 & 2.71 \\
        \rowcolor[gray]{0.9} iPNDM + GITS & 1.10 & \usym{1F5F8}     & 34.22 & 11.99 & 12.44 & 6.08 & 6.20 & 3.53 & 3.48 & 2.91  \\
        \rowcolor[gray]{0.9} iPNDM + GITS & 1.15 & \usym{1F5F8}     & 29.63 & 11.23 & 8.08  & 4.86 & 4.46 & 2.92 & 2.46 & \textbf{2.27}  \\
        \rowcolor[gray]{0.9} iPNDM + GITS & 1.20 & \usym{1F5F8}     & \textbf{25.98} & \textbf{10.11} & \textbf{6.77}  & \textbf{4.29} & \textbf{3.43} & \textbf{2.70} & \textbf{2.42} & 2.28  \\
        \midrule
        \multicolumn{11}{l}{ \textbf{FFHQ 64$\times$64}~\citep{karras2019style}} \\
        \midrule
        DDIM~\citep{song2021ddim}          & -    & \usym{2613}      & 78.21 & 57.48 & 43.93 & 35.22 & 28.86 & 24.39 & 21.01 & 18.37  \\
        \rowcolor[gray]{0.9} DDIM + GITS  & 1.10 & \usym{2613}      & 62.70 & 43.12 & 31.01 & 24.62 & 20.35 & 17.19 & 14.71 & 13.01  \\
        \rowcolor[gray]{0.9} DDIM + GITS  (default) & 1.15 & \usym{2613}      & 60.84 & 40.81 & 29.80 & 23.67 & 19.41 & 16.60 & 14.46 & 13.06  \\
        \rowcolor[gray]{0.9} DDIM + GITS  & 1.20 & \usym{2613}      & 59.64 & 40.56 & 30.29 & 23.88 & 20.07 & 17.36 & 15.40 & 14.05  \\
        iPNDM~\citep{zhang2023deis}        & -    & \usym{2613}      & 45.98 & 28.29 & 17.17 & 10.03 & 7.79 & 5.52 & 4.58 & 3.98  \\
        \rowcolor[gray]{0.9} iPNDM + GITS & 1.10 & \usym{2613}      & 34.82 & 18.75 & 13.07 & 7.79  & 8.30 & 4.76 & 5.36 & 3.47  \\
        \rowcolor[gray]{0.9} iPNDM + GITS (default) & 1.15 & \usym{2613}      & 33.09 & 17.04 & 11.22 & 7.00  & 6.72 & 4.52 & 4.33 & 3.62  \\
        \rowcolor[gray]{0.9} iPNDM + GITS & 1.20 & \usym{2613}      & 31.70 & 16.87 & 10.83 & 7.10  & 6.37 & 5.78 & 4.81 & 4.39  \\
        \rowcolor[gray]{0.9} iPNDM + GITS & 1.10 & \usym{1F5F8}     & 33.19 & 19.88 & 12.90 & 8.29 & 7.50 & 4.26 & 4.95 & \textbf{3.13}  \\
        \rowcolor[gray]{0.9} iPNDM + GITS & 1.15 & \usym{1F5F8}     & 30.39 & 15.78 & 10.15 & 6.86 & 5.97 & \textbf{4.09} & \textbf{3.76} & 3.24  \\
        \rowcolor[gray]{0.9} iPNDM + GITS & 1.20 & \usym{1F5F8}     & \textbf{26.41} & \textbf{13.59} & \textbf{8.85}  & \textbf{6.39} & \textbf{5.36} & 4.91 & 3.89 & 3.51  \\
        \midrule
        \multicolumn{11}{l}{ \textbf{ImageNet 64$\times$64}~\citep{russakovsky2015ImageNet}} \\
        \midrule
        DDIM~\citep{song2021ddim}          & -    & \usym{2613}      & 82.96 & 58.43 & 43.81 & 34.03 & 27.46 & 22.59 & 19.27 & 16.72  \\
        \rowcolor[gray]{0.9} DDIM + GITS  & 1.10 & \usym{2613}      & 60.11 & 36.23 & 27.31 & 20.82 & 16.41 & 14.16 & 11.95 & 10.84  \\
        \rowcolor[gray]{0.9} DDIM + GITS (default)  & 1.15 & \usym{2613}      & 57.06 & 35.07 & 24.92 & 19.54 & 16.01 & 13.79 & 12.17 & 10.83  \\
        \rowcolor[gray]{0.9} DDIM + GITS  & 1.20 & \usym{2613}      & 54.24 & 34.27 & 24.67 & 19.46 & 16.66 & 14.15 & 13.41 & 11.87  \\
        iPNDM~\citep{zhang2023deis}        & -    & \usym{2613}      & 58.53 & 33.79 & 18.99 & 12.92 & 9.17 & 7.20 & 5.91 & 5.11  \\
        \rowcolor[gray]{0.9} iPNDM + GITS & 1.10 & \usym{2613}      & 36.18 & 19.64 & 13.18 & 9.58  & 7.68 & 6.44 & 5.24 & 4.59  \\
        \rowcolor[gray]{0.9} iPNDM + GITS (default) & 1.15 & \usym{2613}      & 34.47 & 18.95 & 10.79 & 8.43  & 6.83 & 5.82 & 4.96 & 4.48  \\
        \rowcolor[gray]{0.9} iPNDM + GITS & 1.20 & \usym{2613}      & 32.70 & 18.59 & 11.04 & 9.23  & 7.18 & 6.20 & 5.50 & 5.08  \\
        \rowcolor[gray]{0.9} iPNDM + GITS & 1.10 & \usym{1F5F8}     & 31.50 & 21.50 & 13.73 & 10.74 & 7.99 & 6.88 & 5.29 & 4.64  \\
        \rowcolor[gray]{0.9} iPNDM + GITS & 1.15 & \usym{1F5F8}     & 28.01 & 18.28 & 10.28 & 8.68  & 6.76 & 5.90 & 4.81 & \textbf{4.40}  \\
        \rowcolor[gray]{0.9} iPNDM + GITS & 1.20 & \usym{1F5F8}     & \textbf{26.41} & \textbf{16.41} & \textbf{9.85}  & \textbf{8.39}  & \textbf{6.44} & \textbf{5.64} & \textbf{4.79} & 4.47  \\
        \midrule
        \multicolumn{11}{l}{\textbf{LSUN Bedroom 256$\times$256}~\citep{yu2015lsun} (pixel-space)} \\
        \midrule
        DDIM~\citep{song2021ddim}          & -    & \usym{2613}      & 86.13 & 54.45 & 34.34 & 25.25 & 19.49 & 15.71 & 13.26 & 11.42  \\
        \rowcolor[gray]{0.9} DDIM + GITS  & 1.05 & \usym{2613}      & 81.77 & 36.89 & 27.46 & 18.78 & 13.60 & 12.23 & 10.29 & 8.77   \\
        \rowcolor[gray]{0.9} DDIM + GITS (default)  & 1.10 & \usym{2613}      & 61.85 & 35.12 & 22.04 & 16.54 & 13.58 & 11.20 & 9.82  & 9.04   \\
        \rowcolor[gray]{0.9} DDIM + GITS & 1.15 & \usym{2613}      & 60.11 & 31.02 & 23.65 & 17.18 & 13.42 & 12.61 & 10.89 & 10.57  \\
        iPNDM~\citep{zhang2023deis}        & -    & \usym{2613}      & 80.99 & 43.90 & 26.65 & 20.73 & 13.80 & 11.78 & 8.38 & 5.57  \\
        \rowcolor[gray]{0.9} iPNDM + GITS & 1.05 & \usym{2613}      & 59.02 & 24.71 & 19.08 & 12.77 & \textbf{8.19}  & \textbf{6.67}  & \textbf{5.58} & \textbf{4.83}  \\
        \rowcolor[gray]{0.9} iPNDM + GITS (default) & 1.10 & \usym{2613}      & 45.75 & 22.98 & \textbf{15.85} & \textbf{10.41} & 8.63  & 7.31  & 6.01 & 5.28  \\
        \rowcolor[gray]{0.9} iPNDM + GITS & 1.15 & \usym{2613}      & \textbf{44.78} & \textbf{21.67} & 17.29 & 11.52 & 9.59  & 8.82  & 7.22 & 5.97  \\
        \bottomrule
    \end{tabular}
}
\end{table}

\subsection{Comparision of Time Schedule and Sample Quality}
\label{subsec:appendix_results}

\begin{table}[t!]
    \caption{Comparison of various time schedules on CIFAR-10. 
    } 
    \label{tab:schedule}
    \vskip 0.1in
    \centering
    \fontsize{8}{10}\selectfont
    \begin{tabular}{clc}
        \toprule
        NFE & TIME SCHEDULE & FID \\
        \midrule
        \multicolumn{3}{l}{\textbf{Uniform}} \\
        \midrule
        3  & [80.0000,  6.9503, 1.2867, 0.0020]                                                         & 50.44 \\
        4  & [80.0000, 11.7343, 2.8237, 0.8565, 0.0020]                                                 & 18.73 \\
        5  & [80.0000, 16.5063, 4.7464, 1.7541, 0.6502, 0.0020]                                         & 17.34 \\
        6  & [80.0000, 20.9656, 6.9503, 2.8237, 1.2867, 0.5272, 0.0020]                                 & 9.75 \\
        7  & [80.0000, 25.0154, 9.3124, 4.0679, 2.0043, 1.0249, 0.4447, 0.0020]                         & 12.50 \\
        8  & [80.0000, 28.6496, 11.7343, 5.4561, 2.8237, 1.5621, 0.8565, 0.3852, 0.0020]                 & 7.56 \\
        9  & [80.0000, 31.8981, 14.1472, 6.9503, 3.7419, 2.1599, 1.2867, 0.7382, 0.3401, 0.0020]         & 10.60 \\
        10 & [80.0000, 34.8018, 16.5063, 8.5141, 4.7464, 2.8237, 1.7541, 1.0985, 0.6502, 0.3047, 0.0020] & 7.35 \\
        \midrule
        \multicolumn{3}{l}{ \textbf{LogSNR}} \\
        \midrule
        3  & [80.0000,  2.3392, 0.0684, 0.0020]                                                         & 88.38 \\
        4  & [80.0000,  5.6569, 0.4000, 0.0283, 0.0020]                                                 & 35.59 \\
        5  & [80.0000,  9.6090, 1.1542, 0.1386, 0.0167, 0.0020]                                         & 19.87 \\
        6  & [80.0000, 13.6798, 2.3392, 0.4000, 0.0684, 0.0117, 0.0020]                                 & 10.68 \\
        7  & [80.0000, 17.6057, 3.8745, 0.8527, 0.1876, 0.0413, 0.0091, 0.0020]                         & 6.56 \\
        8  & [80.0000, 21.2732, 5.6569, 1.5042, 0.4000, 0.1064, 0.0283, 0.0075, 0.0020]                 & 4.74 \\
        9  & [80.0000, 24.6462, 7.5929, 2.3392, 0.7207, 0.2220, 0.0684, 0.0211, 0.0065, 0.0020]         & 3.53 \\
        10 & [80.0000, 27.7258, 9.6090, 3.3302, 1.1542, 0.4000, 0.1386, 0.0480, 0.0167, 0.0058, 0.0020] & 2.94 \\
        \midrule
        \multicolumn{3}{l}{ \textbf{Polynomial ($\rho = 7$)}} \\
        \midrule
        3  & [80.0000,  9.7232, 0.4700, 0.0020]                                                         & 47.98 \\
        4  & [80.0000, 17.5278, 2.5152, 0.1698, 0.0020]                                                 & 24.82 \\
        5  & [80.0000, 24.4083, 5.8389, 0.9654, 0.0851, 0.0020]                                         & 13.59 \\
        6  & [80.0000, 30.1833, 9.7232, 2.5152, 0.4700, 0.0515, 0.0020]                                 & 7.05 \\
        7  & [80.0000, 34.9922, 13.6986, 4.6371, 1.2866, 0.2675, 0.0352, 0.0020]                         & 5.08 \\
        8  & [80.0000, 39.0167, 17.5278, 7.1005, 2.5152, 0.7434, 0.1698, 0.0261, 0.0020]                 & 3.69 \\
        9  & [80.0000, 42.4152, 21.1087, 9.7232, 4.0661, 1.5017, 0.4700, 0.1166, 0.0204, 0.0020]         & 3.17 \\
        10 & [80.0000, 45.3137, 24.4083,12.3816, 5.8389, 2.5152, 0.9654, 0.3183, 0.0851, 0.0167, 0.0020] & 2.77 \\
        \midrule
        \multicolumn{3}{l}{ \textbf{GITS (ours)}} \\
        \midrule
        3  & [80.0000,  3.8811, 0.9654, 0.0020]                                                         & 43.89 \\
        4  & [80.0000,  5.8389, 1.8543, 0.4700, 0.0020]                                                 & 15.10 \\
        5  & [80.0000,  6.6563, 2.1632, 0.8119, 0.2107, 0.0020]                                         & 8.38 \\
        6  & [80.0000, 10.9836, 3.8811, 1.5840, 0.5666, 0.1698, 0.0020]                                 & 4.88 \\
        7  & [80.0000, 12.3816, 3.8811, 1.5840, 0.5666, 0.1698, 0.0395, 0.0020]                         & 3.76 \\
        8  & [80.0000, 10.9836, 3.8811, 1.8543, 0.9654, 0.4700, 0.2107, 0.0665, 0.0020]                 & 3.24 \\
        9  & [80.0000, 12.3816, 4.4590, 2.1632, 1.1431, 0.5666, 0.2597, 0.1079, 0.0300, 0.0020]         & 2.70 \\
        10 & [80.0000, 12.3816, 4.4590, 2.1632, 1.1431, 0.5666, 0.3183, 0.1698, 0.0665, 0.0225, 0.0020] & 2.49 \\
        \bottomrule
    \end{tabular}
\end{table}

\textbf{Time schedule.}
The uniform schedule is commonly used with the DDPM~\citep{ho2020ddpm} backbone. Following EDMs~\citep{karras2022edm}, we rewrite this schedule from its original range $[\epsilon_s, 1]$ to $[t_0, t_N]$, where $\epsilon_s = 0.001$, $t_0=0.002$ and $t_N=80$. We first uniformly sample $\tau_n$ ($n\in [0,N]$) from $[\epsilon_s, 1]$ and then calculate $t_n$ by $t_n = \sqrt{\exp({\frac{1}{2}{\beta_d}{\tau_n}^2 + \beta_{\min}\tau_n}) - 1}$, where $\beta_d = \frac{2}{\epsilon_s - 1} \frac{\log(1+{t_0}^2)}{\epsilon_s} - \log(1+{t_N}^2)$, $\beta_{\min} = \log(1+{t_N}^2) - \frac{1}{2}\beta_d$.

The logSNR time schedule is proposed for fast sampling in DPM-Solver~\citep{lu2022dpm}. We first uniformly sample $\lambda_n$ ($n\in [0,N]$) from $[\lambda_{\min}, \lambda_{\max}]$ where $\lambda_{\min} = -\log t_N$ and $\lambda_{\max} = -\log t_0$. The logSNR schedule is given by $t_n = e^{-\lambda_n}$.

The polynomial time schedule $t_n=(t_0^{1/\rho}+\frac{n}{N}(t_N^{1/\rho}-t_0^{1/\rho}))^{\rho}$ is proposed in EDM~\citep{karras2022edm}, where $t_0=0.002$, $t_N=80$, $n\in [0,N]$, and $\rho=7$.

The optimized time schedules for SDv1.5 in Figure~\ref{fig:sd_ays} include
\begin{itemize}
	\item AYS~\citep{sabour2024align}: [999, 850, 736, 645, 545, 455, 343, 233, 124, 24, 0].
	\item GITS: [999, 783, 632, 483, 350, 233, 133, 67, 33, 17, 0].
\end{itemize} 

Furthermore, we do observe strong similarity in the optimal time schedules across different models and datasets, although the exact trajectory shapes vary slightly due to the influence of the specific models and datasets. (see Figures~\ref{fig:vis_deviation},~\ref{fig:traj_3d},~\ref{fig:app_traj_deviation},~\ref{fig:recon_ffhq}). We then conducted cross-dataset experiments by directly applying a time schedule optimized on one dataset (e.g., CIFAIR-10) to others (e.g., FFHQ and ImageNet). The results are reported in Table~\ref{tab:cross_time_schedule}, where each column corresponds to the dataset used to optimize the time schedule. We can see that the results within each row remain relatively stable, which confirms that trajectory regularity is consistent across different datasets. 

\begin{table}[t]
\caption{Comparison of FID results. Each column represents the dataset used for searching the optimized time schedule. The optimized time schedules may vary during different experiments due to the randomly sampled batch for optimization.}
\label{tab:cross_time_schedule}
\centering
\fontsize{8}{10}\selectfont
\begin{tabular}{cccc}
\toprule
TIME SCHEDULE & CIFAR-10 & FFHQ & ImageNet \\ 
\midrule
\textbf{NFE=5} & & & \\ \midrule
CIFAR10   & \cellcolor[gray]{0.9} 8.78  & 8.88  & 8.21  \\
FFHQ       & 11.22 & \cellcolor[gray]{0.9} 11.12 & 10.61 \\ 
ImageNet & 11.40 & 11.44 & \cellcolor[gray]{0.9}11.02 \\ 
\midrule
\textbf{NFE=6} & & & \\ \midrule
CIFAR10    & \cellcolor[gray]{0.9}5.07  & 4.89  & 4.73  \\ 
FFHQ       & 7.28  & \cellcolor[gray]{0.9}7.00  & 6.90  \\ 
ImageNet & 8.85  & 8.61  & \cellcolor[gray]{0.9}8.43  \\ 
\midrule
\textbf{NFE=8} & & & \\ \midrule
CIFAR10    & \cellcolor[gray]{0.9}3.20  & 3.32  & 3.39  \\ 
FFHQ       & 4.88  & \cellcolor[gray]{0.9}4.52  & 4.59  \\ 
ImageNet & 6.02  & 5.79  & \cellcolor[gray]{0.9}5.94  \\
\midrule
\textbf{NFE=10} & & & \\ \midrule
CIFAR10    & \cellcolor[gray]{0.9}2.43  & 2.40  & 2.61  \\ 
FFHQ       & 3.77  & \cellcolor[gray]{0.9}3.62  & 3.59  \\ 
ImageNet & 4.57  & 4.36  & \cellcolor[gray]{0.9}4.70  \\ 
\bottomrule
\end{tabular}
\end{table}

\clearpage

\vskip 0.2in
\bibliographystyle{plainnat}
\bibliography{ref.bib}

\end{document}